\documentclass[twoside,11pt]{article}

\usepackage{jmlr2e}

\usepackage{amsmath}
\usepackage{amssymb}
\usepackage{bm}

\usepackage{subcaption}
\usepackage{caption}
\usepackage{wrapfig}
\usepackage{float}
\usepackage{xcolor}
\usepackage{tcolorbox}
\usepackage{algorithm}
\usepackage{algpseudocode}

\newcommand{\argmax}{\operatornamewithlimits{arg \,max}}
\DeclareMathOperator{\diag}{diag}

\def\mGamma{\boldsymbol \Gamma}
\def\bmGamma{\boldsymbol \Gamma_{(a, b),(b,a)}}
\def\mgamma{\gamma}

\def\bmG{\boldsymbol G_{(a, b),(b,a)}}
\def\mR{\boldsymbol R}
\def\bmR{\boldsymbol R_{(a, b),(b,a)}}

\def\tN{\tilde{\boldsymbol N}}
\def\t{\tilde}
\def\Z{\boldsymbol Z}
\def\B{\boldsymbol B}
\def\b{\boldsymbol}
\def\h{\widehat}
\def\E{\mathbb{E}}

\usepackage{lastpage}

\ShortHeadings{Spectral clustering for dependent community Hawkes process models}{Zhao, Soliman, Xu, and Paul}
\firstpageno{1}

\begin{document}

\title{Spectral clustering for dependent community Hawkes process models of temporal networks}

\author{\name Lingfei Zhao \email zhao.2412@osu.edu \\
       \addr Department of Statistics\\
       The Ohio State University\\
       Columbus, OH 43120, USA
       \AND
       \name Hadeel Soliman \email hadeel.soliman@rockets.utoledo.edu \\
       \addr Department of Electrical Engineering and Computer Science\\
       University of Toledo\\
       Toledo, OH 43606, USA
       \AND
       \name Kevin S. Xu\thanks{This research was partially conducted while K.~S. Xu was at the University of Toledo.} \email ksx2@case.edu \\
       \addr Department of Computer and Data Sciences\\
       Case Western Reserve University\\
       Cleveland, OH 44106, USA
       \AND
       \name Subhadeep Paul \email paul.963@osu.edu \\
       \addr Department of Statistics\\
       The Ohio State University\\
       Columbus, OH 43120, USA}

\editor{NA}

\maketitle

\begin{abstract} 
Temporal networks observed continuously over time through timestamped relational events data are commonly encountered in application settings including online social media communications, financial transactions, and international relations.  
Temporal networks often exhibit community structure and strong dependence patterns among node pairs. This dependence can be modeled through \emph{mutual excitations}, where an interaction event from a sender to a receiver node increases the possibility of future events among other node pairs. 

We provide statistical results for a class of models that we call \emph{dependent community Hawkes (DCH)} models, which combine the stochastic block model with mutually exciting Hawkes processes for modeling both community structure and dependence among node pairs, respectively. We derive a non-asymptotic upper bound on the misclustering error of spectral clustering on the event count matrix as a function of the number of nodes and communities, time duration, and 
the amount of dependence in the model. Our result leverages recent results on bounding an appropriate distance between a multivariate Hawkes process count vector and a Gaussian vector, along with results from random matrix theory. 
We also propose a DCH model that incorporates only self and reciprocal excitation along with highly scalable parameter estimation using a Generalized Method of Moments (GMM) estimator that we demonstrate to be consistent for growing network size and time duration. 
\end{abstract}


\section{Introduction}

In many application settings involving networks where relations between nodes change over time, the observed data consist of timestamped relational events. For example, in \emph{social media communications}, users interact with each other through specific activities such as liking, mentioning, replying to, sharing, or commenting on another user's content. In \emph{international relations and conflicts}, nations commit acts of hostility or disputes through discrete timestamped events. In daily \emph{interactions among humans}, individuals come in contact with each other through events of co-presence in a physical space. 
These types of data are usually obtained as a table of timestamped ``action" events containing information on sender, receiver, and time of every event. Such data are usually referred to as relational events data, instantaneous interaction data, contact sequences, or more generally, temporal network data \citep{butts20084,brandes2009networks,holme2012temporal}.

A large body of models and methods have been proposed in the literature for analysis of relational events data in the last two decades.
A common modeling approach involves combining a model for an underlying (but unobserved) network with a point process model for the event times. 
The model used for the underlying network is often the Stochastic Block Model (SBM) \citep{DuBois2010,Dubois2013,xin2015continuous,junuthula2019block,arastuie2020chip,soliman2022multivariate}, or the closely related Infinite Relational Model (IRM) \citep{Blundell2012} or overlapping SBM \citep{miscouridou2018modelling}.
The event times among pairs of nodes are often modeled as realizations of temporal point processes (TPPs) that are conditionally independent given the community or block assignments. For example, in \cite{DuBois2010}, the events are generated following independent Poisson processes given the latent block labels of the senders and receivers. The model of  \cite{xin2015continuous} used inhomogeneous Poisson processes and \cite{arastuie2020chip} and \cite{junuthula2019block} used self-exciting Hawkes processes to model the event histories with an SBM.

However, dependencies among the pairwise processes and temporal motifs are commonly observed in relational events data, which most of the models above do not account for\footnote{Notable exceptions include the models proposed by \cite{Dubois2013} and \citet{soliman2022multivariate}.}. For example, consider the communication between two teams within an organization as illustrated in Figure \ref{fig:teamAB}. Suppose A1 and A2 are part of team A, and B1 and B2 are part of team B. If the user A1 sends an email to the user B1 (denoted in the figure as solid black directed arrow), then this action is likely to trigger not only more emails from A1 to B1 (dashed blue arrow), but also a response event from B1 to A1 (dashed red arrow).  Moreover A1 might send an email to B1's teammate B2 to request further clarification (dashed blue arrow) or B1 might send an email to A1's teammate A2 to keep them in the loop (dashed red arrow). Further, A1's teammate A2 might send a follow up email to B1, or B1's teammate B2 might choose to respond to A1 having received the forwarded email (dashed arrows). As this example illustrates, an event has the ability to trigger multiple other events between nodes. 

\begin{figure}[t]
     \centering  \includegraphics[width = 0.3 \textwidth]{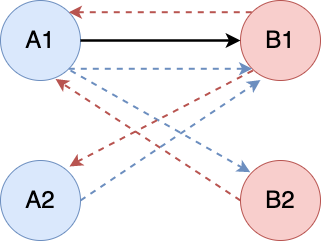}
    \caption{An example of dependence in temporal networks: an event from A1 to B1 (solid arrow) triggers multiple possible future events (red and blue dashed arrows).}
    \label{fig:teamAB}
\end{figure}

As another example, in the Militarized Interstate Disputes (MID) data that we analyze in this paper, we note that an action of threat or display of force by a country $i$ on another country $j$ leads to responses by the allies of both country $i$, the initiator, as well as country $j$, to whom the action is targeted.  In a systematic study,  \cite{paranjape2017motifs} identified a number of \textit{temporal motifs}  commonly observed in continuous-time networks. 
\cite{do2022analyzing} found that, indeed, many of these temporal motifs appear in MID data, even over short time windows. 
The presence of such temporal motifs over short time windows indicates that there are dependencies among the events. Such dependencies are also natural manifestations of social or network influence and contagion that has been widely studied \citep{nath2022identifying,goldsmith2013social}.

An important part of estimation in models based upon the SBM is to estimate the unknown blocks or communities. Several approaches have been used in the literature for estimating the community labels from temporal networks, including posterior inference with MCMC procedures  \citep{DuBois2010,Dubois2013,Blundell2012,fan2022hawkes}, EM type algorithms \citep{xin2015continuous,junuthula2019block} and spectral clustering algorithms \citep{junuthula2019block,arastuie2020chip, soliman2022multivariate}. 
Thus far, the only theoretical guarantees for estimation methods for these models is for the CHIP model  \citep{arastuie2020chip}, where a spectral clustering algorithm was shown to be consistent using proof techniques similar to \cite{lei2015consistency} by leveraging the conditional independence of the Hawkes processes. However, the theoretical guarantees cannot be directly extended to spectral clustering in settings with dependence among the node pairs \citep{Blundell2012,junuthula2019block,soliman2022multivariate}, and so the proof techniques from \cite{arastuie2020chip} cannot be used in this case. 
This paper focuses on developing statistical theory for estimators for models that incorporate dependence with community structure. Further, while the results in \cite{arastuie2020chip} are asymptotic and require the time over which the system is observed $T \to \infty$, our results are non-asymptotic and provide upper bound on estimation error as a function of number of nodes $n$ and $T$.

\subsection{Our Contributions}
We make two main contributions. \textit{First, we develop a theoretical upper bound} on the misclustering error of the spectral clustering algorithm under a general class of models that we call \emph{Dependent Community Hawkes (DCH)} models. The class of DCH models either include or is very closely related to many prior models that combine some variant of an SBM and a Hawkes process \citep{Blundell2012, miscouridou2018modelling, junuthula2019block, arastuie2020chip, soliman2022multivariate}. As mentioned earlier, our upper bounds in this paper are non-asymptotic in both the number of nodes $n$ and the time points $T$, illustrating data requirements in terms of both how many interacting entities we need to observe and how long we need to observe the interactions. Our results also allow us to study the effect of dependence among node pairs on the accuracy of spectral clustering. Finally, by letting $T \to \infty$, we establish conditions under which spectral clustering provides a consistent estimate of the community structure as we observe the system for a long time. 
These results also provide the first theoretical guarantees for estimation in the BHM \citep{junuthula2019block} and MULCH \citep{soliman2022multivariate} models, which fall into the class of DCH models we consider. 

The DCH models can be further thought of as plausible generative models for \textit{static weighted networks} where weights denote some type of counts, and random variables denoting the \textit{weighted edges are dependent}. How to utilize edge weights in a weighted network is a significant open problem in network science, where such weights are often treated as bounded nuisance parameters. Moreover, very few works on network science consider dependent edge weights.
\textit{We hypothesize} that in many application settings, observed static networks, especially those with weighted edges, are generated through an underlying relational event model.  Hence, the theoretical results in this paper are relevant more broadly. 

\textit{Second, we propose the self and reciprocal excitation Hawkes process model (SR)}, which also falls into the class of DCH models. The SR model can be thought of as an intermediate model between the highly scalable but less flexible CHIP model \citep{arastuie2020chip} and the highly flexible but less scalable MULCH model \citep{soliman2022multivariate}. 
We develop a highly scalable estimation approach for the SR model involving Generalized Method of Moments (GMM) estimation. 
We also develop theoretical consistency results for the GMM estimators of the Hawkes process parameters. The estimation method is related to the GMM estimation method in \cite{achab2017uncovering}, but our method is different from \cite{achab2017uncovering} in that we leverage the counts from multiple multivariate Hawkes processes. A theoretical novelty in our result is an explicit proof of the identification condition for a restricted SR model, which is an assumption in the results of \cite{achab2017uncovering}. The identification condition needs to hold for a multivariate Hawkes process in order to consistently estimate the parameters using \cite{achab2017uncovering}'s approach. Such identification is not guaranteed in general for multivariate Hawkes processes and consequently for the DCH family of models. We show that the structure of a restricted version of SR model allows us to prove identification explicitly.  
Our proposed SR model and GMM-based estimator retains the computationally efficiency of CHIP, yet provides better fit to real network datasets.  We further propose a new computationally efficient likelihood refinement procedure that iteratively refines the community assignments given the initial spectral clustering and parameter estimates. The proposed procedure is computationally feasible on large datasets and empirically improves community detection accuracy.

\subsection{Background Literature and Related Work}
\label{sec:background}
\paragraph{Stochastic Block Model:}
The Stochastic Block Model (SBM) is a widely studied random graph model for networks with community structure \citep{hll83}. The SBM proposes that every node in the network belongs to exactly one community, and given the community assignments, the edges between pairs of nodes are formed independently following a Bernoulli distribution whose parameters depend only the community assignment of nodes. 
In many application settings, the community assignments are unobserved and must be estimated from the network itself. 

Spectral clustering has emerged as a computationally efficient estimator for the communities, and recent results provide a variety of theoretical guarantees on its accuracy under different assumptions \citep{rcy11,lei2015consistency,gao2017achieving}. 
We note that these theoretical guarantees all assume conditional independence of edges between node pairs, which does not apply to the class of DCH models we consider in this paper. 
We use some of the proof techniques used in this prior work but also have to consider the dependence between node pairs to provide guarantees for the class of DCH models.

\paragraph{Hawkes Process:}

The Hawkes process \citep{hawkes1971spectra,laub2015hawkes} 
is a temporal point process model for modeling the stochastic process of arrival times of events. When modeling multiple sequences of event histories, the process is self and mutually exciting, implying that the instantaneous intensity of the process is increased by new events occurring both in the self and neighboring processes. 
The mutually exciting Hawkes process is a multidimensional point process model where the instantaneous intensity of arrivals of events in one process or dimension is increased by arrivals in both the same process or dimension as well as other processes or dimensions.

\paragraph{Related Work:}
There is a large body of prior literature on modeling continuous-time networks using a combination of a latent variable model for an underlying (but unobserved) network and a temporal point process model for the observed relational events. 
The underlying network model used is typically a variant of the Stochastic Block Model (SBM) \citep{DuBois2010,Blundell2012,Dubois2013,xin2015continuous,matias2015semiparametric, corneli2018multiple,junuthula2019block,arastuie2020chip,miscouridou2018modelling,soliman2022multivariate, fan2022hawkes} or the Latent Space Model (LSM)
\citep{yang2017decoupling, huang2022mutually,rastelli2021continuous, romero2023gaussian,passino2023mutually}. 
Such models typically assume that the relational events between node pairs are conditionally independent given the latent variables, i.e., the community assignments in the SBM and latent positions in the LSM.

The models with conditionally independent processes for different node pairs, such as CHIP \citep{arastuie2020chip}, fail to model the dependencies across the node pairs in the data. This aspect was recognized by \cite{Blundell2012} who used mutually exciting Hawkes processes to model reciprocating relationships in an IRM. The inhomogeneous Poisson processes in \cite{Dubois2013} incorporated observed count statistics on various types of motifs into its intensity function.

Recently, \cite{soliman2022multivariate}  considered mutually exciting Hawkes processes within an SBM structure to include complex dependencies, including reciprocity, generalized reciprocity, and turn continuing in their MULCH model. 
As they discuss, a fully mutually exciting Hawkes process for modeling such a system will require $O(n^2)$ processes that are dependent on each other and consist of $O(n^2 \times n^2)$ matrix of unknown self and mutual excitation (jump size) parameters. Such a model will be computationally intractable even for moderate sized datasets (e.g., $n=100$ nodes), while fitting the model will be statistically difficult for sparse datasets. As a solution, \cite{soliman2022multivariate} proposed to limit dependence only within the block pair that a node pair belongs to and the reciprocating block pair with the help of latent block or community assignments. Therefore if a node pair $(i,j)$ is such that $i$ belongs to community $a$ and $j$ belongs to community $b$, then interaction events from $i$ to $j$ is independent of events from $k$ to $l$ provided $k \notin \{a,b\}$ and $l \notin \{a,b\}$. However, the MULCH model only includes some specific forms of dependence among the node pairs. 
We generalize this observation by introducing a class of models with a very general form of dependence of node pairs.

\section{Dependent Community Hawkes (DCH) Models}

\label{sec:model}
We consider a relational events data table with timestamped interactions obtained from a continuously evolving system with $n$ nodes over time period $[0,T]$. 
We propose a class of models, which we call \textit{Dependent Community Hawkes (DCH)} models. 
These models are capable of modeling complex dependence patterns among the node pairs and are useful for studying the properties of the spectral clustering of the count matrix. This class of models either subsumes or is closely related to a number of existing models in the literature.

We assume each node in the network, $i$, has an unknown community or block label $z_i$, that takes values in $\{1,\ldots,K\}$. Let 
$X$ denote an assignment operator which assigns an ordered node pair to its ordered block pair. For example if $z_i = a, z_j =b$, then $X(i,j) = (a,b)$.  We assume that events between node pair $(i,j)$, such that $X(i,j) = (a,b)$ are independent of events between node pairs $(i',j')$, if $X(i',j') \notin \{(a,b),(b,a)\}$. On the other hand, events in node pairs $(i',j')$ which are in the same block pair, i.e., $X(i',j') =(a,b)$ \textit{or} reciprocal block pair, i.e.,  $X(i',j') =(b,a)$, exert dependence on events from node $i$ to $j$ controlled by the excitation patterns of a mutually exciting multivariate Hawkes process \citep{hawkes1971spectra,hawkes2018hawkes}.

We define the conditional intensity function for events $i \to j$  in the mutually exciting Hawkes process with the exponential kernel as
\begin{equation*}
    \lambda_{ij}(t) = \mu_{ij} + \sum_{\substack{(i',j'): X(i',j') \in \\ \{(a,b),(b,a)\}}} \bigg \{ \alpha^{i'j' \rightarrow ij} \beta^{i'j' \rightarrow ij} \sum_{t_s \in T_{i'j'}} \exp(-\beta^{i'j' \rightarrow ij}(t - t_s)) \bigg\},
\end{equation*}
where $T_{i'j'}$ is the set of timestamps for events from $i'$ to $j'$, and $\mu_{ij}>0$ is the baseline intensity parameter. Let $\b \mu$ be an $n\times n$ matrix 
whose $(i,j)$th element is $\mu_{ij}$.
The excitation parameters $\alpha^{i'j' \rightarrow ij}$ of the $n^2$ dimensional multivariate Hawkes process that govern the $n^2$ dyadic event times can be written as elements of the $n^2 \times n^2$ matrix $\boldsymbol{\Gamma}$. Since the Kernel function is an exponential Kernel, the parameters $\alpha^{i'j' \to ij}$ has the interpretation of the mean number of events from $i$ to $j$ directly (and causally) triggered by an event from $i'$ to $j'$ \cite{achab2017uncovering}. For ease of exposition, we will explicitly allow self-connections, which also occur in some application settings, e.g., a user posting on their own Facebook wall.
The class of DCH models further assumes a block or community structure in the matrix $\boldsymbol{\mu}$, i.e., $\b{\mu} = \b{Z} \b{M} \b{Z}^T$, where $\b{Z}$ is the matrix whose rows are community indicator vectors and $\b{M}$ is a $K \times K$ matrix. Note the $\b{M}$ matrix is not symmetric, and consequently,  the $\b \mu$ matrix is also not symmetric. 

\subsection{The Block-diagonal Excitation Matrix}
\label{sec:blocks_excitation}

The assumption that a node pair can only receive mutual excitation from node pairs in its own block pair and reciprocal block pair implies that the $\boldsymbol{\Gamma}$ matrix can be rearranged in such a way that the resulting matrix is a block diagonal matrix with $\frac{K(K+1)}{2}$ blocks. Since we are going to describe a $n \times n $ matrix of dyadic relational processes using a $n^2$ dimensional multivariate Hawkes process, we need to define an ordering of node pairs $(i,j)$ such that we can uniquely traverse between the matrix and its vectorized version. 

In particular, we order the rows and columns of the matrix $\b \Gamma$ by block pair assignments of the node pairs given by the operator $X(\cdot,\cdot)$. Let $n_a$ be the number of nodes in the community $a$. If $a=b$, i.e., both the nodes of the pair are in the same community, then we define $\b \Gamma_{(a,a),(a,a)}$ as the $n_{a}^2\times n_{a}^2$ matrix recording the influence the $n_a^2$ node pairs (including self-loop node pairs) for which $X(i,j)=(a,a)$, exert on each other. Let $\{i_{1},i_{2},\dots,i_{n_a}\}$ denote the nodes which are in the community $a$. Then we can order the $n_{a}^2$ directed node pairs as $\mathcal A_{aa} = \{(i_{1},i_{2}),\dots, (i_{1},i_{n_a}),(i_{2},i_{1}),\dots,(i_{2},i_{n_a}),\dots, (i_{n_a},i_{ n_a})\}$. Both the rows and columns of the matrix $\b \Gamma_{(a,a), (a,a)}$ are arranged in the order specified in the set $\mathcal A_{aa}$. 
    
If $a\neq b$, then define $\boldsymbol \Gamma_{(a,b), (b,a)}\in \mathbb{R}^{2n_{a}n_b\times 2n_{a}n_b}$ with rows and columns denoting all $n_an_b$ node pairs for which $X(i,j)=(a,b)$ and all $n_an_b$ node pairs for which $X(i,j)=(b,a)$. Let $\{i_{1},i_{2},\dots,i_{n_a}\}$ denote the nodes that are in block $a$, and $\{j_{1},j_{2},\dots,j_{n_b}\}$ denote the nodes that are in block $b$. We can also arrange the $2n_an_b$ directed node pairs in the following ordered set: $\mathcal A_{ab}=\{(i_{1},j_{1}),\dots, (i_1,j_{n_b}), (i_2, j_1), \ldots, (i_2, j_{n_b}),\ldots, (i_{n_a},j_{n_b}), (j_{1},i_{1}),\dots, (j_{n_b },i_{n_a})\}$,  
such that, for the first $n_an_b$ node pairs, $X(i,j)=(a,b)$, while for the next $n_an_b$ node pairs, $X(i,j)=(b,a)$. 

By this construction, we can reorder all these directed node pairs  
to get the set of ordered node pairs $\mathcal{A}$. We define two operators. Let $\operatorname{vec}(\b A)$ define the vectorized form of a matrix $\b A$ according to some order and $\operatorname{vec}^{-1}(\b b)$ define the matrix one obtains with the elements of vector $\b b$, such that $\operatorname{vec}^{-1}(\operatorname{vec}(\b A))= \b A$.  Then we define $\operatorname{vec}(\b \mu)\in \mathbb{R}^{n^2}$  as the vectorized form of  baseline intensities such that the elements are ordered according to the set $\mathcal{A}$.

According to this construction, we can write 
\[\boldsymbol{\Gamma} = \left(\begin{array}{cccc}
\boldsymbol\Gamma_{(1,1), (1,1)}  & & \\
& \boldsymbol\Gamma_{(1,2), (2,1)}  & \\
& &  \dots& 
\end{array}\right),\]
i.e., a block diagonal matrix consisting of blocks $\left(\boldsymbol\Gamma_{(a,b),(b,a)}\right).$ We can further write each $\boldsymbol \Gamma_{(a,b),(b,a)}$ for $a \neq b$ as a block matrix, i.e., 
\[
\label{eq:block_Gamma_G}
\boldsymbol{\Gamma}_{(a, b),(b,a)}=\left(\begin{array}{ll}
\boldsymbol{\Gamma}_{a b \rightarrow a b} & \boldsymbol{\Gamma}_{ba \rightarrow ab} \\
\boldsymbol{\Gamma}_{ab \rightarrow ba} & \boldsymbol{\Gamma}_{b a \rightarrow b a}
\end{array}\right).\] 
The first block $\b \Gamma_{ab \rightarrow ab}$ of dimension $n_an_b \times n_an_b$ has elements $\alpha^{i'j'\rightarrow ij}$, where $X(i,j)=(a,b)$ and $X(i',j')=(a,b)$. The remaining blocks are also defined similarly.

So far, we have not put any restrictions on the excitation parameters $\alpha^{i'j'\rightarrow ij}$ governing the dependence patterns within a block pair. Now, we further require that, for any block pair $(a,b)$, the submatrices $\b \Gamma_{ab \rightarrow ab}, \b \Gamma_{ab \rightarrow ba}$, $ \b \Gamma_{ba \rightarrow ab},$ and $\b \Gamma_{ba \rightarrow ba}$ have identical row sums and column sums. Therefore if $X(i,j)=X(i',j')=(a,b)$, then the \textit{total influence} through mutual excitation that processes $i \to j$ and $i' \to j'$ send and receive from other processes in block pairs $(a,b)$ and $(b,a)$ are identical. This property can be thought of as the notion of \textit{stochastic equivalence} in the DCH models. For comparison, in SBM, the notion of stochastic equivalence is that, for two nodes $i$ and $i'$, if $z_i = z_{i'}$, then the probabilities of connection with the rest of the network are the same for $i$ and $i'$. The notion of stochastic equivalence in the DCH models implies that, node pairs in the same community pair ($X(i,j)=X(i',j')=(a,b)$), send (row sum) and receive (column sum) identical amount of influence to other node pairs in the community pairs $(a,b)$ and $(b,a)$.

The combination of $\b \mu, \b \Gamma$ matrices defined above along with this notion of stochastic equivalence defines the DCH models. Next we show that CHIP \citep{arastuie2020chip}, BHM \citep{junuthula2019block}, and MULCH \citep{soliman2022multivariate} models are special cases of the DCH models, and  propose another special case of the DCH  models.

\subsection{Examples of DCH Models}
\label{sec:dch_examples}

\paragraph{Community Hawkes Independent Pairs (CHIP) Model:} 
The CHIP model in \cite{arastuie2020chip} is a special case of the DCH models described above with the $\b\Gamma$ matrix being a diagonal matrix.  
The conditional intensity function for the events between node pair $(i,j)$ such that $X(i,j) = (a,b)$ in this model is
\begin{equation}
    \label{eq:cif_chip}
    \lambda_{ij}(t) = M_{ab} + \sum_{t_s\in T_{ij}} \alpha_{ab}^{n} \beta^n_{ab} e^{-\beta^n_{ab} (t-t_{s})}.
\end{equation}
Since the process only has a self-exciting term and no mechanism of mutual excitation, the $\b \Gamma$ matrix is diagonal. Therefore the components of the $\b \Gamma$ matrix are
\begin{equation*}
\boldsymbol{\Gamma}_{(a, b),(b,a)}=\left(\begin{array}{ll}
\alpha_{ab}^n \b I_{n_a n_b} & 0 \\
0 & \alpha_{ba}^n\b I_{n_a n_b}
\end{array}\right)
\end{equation*} 
when $a \neq b$ and $\boldsymbol{\Gamma}_{(a,a),(a,a)} = \alpha_{aa}^n I_{n_a(n_a-1)}$. The matrix $\b{\mu}$ has a block structure since $\mu_{ij}= M_{z_i, z_j}$, which only depends on the community assignments of nodes $i$ and $j$. Therefore, the model is part of the DCH family.\\

\paragraph{Block Hawkes Model (BHM):} The BHM model in \cite{junuthula2019block} uses a self-exciting (univariate) Hawkes process for each block pair $(a,b)$ to generate events. 
The conditional intensity function for events between block pair $(a,b)$ in this model is given by
\begin{equation*}
    \lambda_{ab}(t) = M_{ab} + \sum_{t_s\in T_{ij}} \alpha_{ab}^{n} \beta^n_{ab} e^{-\beta^n_{ab} (t-t_{s})}.
\end{equation*}
Notice that, unlike \eqref{eq:cif_chip}, the Hawkes process is for the entire block pair $(a,b)$, not the individual node pairs $(i,j)$.
This block pair Hawkes process is then randomly thinned so that each node pair $(i,j)$ such that $X(i,j) = (a,b)$ is equally likely to ``receive'' the event. 
This can be equivalently represented by a mutually exciting Hawkes process with $n_a n_b$ different dimensions such that each dimension excites each other dimension equally, i.e., the excitation matrix for block pair $(a,b)$ is a constant multiplied by the all-ones matrix. 
The components of the excitation matrix $\b \Gamma$ then have the following form:
\begin{equation}
\boldsymbol{\Gamma}_{(a, b), (b,a)}=\left(\begin{array}{ll}
\frac{\alpha_{ab}^n}{n_a n_b} \b 1_{n_a n_b} \b 1_{n_a n_b}^T & 0 \\
0 & \frac{\alpha_{ba}^n}{n_a n_b} \b 1_{n_a n_b} \b 1_{n_a n_b}^T
\end{array}\right)
\end{equation} 
when $a \neq b$ and $\boldsymbol{\Gamma}_{(a,a),(a,a)} = \alpha_{aa}^n \b 1_{n_a(n_a-1)} \b 1_{n_a(n_a-1)}^T$.

\paragraph{Multivariate Community Hawkes (MULCH) Model:} The MULCH model in \cite{soliman2022multivariate} is more flexible than the CHIP model and introduces a larger range of mutual excitation types. 
For a node pair $(i,j)$ such that $z_i=a,z_j=b$, 
the conditional intensity function for events $i \to j$  given the history of all events in the mutually exciting Hawkes process is
\begin{equation}
    \lambda_{ij}(t) = \mu_{ij} + \sum_{\substack{(x,y)}} \alpha^{xy \rightarrow ij} \beta^{xy \rightarrow ij} \sum_{t_s \in T_{xy}} \exp(-\beta^{xy \rightarrow ij}(t - t_s))
    \label{fullMBHM}
\end{equation}
The excitation parameters of the multivariate Hawkes process governing the intensity function for events from $i$ to $j$ in (\ref{fullMBHM}) satisfy
\begin{equation*}
\alpha^{xy \rightarrow ij}=\left\{\begin{array}{rl}
 \alpha_{ab}^n,& \text{if } x=i,y=j \text{ (self excitation)}, \\
\alpha_{ab}^r,& \text{if } x=j,y=i \text{ (reciprocal excitation)},\\
\alpha_{ab}^{tc},& \text{if } x=i,z_y=b \text{ (turn continuation)}, \\
\alpha_{ab}^{ac},& \text{if } z_x=a, y=j \text{ (allied continuation)}, \\
\alpha_{ab}^{gr},& \text{if } z_x=b, y=i \text{ (generalized reciprocity)}, \\
\alpha_{ab}^{ar},& \text{if } x=j, z_y=a \text{ (allied reciprocity)}, \\
0,& \text{otherwise},\\
\end{array}\right.
~~~~~~~
\mu_{ij} = M_{ab},
\end{equation*}
and kernel functions have the similar block structure as $\b \Gamma$. 
From the discussion in Appendix A.3 of \cite{soliman2022multivariate}, the condition of identical row sum is satisfied. Therefore, the MULCH model is a special case of the DCH models.

\subsection{Self and Reciprocal Excitation (SR) Model}  
\label{sec:sr_model}
Fitting the CHIP model \citep{arastuie2020chip} to large scale networks with millions of nodes is possible due to its computationally efficient moment-based estimation. 
However, the model lacks flexibility due to not modeling any dependence on dyadic pairs. 
The MULCH model \citep{soliman2022multivariate}, on the other hand, is a highly flexible model that goes even beyond dyadic dependence, but the maximum likelihood estimator is very slow, and thus the model scales only to thousands of nodes. Furthermore, \cite{soliman2022multivariate} used a sum of known kernels approach to approximate the decay parameter $\beta$ because a direct estimation of the parameter is intractable. 

We propose a new model, which we call the \emph{Self and Reciprocal Excitation (SR) model}. It is also a member of the above DCH class of models, just like CHIP and MULCH. The SR model is less flexible than MULCH but is computationally more tractable. Given the community assignments of two nodes $i,j$, the pair of event times $\{T_{ij},T_{ji}\}$ follows a bivariate Hawkes process that is independent of all other node pairs. 
We note that this type of bivariate Hawkes process structure has also been used in several latent space Hawkes process models \citep{yang2017decoupling,huang2022mutually}, which do not belong to the DCH class.

For the SR model with $K$ communities, the conditional intensity function for the process from node $i$ to node $j$ such that $z_i=a$ and $z_j=b$, is given by
\begin{equation}
\label{eq:SR_model}
    \lambda_{ij}(t) = M_{ab} + \sum_{t_s\in T_{ij}, t_s \leq t} \alpha_{ab}^{n} \beta^n_{ab} e^{-\beta^n_{ab} (t-t_{s})} + \sum_{t_s\in T_{ji}, t_s \leq t}  \alpha_{ab}^r \beta^r_{ab}  e^ {-\beta^r_{ab} (t-t_{s})},
\end{equation}
where $\b M,\b\alpha^n,\b \alpha^r,\b\beta^n,\b\beta^r$ are all $K \times K$ non-negative matrices of parameters. The parameter $M_{ab}$ controls the baseline intensity of communication from a node $i$ that belongs to community $a$ to a node $j$ that belongs to community $b$. 
The second term in \eqref{eq:SR_model} models self excitation, i.e., the phenomenon that node $i$ is more likely to send a message to node $j$ if it has sent a message to $j$ in recent past. 
The third term in \eqref{eq:SR_model}, on the other hand, models reciprocal excitation, whereby node $j$ is more likely to send a message to node $i$ (reciprocate) if it receives a message from $i$. The  parameters $\alpha^n_{ab}, \alpha^r_{ab}$ control the jump size, and  $\beta^n_{ab}, \beta^r_{ab}$ control the decay rate of the intensity function followed by a self event $(i,j)$ and a reciprocal event $(j,i)$, respectively. 

For this model, the $\b \Gamma$ matrix defined earlier is block diagonal. Additionally, the $\b \Gamma_{(a,b),(b,a)}$ blocks have a property that, for every row, say the row corresponding to a node pair $(i,j)$, there is the non-zero element $\alpha_{ab}^n$ in the diagonal position, and a non-zero element in exactly one other spot, namely the row corresponding to node pair $(j,i)$ with element $\alpha_{ab}^r$. Clearly, the rows of $\b \Gamma_{(a,b),(b,a)}$ in this case have the same sum, $\alpha_{ab}^n + \alpha_{ab}^r$. Therefore, this model is a special case of the DCH model. 

\subsubsection{Restricted SR Model}
\label{sec:restricted_sr}
We further define a restricted version of this SR model where we let $\alpha^r_{ab} = \alpha^r_{ba}$, so that the amount of reciprocal excitation between block pairs $(a,b)$ and $(b,a)$ is identical.  
This reduces the number of parameters in the $\b M$ and $\b \Gamma$ matrices for block pairs $(a,b)$ and $(b,a)$ with $a\neq b$ from $6$ to $5$ parameters:  $M_{ab},M_{ba}, \alpha^n_{ab}, \alpha^n_{ba},\alpha^r_{ab}$.
This \emph{restricted SR} model reduces the flexibility of the SR model by constraining the reciprocal excitation parameters; however, it enables us to 
propose a computationally fast estimation procedure that includes a generalized method of moment (GMM) estimator of the parameters in Section \ref{sec:gmm_estimation}.

An alternative way to restrict the SR model is to have a shared self excitation rather than reciprocal excitation parameter between block pairs $(a,b)$ and $(b,a)$, i.e., $\alpha_{ab}^n = \alpha_{ba}^n$. 
This also reduces the number of parameters from $6$ to $5$ to enable estimation using the GMM, although our theoretical results in Section \ref{sec:sr_theory} may not hold. 
We consider this model variant in experiments in Section \ref{sec:ablation}.

\section{Spectral Clustering in the DCH Models}
Let $\b N_T$ be the $n\times n$ matrix whose $(i,j)$th element denotes the number of events that node $i$ sends to node $j$ until time $T$. Recall that we allow node $i$ to send events to itself. The diagonal elements $(\b N_T)_{ii}$ denote the events $i$ sends to itself. We call this asymmetric (due to directed events) and weighted matrix the \emph{count matrix}. 

The first step of our estimation procedure in the DCH models is to obtain the community assignments from the spectral clustering method (described in Algorithm \ref{alg:spectral_clustering}) applied to this count matrix. We derive an upper bound on the error rates of community detection using this method for count matrices generated by a model in the class of DCH models. 
The upper bound is non-asymptotic in $n$ and $T$ and provides explicit dependence on $n$, $T$, and other model quantities. This upper bound then leads to results on consistency of spectral clustering when $T \to \infty$. In order to interpret these dependencies on model quantities better, we obtain the bounds under a simplified special case of the DCH models.

We define the notations $\|\cdot \|_2, \|\cdot \|_{\infty}, \|\cdot \|_{1}, \rho(\cdot)$ to denote the spectral norm, maximum absolute row sum, maximum absolute column sum norm, and the spectral radius of a matrix, respectively, while $\|\cdot \|$ denotes the Euclidean norm of a vector.

\begin{algorithm}[t]
\caption{Spectral Clustering on the Count Matrix}
\label{alg:spectral_clustering}
\textbf{Input:} Count matrix $\mathbf{N}_T$; number of clusters $K$

\textbf{Output:} Membership vector $\boldsymbol{z}$

\begin{algorithmic}[1]
\State Compute $\mathbf{X}_L, \mathbf{X}_R \in \mathbb{R}^{n \times K}$ as the top $K$ left and right singular vectors of $\mathbf{N}_T$.
\State Form matrix $\mathbf{X} = \left(\mathbf{X}_L \mid \mathbf{X}_R\right) \in \mathbb{R}^{n \times 2K}$ by column-wise concatenation.
\State Define index set $\mathcal{I} = \left\{ i : \| \mathbf{X}_{i\cdot} \| > 0 \right\}$.
\State Extract rows: $\mathbf{X}^+ = (\mathbf{X}_{\mathcal{I}\cdot})$.
\State Normalize rows to unit length: $\mathbf{X}^{+*}_{ij} =\frac{\mathbf{X}^+_{ij}}{\|\mathbf{X}^+_{i,.}\|}$.
\State Apply $(1+\varepsilon)$-approximate $k$-means to rows of $\mathbf{X}^{+*}$ to get $K$ clusters.
\State Assign nodes not in $\mathcal{I}$ to the first cluster.
\State \Return membership vector $\boldsymbol{z}$.
\end{algorithmic}
\end{algorithm}

\subsection{Non-asymptotic Results for General DCH Models}

We adopt a result from \cite{khabou2021malliavin} which provides a Gaussian concentration result for multivariate Hawkes processes using the Malliavin-Stein method in our context in the following proposition. Let $\mathcal{C}^2(\mathbb R^{n^2})$ denote the class of twice differentiable functions of $n^2$ dimensional real vectors. For a function $g \in \mathcal{C}^2(\mathbb R^{n^2})$, define $\|g\|_{Lip}= \sup_{\bm x \neq \bm y}\frac{|g(\bm x)- g(\bm y)|}{\|\bm x-\bm y\|}$, and $M_2(g) = \sup_{\bm x \neq \bm y} \frac{\|\nabla g(\bm x) - \nabla g(\bm y)\|}{\|\bm x-\bm y\|}$, where $\bm x,\bm y \in \mathbb R^{n^2}$. For any vector $\b x$, define the operator $\diag(\b x)$ as an operator that makes a diagonal matrix with the elements of the vector $\b x$. Further define $\b R=(\b I-\b \Gamma)^{-1}$. Then, the following proposition is a consequence of Theorem 1.1 in \cite{khabou2021malliavin}.

\begin{proposition} Define the distance $d_2$ between two random vectors $X$ and $Y$ as 
\[
d_2 (X,Y) = \sup_{f \in \mathcal{H}} |E[f(X)] - E[f(Y)]|,
\]
where  $\mathcal{H} = \{g \in \mathcal{C}^2(R^{n^2}): \|g\|_{Lip} \leq 1, M_2(g) \leq 1\}$. Let $n$ be a fixed quantity that does not change with $T$ and assume $\rho(\b \Gamma) <1$.
Define \[
Y_T = \frac{\operatorname{vec} (\b N_T) - \b R \operatorname{vec} (\b \mu) T}{\sqrt{T}}.
\]
Let $G \sim N_{n^2}(0, \b R \diag (\b R \operatorname{vec}(\b \mu)) \b R^T)$. 
Then there exists a constant $C(n)$ that does not depend on $T$, but possibly depends on $n$, such that 
\[
d_2 (Y_T, G) \leq \frac{C}{\sqrt{T}},
\]
for any $T$.
\label{prop:normal}
\end{proposition}
The above proposition provides a bound for the $d_2$ distance, which has also been called the ``smooth Wasserstein distance'' in the literature \citep{gaunt2023bounding}, between a suitably transformed count vector from a DCH model and a Gaussian vector with an appropriate covariance matrix. Note that, in the above proposition, the $n^2 \times n^2$ covariance matrix of the zero-mean Gaussian vector $G$ does not depend on $T$. Further,  $d_2(Y_T,G) \to 0$ implies that $Y_T$ converges to $G$ in distribution (Remark 6 in \cite{khabou2021malliavin}, Remark 2.16 in \cite{giovanni2010multi}). However, to make further progress on bounding the spectral norm difference of the count matrix from its expection, we require explicit bounds on the Kolmogorov distance between the vectors, $d_K(Y_T,G)$. The following proposition that follows from the result in \cite{gaunt2023bounding} with $m=2$ provides that.   
\begin{proposition}
Suppose $\sigma=\min_{1\leq j\leq n^2} (\b R \diag (\b R \operatorname{vec}(\b \mu)) \b R^T)_{jj}$. We verify that $\sigma>0$ and $d_2(Y_T,G) \leq \frac{\sqrt{4 \log n}+2}{2\sigma}$ for sufficiently large $T$. Then, 
\[
d_K(Y_T,G) \leq 2 \left ( \frac{\sqrt{4 \log n}+2}{\sigma}\right)^{2/3} (4 C(n))^{1/3}T^{-1/6}.
\]
\label{Kdist}
\end{proposition}

Now, we are ready to state our main results. The following theorem provides a bound on the matrix spectral norm of the difference between the count matrix and its expectation as a function of $n$ and $T$. The probability with which the bound holds is a function of $T$, and the bound can be turned into a high probability bound by letting $T \to \infty$.
\begin{theorem}
\label{thm:spectral_norm_bound}
Let $\b N_T$ be the $n \times n$ count matrix of a temporal network generated from a DCH model with parameters $\b \mu, \b \Gamma$. Let $\mu_{\max} = \max_{i,j}\mu_{ij}$. Assume the following. (1) The spectral radius $\rho(\boldsymbol\Gamma) =\sigma^*<1$, 
(2)  For any block pair $(a,b)$, the maximum absolute row and column sums for the submatrices $\b \Gamma_{ab \rightarrow ab}, \b \Gamma_{ab \rightarrow ba}, \b \Gamma_{ba \rightarrow ab}, \b \Gamma_{ba \rightarrow ba}$ are identical and are upper bounded by $\gamma_{\max}>0$ for all $(a,b)$ pairs. 
Define $\E \b N_T = \operatorname{vec}^{-1}\left((\b R\operatorname{vec}(\b \mu)T\right)$.
Then, for all $n>1$ and $T>1$ we have, with probability at least $1-\exp(-\log n \log T)-\frac{\kappa(n)}{T^{1/6}}$, for some $\kappa(n)>0$ which is a function of $n$ but not of $T$,
\[
\sqrt{\frac{T}{\log T}}\left\|\frac{\boldsymbol{N}_{T} -\E \b N_T}{T} \right\| \leq  3(1-\sigma^*)^{-3}\sqrt{n (1 + \gamma_{\max})^3\mu_{\max}(1+2 \log n)} ).
\]
\end{theorem}
The proof of this theorem is given in Appendix \ref{sec11}. The first assumption states that $\rho(\b \Gamma)$, the spectral radius of $\b \Gamma$, is bounded away from 1, which is a necessary condition for the stability of the multivariate Hawkes process. This assumption also ensures the existence of $(\b I - \b \Gamma)^{-1}$.
The second assumption provides control over the amount of dependence in the mutually exciting Hawkes processes. The assumption of identical row and column sums of the submatrices for any block pair is part of the definition of the DCH model as discussed earlier. The parameter $\gamma_{\max}$ upper bounds the total amount of excitation in the conditional intensity function that the node pair $i,j$ can receive from (or send to) all node pairs which exert an influence on it (which consists of all node pairs in block pairs $(a,b)$ and $(b,a)$). The upper bound in the above theorem provides explicit dependence on key model quantities including $n, T, \mu_{\max}$, and $\gamma_{\max}$.

Next, we note that the expected count matrix for the DCH model can be written as a block matrix (which has identical values in the same block).  Note the the matrix $\E \b N_T =\tN=\operatorname{vec}^{-1}\left(\b R\operatorname{vec}(\b \mu)T\right)$.  One can write $ \tN $ as $\Z\B\Z^T$ where $\b Z\in \{0,1\}^{n\times K}$ is as defined before and $\B\in\mathbb{R}^{K\times K}$ is a nonnegative matrix (Theorem 4.1 in \cite{soliman2022multivariate} with the assumptions of the DCH model). The lemma below shows that the column concatenation of singular vectors of $\tN$ can be used to identify the communities.

\begin{lemma}
\label{lem:row_length_X}
For $\tN $ defined above, let $\tN = \t{\b X}_L \b\Lambda \t{\b X}_R^T$ be its singular value decomposition (SVD) where $ \t{\b X}_L, \t{\b X}_R \in \mathbb{R}^{n\times K}$ and $\b \Lambda \in \mathbb{R}^{K\times K}$. Let $\t{\b X} = (\t{\b X}_L|\t{\b X}_R)\in\mathbb{R}^{n\times 2K}$, which is a column concatenation of $\t{\b X}_L$ and $\t{\b X}_R$. Then we have $\t{\b X} = \b Z \b Y$, where $\b Y \in \mathbb{R}^{K\times 2K}$, $\|\b Y_{i\cdot}\| = \sqrt{2n_i^{-1}}$ and $\|{\b Y}_{i\cdot}  - {\b Y}_{j\cdot}\|=\sqrt{2(n_{i}^{-1}+n_{j}^{-1})}$ for any $1\leq i \leq j \leq K$. Moreover, let $\t{\b X}^*$ be the row normalized version of $\t{\b X}$, i.e., ${\t{\b X}}^*_{ij} ={\t{\b X}}_{ij} /  \|\sum_{j} {\t{\b X}}_{ij}\|$. Then $\t{\b X}^*=\b Z\b Y^*$, where $\b Y^*$ is the row normalized version of $\b Y$, and $\|{\b Y}^*_{i\cdot}  - {\b Y}^*_{j\cdot}\|=\sqrt{2}$ for any $1\leq i \leq j \leq K$.
\end{lemma}

The proof of this lemma is given in Appendix \ref{l4proof}. It is clear from Lemma \ref{lem:row_length_X} that $z_i = z_j$ if and only if $\t{\b X}_{i\cdot}^*=\t{\b X}_{j\cdot}^*=\b Y_{z_i\cdot}^*$. Recall $z_i$ gives the community label for $i$ and hence if $z_i=q$, then $\b Y_{z_i\cdot}^*$ denotes the $q$th row of $\b Y^*$. Therefore, applying some clustering algorithm (e.g., k-means) on the rows of the matrix $\t{\b X}^*$ can return a perfect clustering result. However, we cannot get $\t{\b X}^*$ in practice since $\tN$ is not observed. 
A variation of the Davis-Kahan Theorem, which we state  and prove in Appendix \ref{sec:davis_kahan}, lets us derive an upper bound on the misclustering rate if we apply $(1+\varepsilon)$-approximate k-means algorithm (\cite{kumar2004simple}) on the rows of $\t{\b X}^*$. We define the misclustering error rate as 
$r=\inf _{\Pi} \frac{1}{n} \sum_{i=1}^{n} 1\left(z_{i} \neq \Pi\left(\hat{z}_{i}\right)\right)$
where we take the infimum over all permutations $\Pi(\cdot)$ of the community labels. We further define $n_{\max} = \max_{1\leq a\leq K} n_{a}$, the number of nodes in the largest community.  The following theorem is the  main result of this paper.

\begin{theorem}
\label{thm:misclustering_rate}
Let $\b N_T$ be the count matrix of a temporal network generated from a DCH model with parameters $\b \mu, \b \Gamma$. We use $ \lambda_1 \geq \dots \geq \lambda_K >0$ to denote the top $K$ singular values of $\frac{\E{\b N}_T}{T}$. Under the assumptions of Theorem \ref{thm:spectral_norm_bound}, the misclustering rate of community detection using Algorithm \ref{alg:spectral_clustering} applied to $\b N_T$ is
\begin{align*}
        r
        &\leq \left(\frac{\log T}{T}\right) \frac{1440(2+\varepsilon)^2 n_{\max} K}{ \lambda_{K}^{2}}\left( 
        (1-\sigma^*)^{-6} {(1 + \gamma_{\max})^3\mu_{\max}(1+2 \log n)}\right).
\end{align*}
with probability at least $1-\exp(-\log n \log T)- \frac{\kappa(n)}{T^{1/6}}$ for any $n>K$ and $T>1$.

\end{theorem}

The proof of this theorem is provided in Appendix \ref{t5proof}. The above result provides a scaling for the misclustering rate  that involves $n, K, T$ and the parameter $\gamma_{\max}$, which controls the amount of dependence across pairs of Hawkes processes.  We also note that, while this result is non-asymptotic in $n$ and $T$, we can also let $T \to \infty$, and then the upper bound holds with probability at least $1-o(1)$. In particular, the upper bound implies that $ r \overset{p}{\to}0$ as $T \to \infty$. In order to understand the dependence of the error rate on the model parameters more clearly, we consider a simplified special case of a DCH model next.

\subsection{Results for Specific DCH Models: MULCH,  SR, and CHIP}
\label{sec:simple_mulch}
We define a \textit{simplified symmetric MULCH model (SS-MULCH)}, which is a special case of the MULCH model \citep{soliman2022multivariate} defined in Section \ref{sec:dch_examples} and is part of the DCH class of models. In this SS-MULCH model, the within community parameters are the same and the between community parameters are the same, i.e., for any $1\leq a,b\leq K$ and $a\neq b$,
\begin{equation}
\label{eq:symmetric_general_model}
\begin{split}
M_{aa} = \mu_1, ~\alpha^{n}_{aa} =\alpha^{n}_1,~\alpha^{r}_{aa} =\alpha^{r}_1,~ \alpha^{tc}_{aa} =\alpha^{tc}_1, ~ \alpha^{ac}_{aa} =\alpha^{ac}_1, ~ \alpha^{gr}_{aa} =\alpha^{gr}_1,~ \alpha^{ar}_{aa} =\alpha^{ar}_1,\\
M_{ab} = \mu_2, ~\alpha^{n}_{ab} =\alpha^{n}_2,~\alpha^{r}_{ab} =\alpha^{r}_2,~ \alpha^{tc}_{ab} =\alpha^{tc}_2, ~ \alpha^{ac}_{ab} =\alpha^{ac}_2, ~ \alpha^{gr}_{ab} =\alpha^{gr}_2,~ \alpha^{ar}_{ab} =\alpha^{ar}_2.
\end{split}
\end{equation}
We do not add restrictions on the decay parameters $\beta$ here since it will not influence our results. We assume all blocks are of equal size, i.e., containing  $(n/K)$ nodes. First, from the construction of $\b \Gamma_{n^2 \times n^2}$ matrix, we can infer that, for any $1\leq a\leq K$, the block matrix $\b \Gamma_{(a,a),(a,a)}$ has identical row sum $ \gamma_1$, and for any $1\leq a < b\leq K$, the block matrix $\b \Gamma_{(a,b),(b,a)}$ has identical row sum $ \gamma_2$. Given that every block contains $n_a = \frac{n}{K}$ nodes, a row in $\b \Gamma_{(a,a),(a,a)}$ contains  $n_a^2=(\frac{n}{K})^2$ elements while a row in $\b \Gamma_{(a,b),(b,a)}$ contains $2n_a n_b = 2(\frac{n}{K})^2$ elements. In the SS-MULCH model, we can compute the row sums as follows. The row sum for $\b \Gamma_{(a,a),(a,a)}$ is 
\[
         \gamma_1 = \alpha_1^n + \alpha_1^r +(n^2/K^2-2)(\alpha_1^{ac}+\alpha_1^{tc} +\alpha_1^{gr} + \alpha_1^{ar}),
         \] 
         The row sum for $\b \Gamma_{(a,b),(b,a)}$ is given by 
\[\gamma_2 
         = \gamma_{ab\rightarrow ab} + \gamma_{ba\rightarrow ab}
        = \alpha_2^n + \alpha_2^r +(n^2/K^2-1)(\alpha_2^{ac}+\alpha_2^{tc} + \alpha_2^{gr}+\alpha_2^{ar}).
        \]
        
Since $\b \Gamma$ is a block diagonal matrix, we know that $\sigma^* = \rho(\b\Gamma) = \max_{1\leq a\leq b\leq K}\rho(\b \Gamma_{(a,b), (b,a)})$. Then, using Proposition \ref{lem:bound_eig} (in Appendix \ref{sec:props_lemmas}) and noting that the row sums are identical, we can further see that $\rho(\b \Gamma_{(a,a),(a,a)}) = \gamma_1$ and $\rho(\b \Gamma_{(a,b),(b,a)}) = \gamma_2$. Therefore, $ \sigma^*= \max\{ \gamma_1, \gamma_2\}$. By the definition of the $\gamma_{\max}$ in Theorem \ref{thm:spectral_norm_bound}, we note that we can set $\gamma_{\max}$ such that $\max\{ \gamma_1, \frac{\gamma_2}{2}\} \leq \gamma_{\max} \leq \max\{ \gamma_1, \gamma_2\}$. Consequently, $\sigma^*/2 \leq \gamma_{\max} \leq \sigma^*$. In order to ensure stability of the process, we need to further assume $\sigma^* <1$. 
With these results we have the following corollary.

\begin{corollary}
\label{cor:simplified_model_misclustering_error}
For the simplified symmetric MULCH (SS-MULCH) model, under the same assumptions as in Theorem \ref{thm:misclustering_rate}, the misclustering rate is
\begin{align*}
   r \leq  \frac{ cK^2\mu_{\max}\left(1-\sigma^{*}\right)^{-6}\left(1+\gamma_{\max }\right)^3}{((1 -   \gamma_1)^{-1}\mu_1- (1 -   \gamma_2)^{-1}\mu_2)^2} 
    \left(
        \frac{\log T (1+2 \log n)}{nT}  
        \right),
\end{align*}
for a constant $c>0$, with probability at least $1-\exp(-\log n \log T)- \frac{\kappa(n)}{T^{1/6}}$.
\end{corollary}

 Note that since the above relation between $\gamma_{\max}$ and $\sigma^*$ implies that $\sigma^* \leq 2\gamma_{\max}$, assuming $\gamma_{\max}<1/2$ guarantees $\sigma^*<1$ ensuring the stability of the process. With this assumption, $ (1-\sigma^*)^{-1}$ is a constant that does not depend on $n, T$.
  We define a function $
  h(\gamma_1,\gamma_2,\mu_1,\mu_2) = \left( (1-\gamma_1)^{-1} -(1-\gamma_2)^{-1}\frac{\mu_2}{\mu_{1}}\right)^2$. We assume $n$ is large enough that $2 \log n >1$, and without loss of generality assume $\mu_1>\mu_2$ and hence $\mu_{\max}=\mu_1$.
  Then from Corollary \ref{cor:simplified_model_misclustering_error}, we have
\begin{align}
\label{eq:r_simple_symmetric_model}
   r\leq \frac{c (1+\gamma_{\max})^3}{h(\gamma_1,\gamma_2,\mu_1,\mu_2)(1-\sigma^*)^6}\left(\frac{K^2\log n \log T}{ nT\mu_{\max}} \right),
\end{align}
where $c$ absorbs numerical constants that do not depend on model parameters. First we note that $\gamma_{\max}$ which upper bounds the total influence a node pair receives from other node pairs appears in the upper bound. In particular the misclustering rate upper bound increases as $\gamma_{\max}$ increases. If $\gamma_{\max}$ becomes close to 1 and consequently $\sigma^*$ becomes close to 1, then the misclustering rate bound blows up. We note that $r$ also depends on $h(\gamma_1,\gamma_2,\mu_1,\mu_2)$. This means if the expected counts of within and between block are indistinguishable, then the misclustering error rate can be very large. In addition, we note that the upper bound increases quadratically with increasing $K$, and decreases with increasing $n$, $T$, and $\mu_{\max}$. We observe some of these dependencies in our finite sample simulations as well in Section \ref{sec:exp_spectral_clustering}.

Now, turning our attention to asymptotic rates, we let $T \to \infty$, while keeping $n$ fixed. To simplify our presentation and focus on the dependency on $n, K, \mu_{\max}$ and $T$, we assume $h(\gamma_1,\gamma_1,\mu_1,\mu_2)$ does not vanish and is a constant as $T \to \infty$. Then 
\begin{equation}
\label{eq:misclustering_rate_symmetric_general_model}
    r \lesssim \frac{K^2\log n \log T}{ nT\mu_{\max}}, \text{ with probability } 1-o(1).
\end{equation}
Note that the expected density of the count matrix varies as $\mu_{\max}T$ when the jump and decay parameters remain constant as a function of $T$. Therefore, consistent clustering requires $\mu_{\max} >> (\frac{K^2 \log n }{n}) \frac{\log T}{T}$.

Notice that in this model, the parameters are ``symmetric" because the parameters for directed block pair $(a,b)$ are the same as the parameters for directed block pair $(b,a)$ (i.e., $\mu_{ab} = \mu_{ab}, \b\alpha_{ab} = \b\alpha_{ba}$, where $\b \alpha_{ab} = \{\alpha_{ab}^n,\alpha_{ab}^r,\alpha_{ab}^{ac},\alpha_{ab}^{tc},\alpha_{ab}^{gr},\alpha_{ab}^{ar}\}$), and hence we must let $\gamma_{\max}  < 1$ to ensure the stability condition ($ \sigma^*<1$) in our discussion above. However, in the ``asymmetric" case, we can have $\gamma_{\max} > 1$. However, in that case $\t{ \lambda}_K$ will have a more complicated form than the result used in Corollary \ref{cor:simplified_model_misclustering_error}. 

Consider a subset of the SS-MULCH model that only consists of self excitation and reciprocal excitation; that is, $\alpha_i^{tc} =\alpha_i^{ac}=\alpha_i^{gr}=\alpha_i^{ar}=0$ for $i=1,2$. 
It is thus a simplified symmetric case of the SR model we introduced in Section \ref{sec:sr_model}. 
Then we have $\gamma_1 = \alpha_{1}^n + \alpha_{1}^r, \gamma_2=\alpha^{n}_2+\alpha^{r}_2$. As long as $\max\{\gamma_1, \gamma_2\} <1$, the result in \eqref{eq:r_simple_symmetric_model} holds and provides an upper bound on the misclustering rate in this case.

\paragraph{Comparison with Prior Results on CHIP Model:}
The CHIP model \citep{arastuie2020chip} only involves self excitation, which also satisfies our conditions, so our results can still be applied on it directly.  
 Unlike \cite{arastuie2020chip}, our results in this paper are non-asymptotic and hold for all $n$ and $T$. While \cite{arastuie2020chip} relied on asymptotic convergence of (univariate) Hawkes process counts to Gaussian limits as $T \to \infty$, we achieve non-asymptotic results by explicitly obtaining a form of the probability with which our upper bound holds. Further, in the DCH models, including MULCH and SR, we allow for more excitation types, so the entries in the count matrix can be dependent, and thus, our misclustering error rate is more general.  The form of the result in \eqref{eq:r_simple_symmetric_model} in terms of dependencies on $n$ and $T$ also qualitatively matches the upper bound for spectral clustering on multilayer and discrete time SBM, e.g., as in \cite{lei2022bias} and \citet{paul2020spectral}.                                   

\paragraph{Relation to Community Detection in Weighted Networks:} We note that, in many application settings involving static networks, the network edges are weighted and directed counts. We put forth the proposed DCH model as a statistical generative model for such ``count" networks. Even though the DCH model is a statistical model for ``observed" relational events data, it can be thought of as an implicit generative model in situations when only a static network with the counts are observed.  Our theoretical results in Theorems \ref{thm:spectral_norm_bound} and \ref{thm:misclustering_rate} and the discussions in this section provide useful indicators of the accuracy of spectral clustering for a weighted network.

\section{Parameter Estimation in the Restricted SR Model}
\label{sec:gmm_estimation}
For the models in the DCH family, the parameters can be estimated from the event times by maximizing the multivariate Hawkes process log likelihood function. However, directly maximizing the likelihood function is slow and hard to scale to large datasets. For some simpler models within the DCH family, it is possible to develop estimators based on the Generalized Method of Moments (GMM) approach using relatively lower-order moments of the aggregate counts. This approach might not be appropriate for more complex models that require higher-order moments since the higher-order sample moments are highly unstable. However, a researcher might be willing to trade off model fit with computational efficiency. 
For example, \cite{arastuie2020chip} propose a moment-based estimator for the baseline parameters $\b M$ and jump parameters $\b \Gamma$ in the CHIP model, which utilizes only self excitation.

We develop a GMM procedure for the restricted version of the SR model proposed in Section \ref{sec:restricted_sr}, which shares a reciprocal excitation parameter $\alpha_{ab}^r$ between block pairs $(a,b)$ and $(b,a)$.
The GMM for this restricted SR model can efficiently and accurately estimate $\b M$ and $\b \Gamma$, so that we only need to maximize the likelihood function if we want to estimate the decay parameters $\b \beta$. Therefore, the GMM step reduces the parameter dimension when maximizing the likelihood function, making the algorithm faster.

\cite{achab2017uncovering} proposed a GMM method for multivariate Hawkes processes. 
Our method and theoretical results below differ from those of \cite{achab2017uncovering} in terms of the information utilized to compute the sample moments. While \cite{achab2017uncovering}'s method estimates the parameters from a single multivariate Hawkes process by estimating sample mean and covariance from the count time series, we leverage i.i.d.~replicates of bivariate Hawkes processes at the level of node pairs in a block pair to estimate those quantities. Therefore, our results are of a different nature.  Further, while \cite{achab2017uncovering} \textit{assumed} the identification condition (Assumption 1 in Theorem 2.1) necessary for GMM procedure to work, we explicitly \textit{prove} it under the restricted SR model in Lemma \ref{lem:id}. In general, one needs to verify that for a multivariate Hawkes process the identification condition will be satisfied by the parameters of the process. We view that not all models under the DCH family will satisfy the identification condition, and therefore, the GMM is not feasible for all models. However, as we show in Lemma \ref{lem:id}, the restricted SR model satisfies the conditions.

In the restricted SR model defined in Section \ref{sec:restricted_sr}, for block pairs $(a,b)$ and $(b,a)$ with $a\neq b$, we have the following set of unknown baseline and excitation parameters:  $M_{ab},M_{ba}, \alpha^n_{ab}, \alpha^n_{ba},\alpha^r_{ab}$.
Recall that $\operatorname{vec}(\b N_t)\in \mathbb{R}^{n^2}$ is the vector form of the count matrix  at time $t$ ordered according to the set $\mathcal{A}$.
From the results of \cite{achab2017uncovering}, for node pairs $(i,j)$  in $\mathcal{A}$, we can define the first and second order integrated cumulants by 
\[\Lambda_{ij} \, d t =\mathbb{E}\left(d  (\b N_{t})_{ij}\right)\]
and 
\[C_{ij,ji} \, d t =\int_{\tau \in \mathbb{R}}(\mathbb{E}\left(d (\b N_{t})_{ij} \; d (\b N_{t+\tau})_{ji}\right)-\mathbb{E}\left(d (\b N_{t})_{ij}\right) \mathbb{E}(d (\b N_{t+\tau})_{ji})),\]
where $\b \Lambda$ is the mean intensity of the Hawkes process, and $\b C$ is the integrated covariance density. \cite{achab2017uncovering} showed that there is an explicit relationship between these integrated cumulants and the parameters of the multivariate Hawkes process.

In the restricted SR model, we have the following cumulant relationship equations for the block pair parameters. Define 
\begin{equation*}
\b M_{(a,b),(b,a)}=\begin{pmatrix}
M_{ab}  \\
M_{ba} \\
\end{pmatrix}
\text{ and }
\,\b \Gamma_{(a,b),(b,a)} = \begin{pmatrix}
\alpha^n_{ab} & \alpha^r_{ab} \\
\alpha^r_{ab} & \alpha^n_{ba} \\
\end{pmatrix}.     
\end{equation*}
Clearly, estimating the parameters of the SR model is equivalent to estimating the parameter matrices $\b M$ and $\b \Gamma$ for all $(a,b), (b, a)$ pairs.

Define $\b R_{(a,b),(b,a)} = \left(\b I_{2 \times 2} -\b \Gamma_{(a,b),(b,a)}\right)^{-1}.$ Then, for any $(i,j)$ such that $X(i,j)=(a,b)$, we abuse notation slightly to let $\Lambda_{ij}=\Lambda_{ab}$ and $C_{ij,ji}=C_{ab,ba}$ and write the following relations for the $(i,j)$ and $(j,i)$ node pairs together:
\begin{gather}
\b \Lambda_{(a,b),(b,a)} = \left(\begin{array}{l}
\Lambda_{ab} \\
\Lambda_{ba}
\end{array}\right)= \b R_{(a,b),(b,a)} \b M_{(a,b),(b,a)}, \label{eqlambda}\\
\b C_{(a,b),(b,a)}=\begin{pmatrix}
C_{ab,ab} & C_{ab,ba} \\
C_{ba,ab} & C_{ba,ba} \\
\end{pmatrix}
= \b R_{(a,b),(b,a)} \text{ diag}(\b \Lambda_{(a,b),(b,a)})\b  R_{(a,b),(b,a)}^T.
\label{eqC}
\end{gather}

Therefore, for each block pair $(a,b),(b,a)$, if we can estimate the population cumulants $\b \Lambda_{(a,b),(b,a)}$ and $\b C_{(a,b),(b,a)}$, then we can solve the above set of equations and solve for $\b M_{(a,b),(b,a)}$ and $\b \Gamma_{(a,b),(b,a)}$. This estimation method is widely known as the Generalized Method of Moments (GMM) \citep{hall2004generalized}. Recall that, for each block pair $(a,b),(b,a)$, we observe a collection of bivariate counting processes given by $\{(\b N_t)_{ij}: t \in [0,T], X(i,j)=(a,b)\}$. Then, we define the corresponding sample moments as follows:
\begin{equation}
\label{eq:GMM_sample_statistics}
\begin{gathered}
\widehat{\Lambda}_{ab}=  \sum_{X(i,j)=(a,b) } \frac{(\b N_T)_{ij}}{Tn_{ab}}, \quad \quad 
\widehat{\Lambda}_{ba}=  \sum_{X(i,j)=(b,a) } \frac{(\b N_T)_{ij}}{Tn_{ab}}, \\
\widehat{C}_{ab,ab} =\sum_{X(i,j)=(a,b) } \frac{1}{Tn_{ab}}\left((\b N_T)_{ij} - \widehat{\Lambda}_{ab}\right)^2,\\
\widehat{C}_{ba,ba} =\sum_{X(i,j)=(b,a) } \frac{1}{Tn_{ab}}\left((\b N_T)_{ij} - \widehat{\Lambda}_{ba}\right)^2,\\
\widehat{C}_{ba,ab}=\widehat{C}_{ab,ba}=\sum_{X(i,j)=(a,b)}   \frac{1}{Tn_{ab}}\left((\b N_T)_{ij} - \widehat{\Lambda}_{ab}\right)\left((\b N_T)_{ji} - \widehat{\Lambda}_{ba}\right).\\
\end{gathered}
\end{equation}
Here, $n_{ab}=n_a n_b$ is the number of pairs of nodes with one node being in community $a$ and the other node in community $b$. Note that, unlike the method in \cite{achab2017uncovering}, the above sample moments only uses aggregate counts at time $T$ and takes sample means over $n_{ab}$ pairs of Hawkes processes. 

Solving the cumulant relationship equations directly may be difficult, so we use a least squares method to solve it. We define the function $\b g_n(\b N, \b M_{(a,b),(b,a)}, \b \Gamma_{(a,b),(b,a)} ) \in \mathbb{R}^{5}$ such that the components are defined as
\begin{gather*}    g_{n1}(\cdot,\cdot,\cdot)   =  \Lambda_{ab} - \widehat \Lambda_{ab},  \quad 
    g_{n2}(\cdot,\cdot,\cdot) =  \Lambda_{ba} - \widehat \Lambda_{ba}, \\  
    g_{n3}(\cdot,\cdot,\cdot)   =  C_{ab,ab} - \widehat C_{ab,ab},\quad
    g_{n4}(\cdot,\cdot,\cdot)  =  C_{ba,ba} - \widehat C_{ba,ba},\quad
    g_{n5}(\cdot,\cdot,\cdot) = C_{ab,ba} - \widehat C_{ab,ba}.
\end{gather*}
Then, our GMM estimator $(\h{\b M}_{(a,b),(b,a)}, \h{\b \Gamma}_{(a,b),(b,a)})$ is the minimizer of the following optimization problem:
\begin{equation}
\label{eq:GMM_minization}
\underset{\b \Theta_{(a,b),(b,a)}}{\min}
    \b  g_n(\b N, \b M_{(a,b),(b,a)}, \b \Gamma_{(a,b),(b,a)} )^T \b g_n(\b N, \b M_{(a,b),(b,a)}, \b \Gamma_{(a,b),(b,a)} ),
\end{equation}
where $\b \Theta_{(a,b),(b,a)} $ is the feasible parameter space in the restricted SR model given by
\begin{align}
\label{eq:restricted_SR_feasible_set}
\bigg \{\b M_{(a,b),(b,a)}, \b \Gamma_{(a,b),(b,a)} : \rho({\b G}_{(a,b),(b,a)})\leq \sigma^*<1, M_{ab},M_{ba} >0,\! \text{ and }  \alpha^n_{ab}, \alpha^n_{ba},\alpha^r_{ab} \geq 0  \bigg \}.
\end{align}
Here, $\rho({\b \Gamma}_{(a,b),(b,a)})\leq \sigma^*<1$ is the stability condition as defined before. For notational convenience, henceforth we will use $\b \theta$ to denote $\b M$ and $\b \Gamma$ together.  

\subsection{Results for the Restricted SR Model}
\label{sec:sr_theory}
For the restricted SR model, we can explicitly state the stability condition in terms the parameters of the model as below:
\begin{lemma}
\label{lem:stability}
(Stability condition for the restricted SR model) The restricted SR model is stable if, for any block pair $(a,b)$, the $\b \Gamma_{(a,b),(b,a)}$ matrix has spectral radius $\rho\left(\b \Gamma_{(a,b),(b,a)}\right) < 1$, which is equivalent to $\alpha^n_{ab}  \leq \sigma^*< 1, \alpha^n_{ba} \leq \sigma^*<1$, and $\alpha^r_{ab} \leq \sigma^*<\sqrt{ {(\sigma^*-\alpha^n_{ab})(\sigma^*-\alpha^n_{ba})}}$.
\end{lemma}

Let $\b \theta_0 = \{\b M_0, \b \Gamma_0\} \in \b \Theta$ be the true parameters. Further, let  $\b g_0 (\b \Theta)$ be the population version of the GMM function defined in \eqref{eq:GMM_sample_statistics} obtained by replacing $\hat{\b \Lambda}$ with $\b \Lambda_0$ and $\hat{\b C}$ with $\b C_0$, where $\b \Lambda_0$ and $\b C_0$ are in turn obtained from \eqref{eqlambda} and \eqref{eqC} with $\b M_0$ and $\b \Gamma_0$. The next lemma shows that the true parameter can be identified from this population function $\b g_0$.
\begin{lemma}
\label{lem:id}
    (Identification result) For the restricted SR model, $\b g_0(\b \theta)=\b 0$ if and only if $\b \theta = \b \theta_0$  for all block pairs $(a,b)$.
\end{lemma}

Next, we show that the GMM estimator will converge in probability to the true parameters under an asymptotic regime where $T\rightarrow \infty$ and $n_{ab}\rightarrow \infty$ for any block pair $(a,b)$. Note that our procedure leverages the availability of event counts from $n_{ab}$ bivariate Hawkes processes to construct the sample moments, and hence, our asymptotic framework is in terms of both increasing $T$ and $n_{ab}$. However, we emphasize that we do not require the Hawkes process counts to converge to a limiting Gaussian distribution, which may not hold for growing dimension Hawkes processes or simultaneously for infinitely many bivariate Hawkes processes unless $T$ grows much faster than the dimension or number of Hawkes processes.  Since the parameters are estimated by block pair, we prove the result for any generic block pair $(a,b)$. 
(We switch notation to use superscripts to denote block pairs when we also have the subscript $0$ to denote the true parameter, e.g., ${\b M}_{0}^{(a,b),(b,a)}$.)
The theorem is proved in Section \ref{sec:app_gmm_proofs} by verifying the sufficient conditions laid out in \cite{newey1994large} for the GMM estimator to be consistent.

\begin{theorem}
\label{thm:gmm}
Consider any block pair $(a,b)$ in the restricted SR model. Let the parameter space $\b \Theta_{(a,b),(b,a)}$ defined in \eqref{eq:restricted_SR_feasible_set} be compact and contain the true parameters ${\b M}_{0}^{(a,b),(b,a)}$, $\b \Gamma_{0}^{(a,b),(b,a)}$. Then the estimator $(\h{\b M}_{(a,b),(b,a)}, \h{\b \Gamma}_{(a,b),(b,a)})$ defined in \eqref{eq:GMM_minization} will converge to the true parameters
in probability as $T\rightarrow \infty$ and $n_{ab}\rightarrow \infty$.
\end{theorem}

Notice that the dimension of parameters in each block is 5, which is equal to the dimension of $\b g$, so it is possible that $\b g_0(\b \theta_0)=0$ has a unique solution. 
In Lemma \ref{lem:id}, we show that to be the case for the restricted SR model. For the unrestricted SR model and the MULCH model, which have more parameters, this estimating procedure with just the first two moments cannot ensure a unique solution. For those models, an alternative is to consider higher order cumulants as in \cite{achab2017uncovering}. However, the estimators of higher order cumulants may have large variance, and thus, the final estimation from the GMM procedure might be less accurate. Further, it is not immediately clear if the identification condition similar to our Lemma \ref{lem:id}, which the results of \cite{achab2017uncovering} require to hold for a given multivariate Hawkes process, will hold for unrestricted SR or MULCH models.


\subsection[Estimating beta and Local Likelihood Refinement]{Estimating $\b \beta$ and Local Likelihood Refinement}
\label{sec:refinement}
After the $\b \Gamma$ parameters are estimated from the GMM procedure described above, we can estimate $\b \beta$ by maximum likelihood, if desired. 
(An alternative to estimating $\b \beta$ is to assume fixed $\b \beta$ \citep{bacry2015hawkes} or a weighted sum of multiple $\b \beta$ values, as in MULCH and other similar temporal network models \citep{soliman2022multivariate, yang2017decoupling, huang2022mutually}.) 
If we are estimating $\b \beta$, we plug in the GMM estimates of $\b \Gamma$ into the likelihood equation. The likelihood now becomes a function only of $\b \beta$, which is then estimated through maximum likelihood.

We also propose a local likelihood refinement algorithm for the SR model to further improve the community detection and parameter estimation accuracy given the initial estimates of the community assignments and the parameters. Similar procedures are used in the SBMs \citep{gao2017achieving,chen2022global} and Hawkes process network models \citep{junuthula2019block,soliman2022multivariate} literature for obtaining an improved community assignment after the spectral clustering. However, in densely dependent settings, e.g., the MULCH model \citep{soliman2022multivariate}, it is nearly impossible to implement the refinement algorithm on a large dataset due the computational limitation. In contrast, in the SR model, we are able to write the change in log likelihood due to one refinement step in a computationally efficient manner, and consequently, the refinement algorithm can be scaled to large datasets.

The refinement procedure for node $i$ utilizes the initial community assignments \textit{for all other nodes} and Hawkes process parameter estimates to compute the likelihood of node $i$ belonging to the different blocks. Then we assign the node to the block which maximizes the likelihood. We start with the first node (arbitrary order) and repeat this procedure until community assignment of all nodes have been refined. Finally, we re-estimate the parameters using the new community assignment. The full refinement procedure is summarized in Algorithm \ref{algrefine}.
In the SR model, computing the likelihood for node $i$ given the community assignment of all other nodes only involves computing Hawkes process likelihood for events from $i$ and to $i$. Therefore, this computation includes a very small amount of events and thus it is computationally efficient and practical. 

%
\begin{algorithm}[t]
\caption{Local refinement procedure to update community assignments in the SR model. For the restricted SR model, set $\alpha_{ba}^r = \alpha_{ab}^r$ and $\beta_{ba}^r = \beta_{ab}^r$.}
\label{algrefine}
\textbf{Input:} Events time data $\b{E}$; number of blocks $K$; initial Hawkes process parameters $\boldsymbol{\Theta} = (\boldsymbol{M}, \boldsymbol{\alpha}, \boldsymbol{\beta})$; initial community assignment $\boldsymbol{z}$

\textbf{Output:} New membership vector $\boldsymbol{z}$; new Hawkes process parameters $\boldsymbol{\Theta}$

\begin{algorithmic}[1]
\For{each node $i$}
    \State Update membership $z_i$ by:
    \begin{align*}
    z_i &= \argmax_{a \in \{1,\dots,K\}} \sum_{b=1}^K \sum_{\substack{j:z_j=b \\ j\neq i}} \bigg\{
    - M_{ab}T - M_{ba}T \\
    &- \sum_{t_s \in T_{ij}} \left[\alpha_{ab}^{n}\left(1-e^{-\beta^n_{ab}(T-t_s)}\right) + \alpha_{ba}^{r}\left(1-e^{-\beta^r_{ba}(T-t_s)}\right) \right] \\
    &- \sum_{t_s \in T_{ji}} \left[\alpha_{ab}^{r}\left(1-e^{-\beta^r_{ab}(T-t_s)}\right) + \alpha_{ba}^{n}\left(1-e^{-\beta^n_{ba}(T-t_s)}\right)\right] \\
    &+ \sum_{t_s \in T_{ij}} \ln\left[M_{ab} + \alpha_{ab}^{n}\beta^n_{ab} R^{ij \rightarrow ij}_{ab,n}(t_s)
    + \alpha_{ab}^{r}\beta^r_{ab} R^{ji \rightarrow ij}_{ab,r}(t_s)\right] \\
    &+ \sum_{t_s \in T_{ji}} \ln\left[\mu_{ba} + \alpha_{ba}^{n}\beta^n_{ba} R^{ji \rightarrow ji}_{ba,n}(t_s)
    + \alpha_{ba}^{r}\beta^r_{ba} R^{ij \rightarrow ji}_{ba,r}(t_s)\right]
    \bigg\}
    \end{align*}
where $ R^{ij \rightarrow ij}_{ab,n}(t_s) = \sum_{\substack{t_r \in T_{ij}\\ t_r < t_s}} e^{-\beta^n_{ab}(t_s - t_r)}, \quad 
        R^{ji \rightarrow ij}_{ab,r}(t_s) = \sum_{\substack{t_r \in T_{ji}\\ t_r < t_s}} e^{-\beta^r_{ab}(t_s - t_r)}$, and $R^{ji \rightarrow ji}_{ba,n}(t_s)$, $R^{ij \rightarrow ji}_{ba,r}(t_s)$ are defined similarly.
\EndFor
\State Use updated $\boldsymbol{z}$ to re-estimate Hawkes parameters $\boldsymbol{\Theta}$ via GMM and MLE.
\State \Return updated $\boldsymbol{z}$ and $\boldsymbol{\Theta}$
\end{algorithmic}
\end{algorithm}

\section{Simulation Experiments}

\subsection{Community Detection using Spectral Clustering}
\label{sec:exp_spectral_clustering}
We present simulation experiments to analyze the effects of different parameters of the DCH model on the accuracy of spectral clustering to recover the true memberships of the nodes. 
An additional simulation experiment examining sensitivity of Hawkes process parameters on community detection is presented in Appendix \ref{sec:app_sim_community_detection} in the supplementary materials. 
For all experiments, we simulate several relational events datasets, run the spectral clustering method in Algorithm \ref{alg:spectral_clustering} on the count matrix, and then compute the adjusted Rand index (ARI) \cite{hubert1985comparing} between estimated and true node membership vectors. An ARI of 1 indicates perfect community detection, while the ARI has an expectation of 0 for a random assignment. 

\paragraph{Community Detection while Varying $n$, $K$, and $T$: }
We simulate relational events data from the \textit{simplified symmetric MULCH model} (presented in Section \ref{sec:simple_mulch}) while varying two out of the three quantities: number of nodes $n$, number of blocks $K$, and data duration $T$. 
We let the intra-block parameters be 
\[(\mu_{1}, \alpha_{1}^{n}, \alpha_{1}^{r}, \alpha_{1}^{tc}, \alpha_{1}^{ac}, \alpha_{1}^{gr}, \alpha_{1}^{ar}) =(0.005, 0.2, 0.2, 0.05/s_1, 0.05/s_1, 0.05/s_1,  0.05/s_1),\]
and let the inter-block parameters be 
\[
(\mu_{2}, \alpha_{2}^{n}, \alpha_{2}^{r}, \alpha_{2}^{tc}, \alpha_{2}^{ac}, \alpha_{2}^{gr}, \alpha_{2}^{ar}) =  (0.003, 0.1, 0.1, 0.025/s_2, 0.025/s_2, 0.025/s_2,  0.025/s_2),\]
where the parameters are as defined in (\ref{eq:symmetric_general_model}), $s_1 = n/K -2$, and $s_2 = n/K - 1$. We let the decay parameter $\beta=1$ in all kernel functions when simulating the event table. Then, we can easily compute that $\gamma_1 = 0.6$ and $\gamma_2 = 0.3$. Therefore in this setting, the quantities $h(\gamma_1,\gamma_2,\mu_1,\mu_2)$ and $\mu_{\max}$ remain constant as we vary $n,K,T$.

The community detection accuracy averaged over 15 simulations is presented in Figure \ref{RI_matrix}.
As shown in Figure \ref{fig:RI_matrix_n_T}, the adjusted Rand index increases as both $n$ and $T$ increase while fixing $K=4$. That is what we expect from our non-asymptotic analysis.  Intuitively, increasing $T$ can reduce the variance of the count matrix while increasing $n$ improves the spectral clustering accuracy. Similarly, when fixing $n=60$, and varying $K$ and $T$, we can see the negative association between number of blocks $K$ and adjusted Rand index in Figure \ref{fig:RI_matrix_K_T}, while increasing $T$ improves the accuracy. Finally, in Figure \ref{fig:RI_matrix_K_n}, we verify that the adjusted Rand index increases by increasing $n$ and decreasing $K$ while fixing $T$. All these results align with the prediction in Corollary \ref{cor:simplified_model_misclustering_error} and equation (\ref{eq:misclustering_rate_symmetric_general_model}) which states the misclustering error rate varies as $\frac{K^2\log n \log T}{ nT\mu_{max}}$.

\begin{figure}[t]
    \newcommand{\figwidth}{0.325\textwidth}
    \centering
    \hfill
    \begin{subfigure}[c]{\figwidth}
        \centering
        \includegraphics[width=\textwidth]{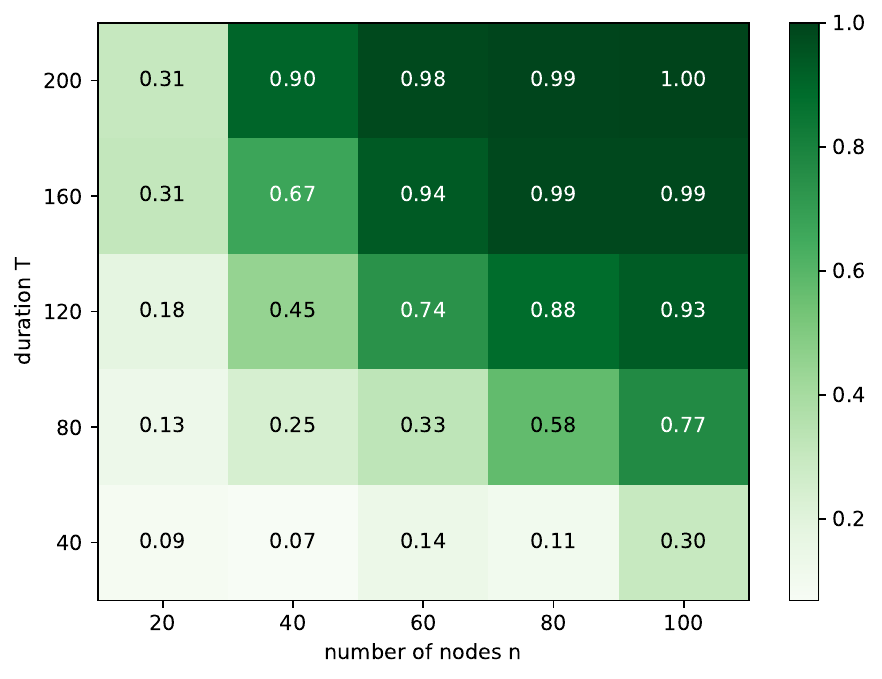}
        \caption{Fixed $K=4$}
        \label{fig:RI_matrix_n_T}
    \end{subfigure}
    \hfill
    \begin{subfigure}[c]{\figwidth}
        \centering
        \includegraphics[width=\textwidth]{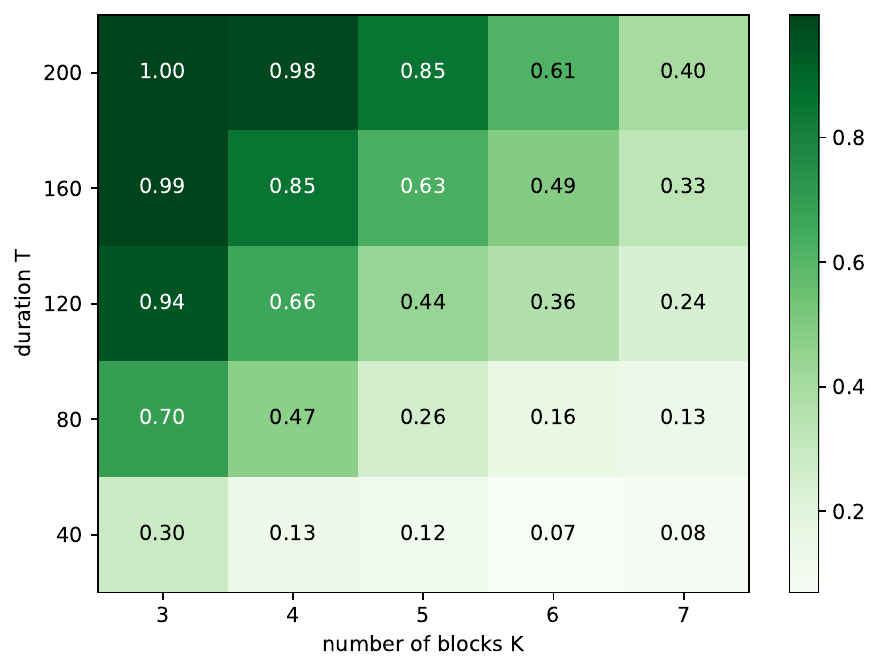}
        \caption{Fixed $n=60$}
        \label{fig:RI_matrix_K_T}
    \end{subfigure}
    \hfill
    \begin{subfigure}[c]{\figwidth}
        \centering
        \includegraphics[width=\textwidth]{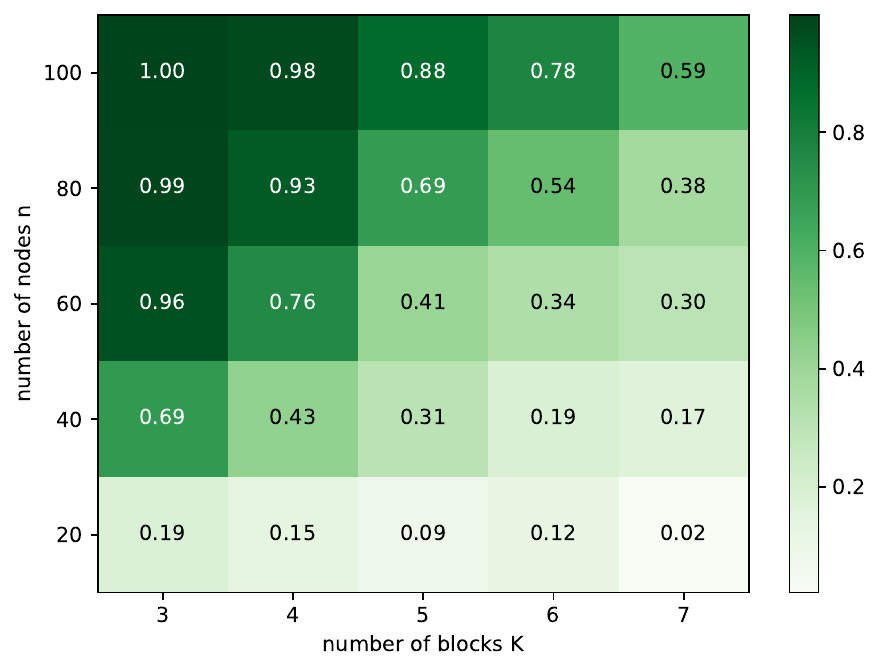}
        \caption{Fixed $T=120$}
        \label{fig:RI_matrix_K_n}
    \end{subfigure}
    \hfill
    \caption{Heat map of adjusted Rand index of spectral clustering with varying $n$, $T$, and $K$, averaged over 15 simulated networks.}
    \label{RI_matrix}
\end{figure}

\paragraph{Effect of $\gamma_{\max}$ on Community Detection Accuracy:}
\label{sec:experiment_gamma_max}
 Theorem \ref{thm:spectral_norm_bound} showed that the spectral norm of the fluctuation of the count matrix from its expectation, $\| \b N_T - \E \b N_T\|$, depends on the sum of excitations $\gamma_{\max}$, and consequently, the spectral clustering error rate in Theorem \ref{thm:misclustering_rate} also depends on $\gamma_{\max}$ in the general model. To numerically evaluate the role of $\gamma_{\max}$, we design a simulation with the self and reciprocal excitation (SR) Hawkes process model (\ref{eq:SR_model}). We let $K=2$ with equal block sizes, and set parameters as 
\begin{equation*}
\b \mu = \left(\begin{array}{cc}
0.002&0.001 - s\\
0.0001 &0.002\\
\end{array}\right),
\quad
\b \alpha^n = \left(\begin{array}{cc}
0&0\\
0&0\\
\end{array}\right),
\quad
\b \alpha^r = \left(\begin{array}{cc}
0&s\\
0&0\\
\end{array}\right),    
\end{equation*}
where $s$ is a scalar. We let all decay parameters $\b \beta=1$. Note that there is no self excitation since $\b \alpha^n$ is a $0$ matrix. The reciprocal excitation $\b \alpha^r$ is controlled by the parameter $s$.

From our definition of $\gamma_{\max}$ in Theorem \ref{thm:spectral_norm_bound}, we know $\gamma_{\max}=s$ in the above setup. When $s=0$, we know all events are based on the base intensity $\b\mu$ since both $\b \alpha^n$ and $\b \alpha^r$ are 0. When $s>0$, then some events are generated by reciprocity. When $T$ is large enough, we can also derive the expectation of the count matrix $\E \b N_T$ (see Section \ref{sec:appendix_gamma_max}) to find that it does not depend on $s$ and has a block structure. 
In this setting, we will only change $s$ and fix $n=40, T=300$, so we know all other parameters ($\sigma^*, \mu_{max},  \lambda_{K}^2,K$) that enter in the expression of Theorems \ref{thm:spectral_norm_bound} and \ref{thm:misclustering_rate} will stay unchanged. 

We show the spectral norm of the difference between sample count matrix and its expectation, and the spectral clustering accuracy over 100 simulations in Figure \ref{fig:gamma_max}. As we see, when we increase $\gamma_{\max}$ by increasing $s$, the spectral norm of error $\|\b N_T - \E \b N_T\|$ increases while the clustering accuracy decreases. These results confirm that our upper bounds in Theorem \ref{thm:spectral_norm_bound} and Theorem \ref{thm:misclustering_rate} are meaningful, and we find that $\gamma_{\max}$ controls the variance of the count matrix. Therefore, increasing $\gamma_{\max}$ will introduce more dependence in the count matrix, which in turn will increase the variance and decrease spectral clustering accuracy.
\begin{figure}[t]
    \newcommand{\figwidth}{0.49\textwidth}
    \centering
    \hfill
    \begin{subfigure}[c]{\figwidth}
        \centering
        \includegraphics[width=\textwidth]{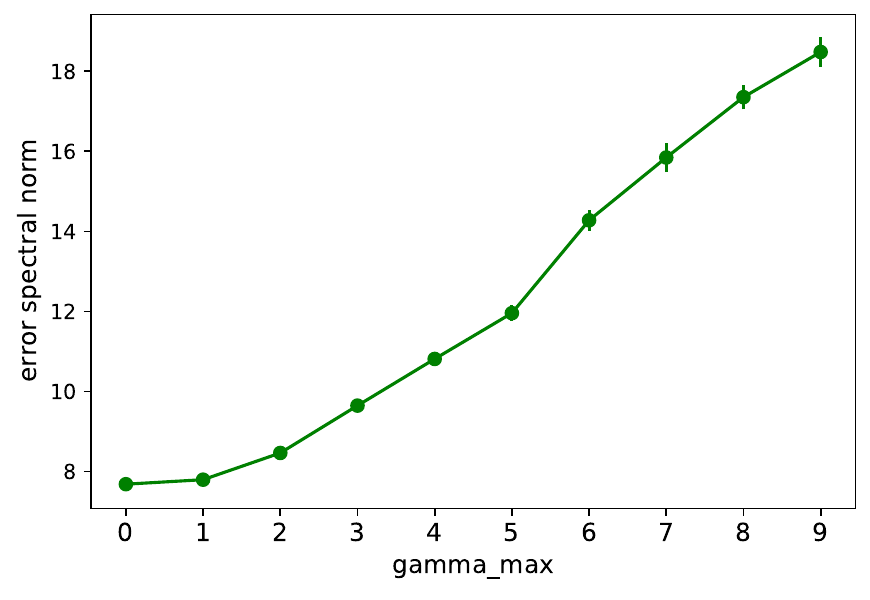}
        \caption{Spectral norm of error: $\|\b N_T - \E \b N_T\|$}
        \label{fig:gamma_max_err}
    \end{subfigure}
    \hfill
    \begin{subfigure}[c]{\figwidth}
        \centering
        \includegraphics[width=\textwidth]{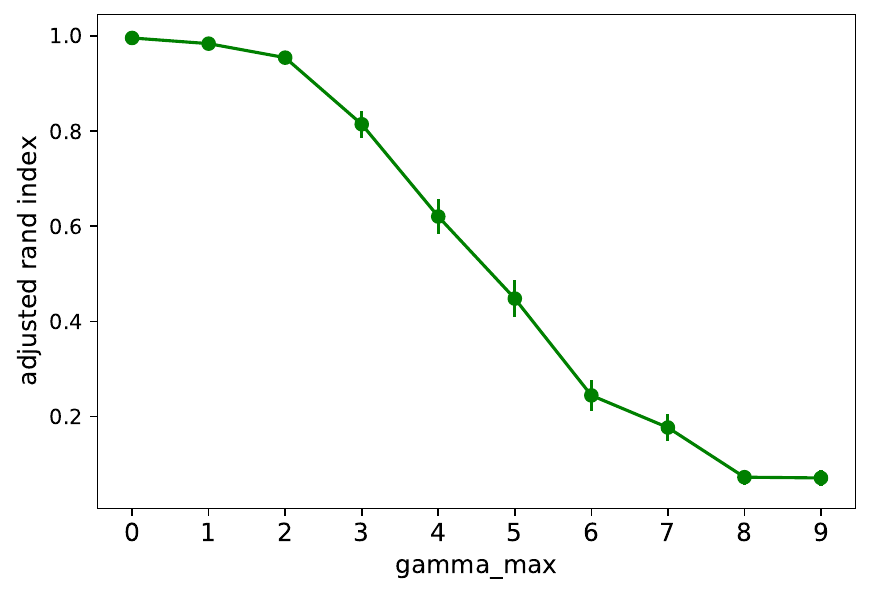}
        \caption{Spectral clustering accuracy}
        \label{fig:gamma_max_sc}
    \end{subfigure}
    \hfill
    \caption {The spectral norm of error and the spectral clustering accuracy with different $\gamma_{\max}$ ($\pm$ standard error over 100 simulated networks). As $\gamma_{\max}$ increases, the spectral norm of the error increases superlinearly while the clustering accuracy decreases.}
    \label{fig:gamma_max}
\end{figure}

\subsection{Accuracy of GMM Estimators}

Next, we examine the parameter estimation accuracy of the GMM procedure for the restricted SR models. We simulate networks from an SR model \eqref{eq:SR_model} with $K=4$, equal block sizes, and the following structured parameters:
for any $1\leq a,b\leq K$ and $a\neq b$, we have $\mu_{aa} = 0.002, ~\alpha^{n}_{aa} =0.2,~\alpha^{r}_{aa} =0.2, ~\beta^{n}_{aa} = 1,~\beta^{r}_{aa} =1$ and $
\mu_{ab} = 0.001, ~\alpha^{n}_{ab} =0.1,~\alpha^{r}_{ab} =0.1,~\beta^{n}_{ab} = 0.5,~\beta^{r}_{ab} =0.5$.
We then run the spectral clustering algorithm followed by the GMM estimation method. We showed in Theorem \ref{thm:gmm} that the GMM estimators will converge to the true parameters as both $n$ and $ T$ go to infinity, and we should see this phenomena in the experiments. 

Figures \ref{fig:MSE_T_mu}-\ref{fig:MSE_T_a_r} show mean squared errors (MSEs) of GMM estimators for $\b \mu$, $\b \alpha^n$ and $\b \alpha^r$ when fixing $n=90$ and varying the observation duration $T$. We observe the MSEs drop very fast when $T$ is increased from 200 to 500, and the clustering error rate reaches close to 0 when $T$ is larger than 500. However, when $T$ is larger than 500, the MSEs drop very slowly. Figures \ref{fig:MSE_N_mu}-\ref{fig:MSE_N_a_r} shows the MSEs when fixing $T=600$ and varying the number of nodes $n$. Also, when $n$ is increased from 40 to 70, the spectral clustering error rate decreases quickly towards 0, and the MSEs also drop fast. But we observe that the MSEs keep dropping as $n$ increases even when $n$ is greater than 70. This is in contrast to the behavior with increasing $T$. Theorem \ref{thm:gmm} requires both $T$ and $n$ go to infinity to ensure the consistency of the estimators, but from these experiments, we conjecture that if both $T, n$ are large enough and the clustering is perfect, increasing $T$ has little effect on improving the GMM estimators accuracy, but increasing $n$ can still reduce the error.

Figures \ref{fig:MSE_T_b_n}-\ref{fig:MSE_T_b_r} (fixed $n=90$ and varying $T$) and Figures \ref{fig:MSE_N_b_n}-\ref{fig:MSE_N_b_r} (fixed $T=600$ and varying $n$) show the MSEs for the kernel parameters estimations $\b \beta^n, \b \beta^r$, which are estimated by the maximum likelihood method. Although we have no theoretical guarantees, we can still see that $\b \beta^n, \b \beta^r$ can also be accurately estimated as $n$ and $T$ both increase.
\begin{figure}[tp]
    \newcommand{\figwidth}{0.192\textwidth}
    \centering
    \hfill
    \centering
    \begin{subfigure}[c]{\figwidth}
        \centering
        \includegraphics[width=\textwidth]{ 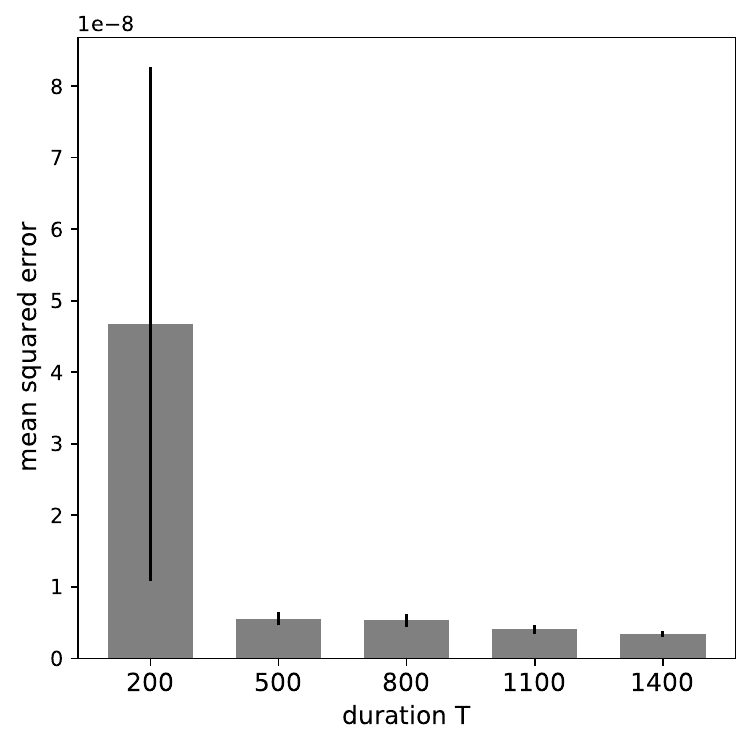}
        \caption{$\b \mu$}
        \label{fig:MSE_T_mu}
    \end{subfigure}
    \hfill
    \centering
    \begin{subfigure}[c]{\figwidth}
        \centering
        \includegraphics[width=\textwidth]{ 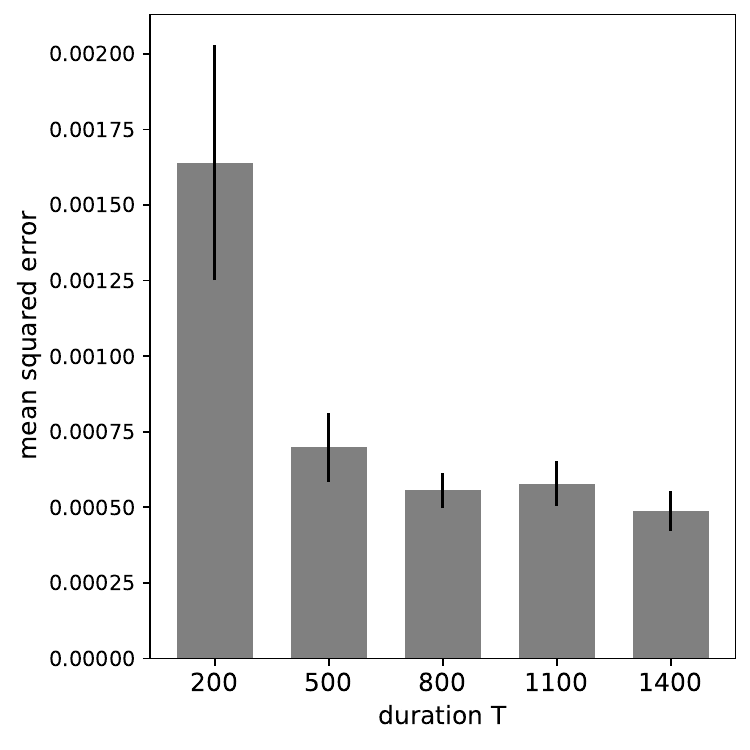}
        \caption{$\b \alpha^{n}$}
    \end{subfigure}
    \hfill
    \centering
    \begin{subfigure}[c]{\figwidth}
        \centering
        \includegraphics[width=\textwidth]{ 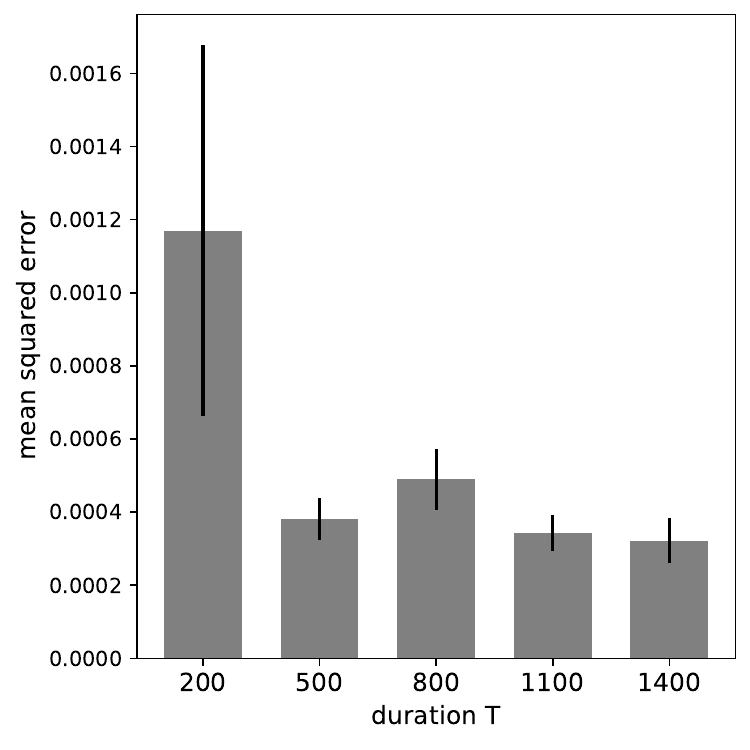}
        \caption{$\b \alpha^{r}$}
        \label{fig:MSE_T_a_r}
    \end{subfigure}
    \hfill
    \begin{subfigure}[c]{\figwidth}
        \centering
        \includegraphics[width=\textwidth]{ 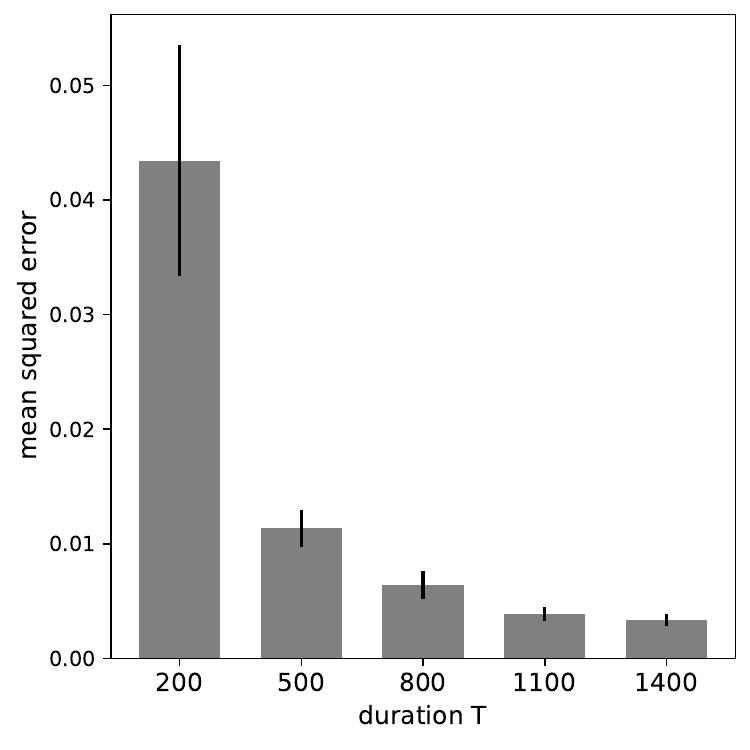}
        \caption{$\b \beta^n$}
        \label{fig:MSE_T_b_n}
    \end{subfigure}
        \hfill
    \centering
    \begin{subfigure}[c]{\figwidth}
        \centering
        \includegraphics[width=\textwidth]{ 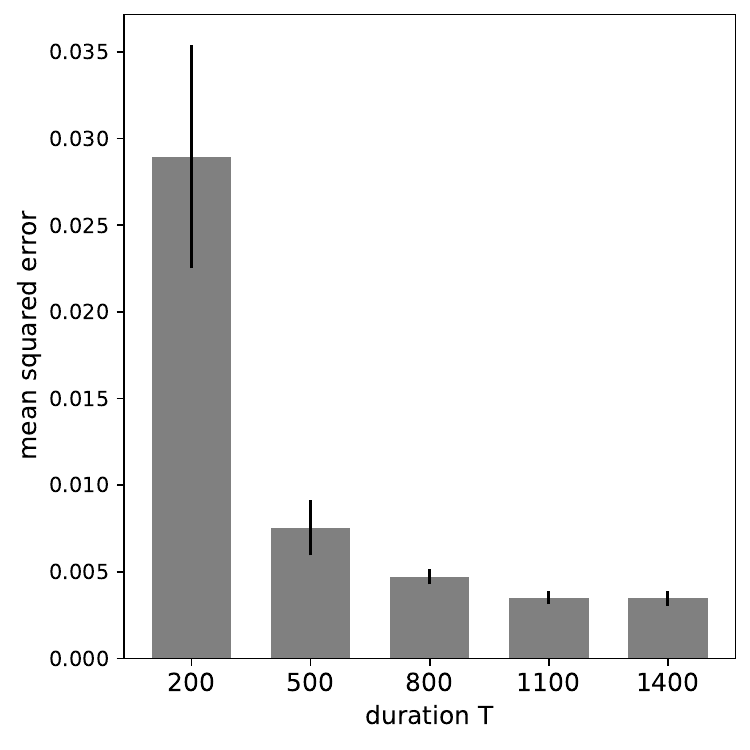}
        \caption{$\b \beta^r$}
        \label{fig:MSE_T_b_r}
    \end{subfigure}
    \hfill
    \\[12pt]
    \centering
    \begin{subfigure}[c]{\figwidth}
        \centering
        \includegraphics[width=\textwidth]{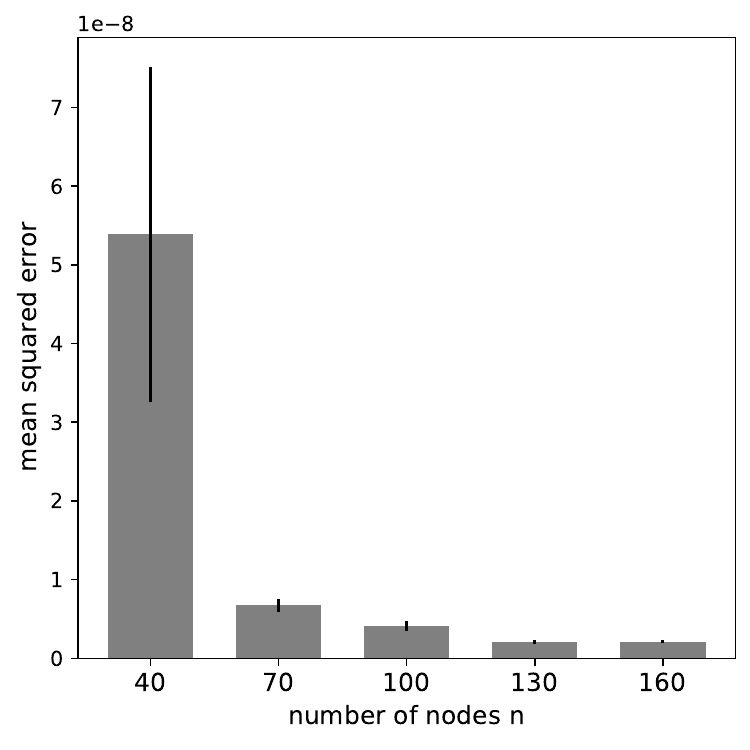}
        \caption{$\b \mu$}
        \label{fig:MSE_N_mu}
    \end{subfigure}
    \hfill
    \begin{subfigure}[c]{\figwidth}
        \centering
        \includegraphics[width=\textwidth]{ 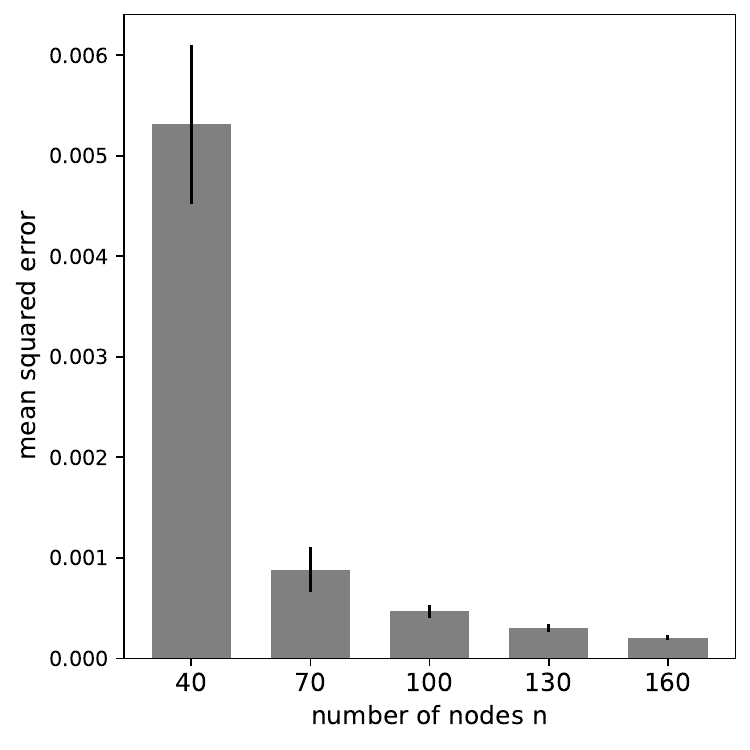}
        \caption{$\b \alpha^n$}
    \end{subfigure}
    \hfill
    \begin{subfigure}[c]{\figwidth}
        \centering
        \includegraphics[width=\textwidth]{ 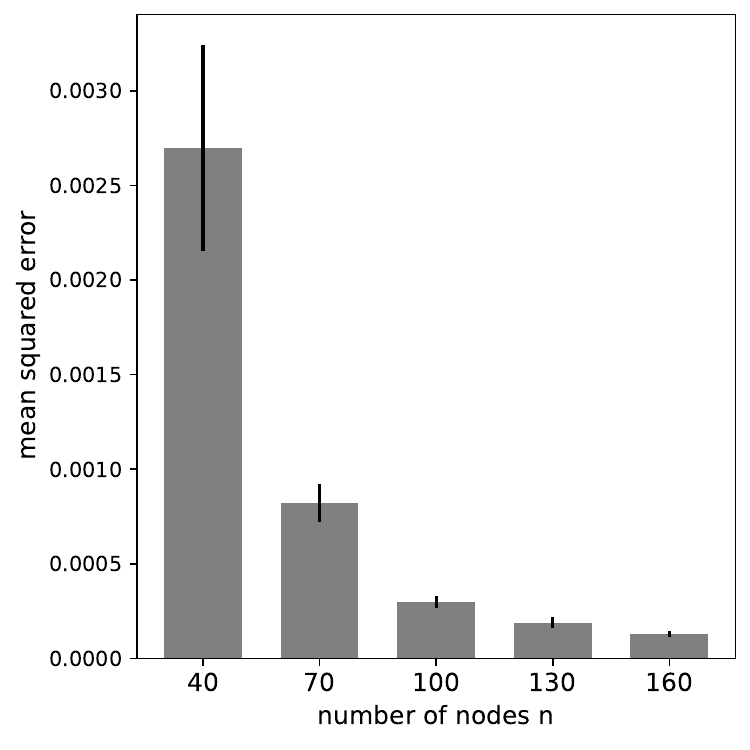}
        \caption{$\b \alpha^r$}
        \label{fig:MSE_N_a_r}
    \end{subfigure}
    \hfill
    \begin{subfigure}[c]{\figwidth}
        \centering
        \includegraphics[width=\textwidth]{ 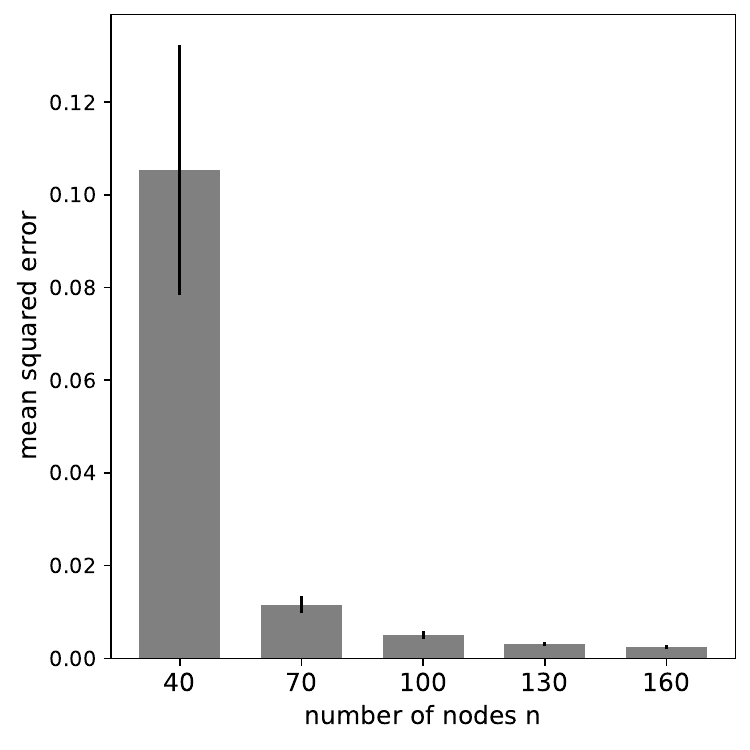}
        \caption{$\b \beta^n$}
        \label{fig:MSE_N_b_n}
    \end{subfigure}
    \hfill
    \begin{subfigure}[c]{\figwidth}
        \centering
        \includegraphics[width=\textwidth]{ 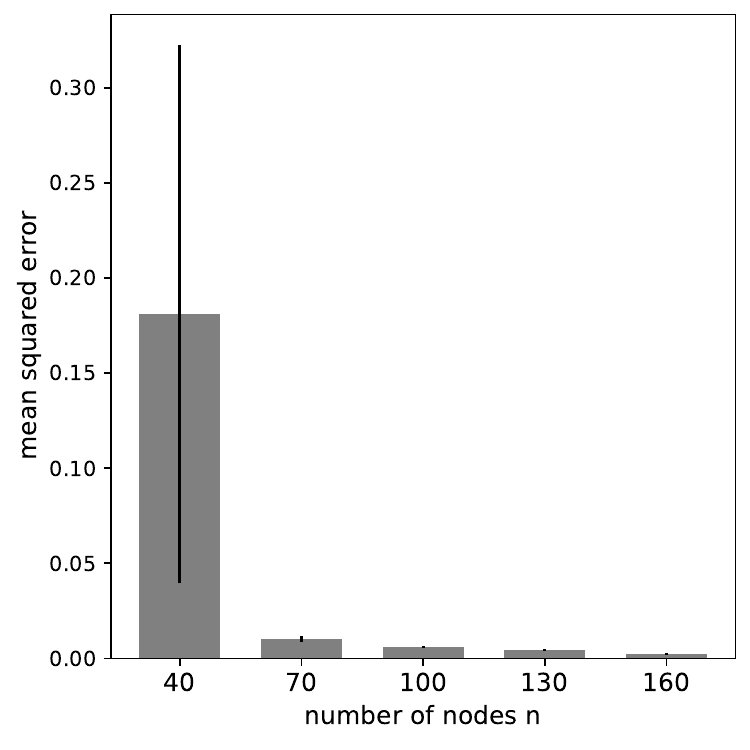}
        \caption{$\b \beta^r$}
        \label{fig:MSE_N_b_r}
    \end{subfigure}
    \hfill
    \caption{Averaged mean squared errors (MSEs) of GMM estimator for $\b \mu$, $\b \alpha^n$, $\b \alpha^r$, and averaged MSEs of maximum likelihood estimator for $\b \beta^n$, $\b \beta^r$ ($\pm$ standard error over 10 runs). (\subref{fig:MSE_T_mu})-(\subref{fig:MSE_T_b_r}) Fixed $n=90$ while varying duration $T$. (\subref{fig:MSE_N_mu})-(\subref{fig:MSE_N_b_r}) Fixed $T=600$ while varying number of nodes $n$. The MSEs for all parameters decrease as $n$ or $T$ decreases.
    }
    \label{fig:MSEs}
\end{figure}

\subsection{Refinement Procedure in the SR Model}

\begin{figure}[t]
    \newcommand{\figwidth}{0.49\textwidth}
    \centering
    \hfill
    \begin{subfigure}[c]{\figwidth}
        \centering
        \includegraphics[width=\textwidth]{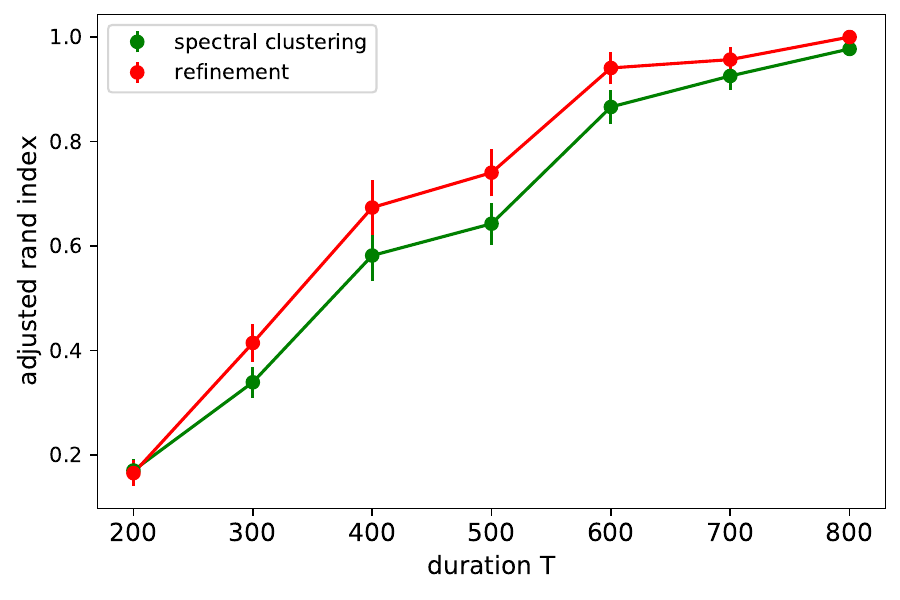}
        \caption{Fixed $n=40$}
        \label{fig:ref_T}
    \end{subfigure}
    \hfill
    \begin{subfigure}[c]{\figwidth}
        \centering
        \includegraphics[width=\textwidth]{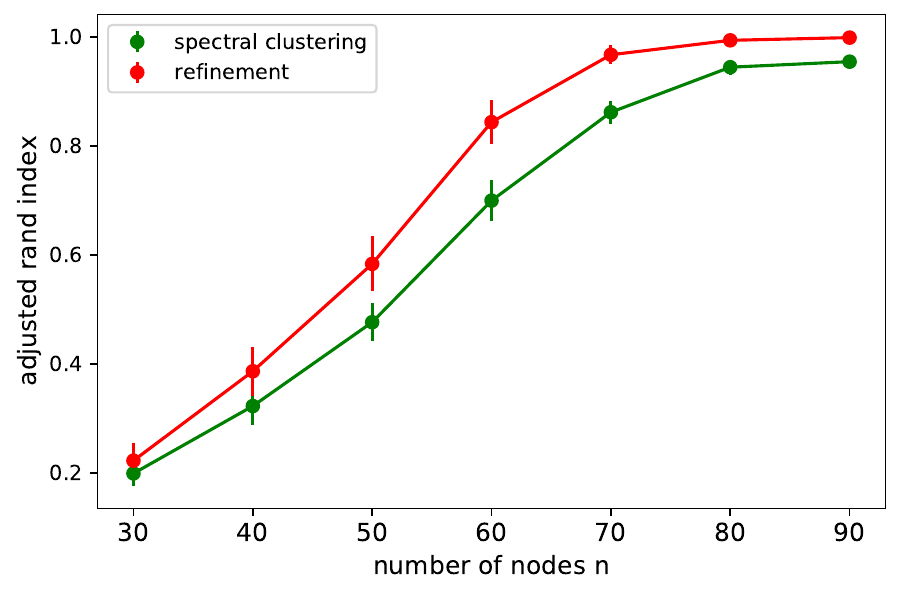}
        \caption{Fixed $T=300$}
        \label{fig:ref_n}
    \end{subfigure}
    \hfill
    \\[12pt]
     \centering
     \hfill
    \begin{subfigure}[c]{\figwidth}
        \centering
        \includegraphics[width=\textwidth]{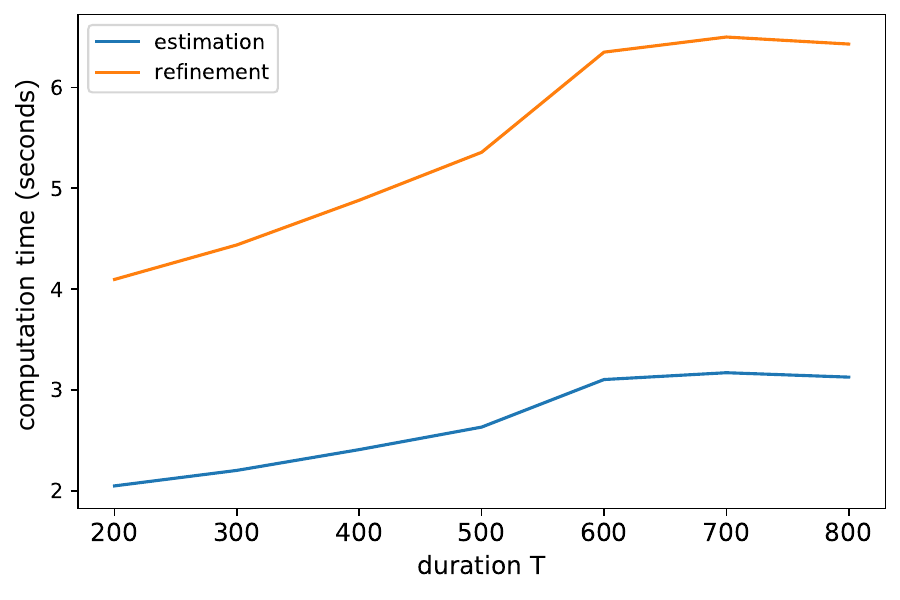}
        \caption{Fixed $n=40$}
        \label{fig:ref_T_time}
    \end{subfigure}
    \hfill
    \begin{subfigure}[c]{\figwidth}
        \centering
        \includegraphics[width=\textwidth]{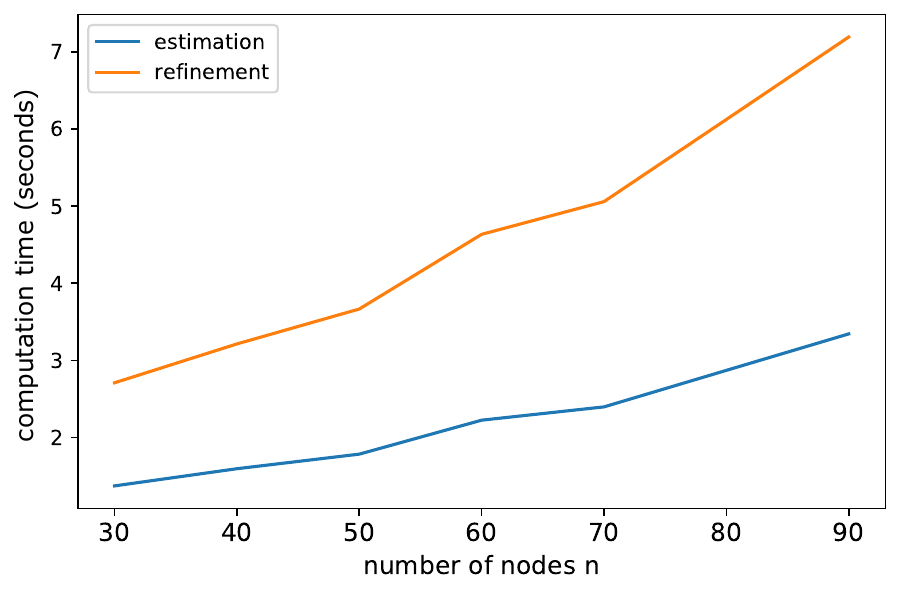}
        \caption{Fixed $T=300$}
        \label{fig:ref_n_time}
    \end{subfigure}
    \hfill
    \caption {(\subref{fig:ref_T})-(\subref{fig:ref_n}) Adjusted Rand index of spectral clustering and the refinement algorithm with varying $T$ and $n$, respectively ($\pm$ standard error over 10 simulated networks). (\subref{fig:ref_T_time})-(\subref{fig:ref_n_time}) Computation time of the spectral clustering + estimation time with and without refinement while varying $T$ and $n$, respectively.}
    \label{fig:refinement}
\end{figure}

We compare the accuracy of community detection and computation time of the spectral clustering algorithm and the refinement algorithm (Algorithm \ref{algrefine}) in Figure \ref{fig:refinement}. For this purpose, we simulate 10 relational event datasets by fixing $n$ and varying $T$ or fixing $T$ and varying $n$. We keep the $\b \mu$ and $\b \alpha^n, \b \alpha^r$ parameters the same as in the previous section. However, we decrease the decay parameters that correspond to the between community excitations by letting $\beta^{n}_{ab} = 0.1,\beta^{r}_{ab} =0.1$, while the intra community decay parameters are still kept at $\beta^{n}_{aa} = 1,\beta^{r}_{aa} =1$. Therefore, we have a substantial difference between the intra block and inter block decay parameters. In our model, the decay parameters do not influence the expected count matrix, so it should not have any significant effect on the performance of spectral clustering. However, the timestamp of the events should bring us more information, and we expect that the refinement procedure which utilizes the likelihood function can improve the community detection results. 

Figures \ref{fig:ref_T}-\ref{fig:ref_n} show we can generally obtain a higher adjusted Rand index after applying the refinement algorithm. We further notice that the improvement is more significant towards the middle of the curves, when the initial clustering has reasonable result but is not perfect. On one hand, when the initial clustering result is very bad, the parameters estimation will be inaccurate, so the likelihood refinement algorithm will also perform poorly. On the other hand, when the initial clustering result is already good enough, there are not too many misclustered nodes left, so the improvement is also limited. 

We further compare the computation time of spectral clustering followed by parameter estimation with the refinement algorithm, which includes the initial spectral clustering, parameter estimation, local refinement, and parameter re-estimation. We can see in Figures \ref{fig:ref_T_time} and \ref{fig:ref_n_time} that 
the refinement process requires approximately double the time compared to the initial spectral clustering and parameter estimation. This shows that our refinement algorithm is practical, and even if we need to re-estimate the parameters, the time complexity is only a constant multiple of the original one.

\section{Real Data Experiments}

We analyze 5 real relational events datasets to evaluate the restricted SR model's predictive ability and computational efficiency. Each dataset consists of a list of events where each event consists of a sender, a receiver and a timestamp. 
Summary statistics for the datasets are shown in Table \ref{tab:dataStats}.
For all datasets, the timestamps are scaled to be in the range $[0,1000]$, following the same set up as \citet{soliman2022multivariate}. The datasets are divided into training and test data as noted in Table \ref{tab:dataStats}, with the test events occurring at the end of the dataset. We briefly describe the datasets below.

\begin{table}[t]
    \centering
    \caption{Summary statistics of real network datasets. 
    Test events are held out and used only for evaluating predictive accuracy.}
    \label{tab:dataStats}
    \begin{tabular}{cccc}
    \hline
    Dataset  & Nodes    & Total Events & Test Events \\
    \hline
    Reality  & $70$     & $2,161$      & $661$ \\
    Enron    & $142$    & $4,000$      & $1,000$ \\
    MID      & $147$    & $5,117$      & $1,078$\\
    email-Eu & $888$    & $264,360$    & $51,667$ \\
    Facebook & $43,953$ & $852,833$    & $170,567$ \\
    \hline
    \end{tabular}
\end{table}

\begin{itemize}
\item \textbf{MIT Reality Mining} \citep{eagle2009inferring}: We analyze a dataset consisting of 2,161 phone calls among core 70 callers and recipients. We use the start time of each call as the event timestamp.

\item \textbf{Enron Emails} \citep{klimt2004enron}: We consider a subset of the Enron email corpus as in \citet{Dubois2013} and \citet{soliman2022multivariate}, which includes 4,000 emails exchanged among 142 individuals.

\item \textbf{Militarized Interstate Disputes (MIDs)} from the Correlates of War project \citep{palmer2021mid5}:
    We consider a total of 5,117 events between 147 (sovereign) states where each event is an act of hostility from one state to another state. We remove 8 nodes from the dataset which are disconnected from the largest connected component.

\item \textbf{Email-Eu-core temporal network} \citep{paranjape2017motifs}: We consider a subset of the Email-Eu-core temporal network dataset that includes 264,360 email communications between 888 members of an European research institution in 452 days. All these nodes are part of the largest connected component.

 \item \textbf{Facebook Wall Posts} \citep{viswanath2009evolution}: 
    We consider a total of 852,833 Facebook wall posts from September 2004 to January 2009 among 43,953 users, 
    We only consider posts from a user to another user so that there are no self-edges.
    We remove the nodes that are not connected to the largest connected component.
\end{itemize}

We compare our models with several other temporal point process models: MULCH \citep{soliman2022multivariate}, CHIP \citep{arastuie2020chip}, BHM \citep{junuthula2019block}, and DLS \citep{yang2017decoupling}.
The MULCH, CHIP, and BHM models are described in Section \ref{sec:dch_examples} and are also part of the DCH models. 
For these DCH models, we assign nodes that are in the test set but not the training set to the largest block, consistent with \citet{arastuie2020chip} and \citet{soliman2022multivariate}. 
The Dual Latent Space (DLS) model uses continuous latent spaces to model a reciprocal excitation network. 
For the DLS model, we randomly sample the latent positions form multivariate Gaussian for those new nodes.

\subsection{Predictive Ability}
\label{sec:realDataPredictive}
\paragraph{Test Log-likelihood:}
To evaluate the performance of the different models, we compute the test data log-likelihood per event, which has been used in many prior studies \citep{soliman2022multivariate,arastuie2020chip,Dubois2013}.
We use the data in the training set to fit the models by estimating the community assignments of the nodes and the Hawkes process parameters, and then we evaluate predictive accuracy using the log-likelihood per event of the data in the testing set. 

From Table \ref{tab:testLogLik}, we find that our restricted SR model can achieve high test log-likelihood on all datasets, either the highest or second highest among all models. 
Its test log-likelihood on most datasets is only slightly worse than the more complex and much slower MULCH model.

\begin{table}[tp]
	\centering
    \caption{Mean test log-likelihood per event for 5 real network datasets across all models. Larger values indicate higher predictive ability. Bold entry denotes highest log-likelihood for each dataset, and underline denotes the second highest one. The DLS model cannot scale to the email-Eu and Facebook datasets.}
    \label{tab:testLogLik}
    \begin{tabular}{cccccc}
    \hline
    Model & Reality      & Enron      & MID      & email-Eu   & Facebook   \\
    \hline
    Restricted SR & $\underline{-4.49}$ & $\underline{-5.41}$ & $\b{-3.52}$  &$\underline{-3.64}$& $\underline{-7.30}$ \\
    MULCH & $\b{-3.82}$ & $\b{-5.13}$ & $\underline{-3.53}$  &$-3.76$ & $\b{-6.82}$\\
    CHIP  & $-4.83$      & $-5.61$      & $-3.67$    & $-4.26$  & $-9.46$  \\
    BHM   & $-5.37$      & $-7.49$      & $-5.33$     & $\b{-3.54}$ & $-14.4$  \\
    DLS   & $-5.65$      & $-7.57$      & $-4.52$      & &\\
    \hline
    \end{tabular}
\end{table}

\paragraph{Dynamic Link Prediction:}
Next we compare the models in terms of their dynamic link prediction ability by randomly sampling $100$ time intervals $[t, t +\delta]$ in the test set and using the models to compute the probability that a event will occur between each node pair in the time intervals. The probability that a event occur between node pair $(i,j)$ in $[t, t +\delta]$ is given by $1-\exp\{-\int_{t}^{t+\delta}\lambda_{ij}(s)ds\}$ \citep{yang2017decoupling}. 
For each time interval, we calculate the area under the receiver operating characteristic curve (AUC) for each model on each dataset. 
This experiment set-up has been used in several prior studies \citep{soliman2022multivariate,yang2017decoupling}. 

We choose the same $\delta$ as in \citet{soliman2022multivariate} for Reality, Enron, MID and Facebook datasets. For the email-EU dataset, we choose $\delta$ to be 1 month. 
For the Facebook dataset, we randomly sample $1,000$ sender nodes and $1,000$ receiver nodes, and only make prediction for the node pairs among them, as the network is too large to consider all sender and receiver pairs.
For each model, the $K$ is chosen to be the one that maximize the test log-likelihood.

From Table \ref{tab:dlpAUC}, we can see our restricted SR model achieves the highest AUC on the Facebook dataset and second highest on email-Eu. 
The more complex MULCH and DLS models perform better than the simpler models in this experiment, but they cannot scale to even the downsampled Facebook data (and email-Eu in the case of DLS).

\begin{table}[tp]
    \setlength{\tabcolsep}{3pt}
	\centering
    \caption{Dynamic link prediction AUC for 5 real network datasets across all models. Mean (standard deviation) of AUC over 100 random short time intervals. 
    Bold entry denotes highest mean link prediction AUC for a dataset, and underline denotes the second highest one. 
    MULCH does not scale to the Facebook dataset, and DLS does not scale to the email-Eu or Facebook datasets.
    be scaled to this dataset.}
    \label{tab:dlpAUC}
    \begin{tabular}{cccccc}
    \hline
    Model & Reality         & Enron               & MID& email-Eu&Facebook \\
    \hline
    Restricted SR & $0.921(.041)$ & $0.810(.004)$      & $0.968(.026)$& $\underline{0.958(.006)}$ & $\b{0.763(.097)}$\\
    MULCH & $\bm{0.954 (.036)}$ & $\underline{0.852 (.006)}$      & $0.968 (.023)$& $\b{0.959(.008)}$&N/A\\
    CHIP  & $0.931 (.033)$      & $0.792 (.005)$      & $0.966 (.030)$ &$0.926(.009)$&$0.756(.093)$\\
    BHM   & $\underline{0.951 (.035)}$      & $0.846 (.005)$      & $\underline{0.973 (.022)}$ & $0.889(.013)$ &$0.661(.089)$\\
    DLS   & $0.935 (.034)$      & $\bm{0.872 (.001)}$ & $\bm{0.981 (.013)}$ &N/A&N/A\\
    \hline
    \end{tabular}
\end{table}

\subsection{Computational Efficiency}
\begin{table}[t]
	\centering
    \caption{Wall clock time to fit each model on the 3 largest real network datasets. For each model, the $K$ is chosen to be the one that maximize the test log-likelihood. The DLS model does not scale to email-Eu or Facebook.}
    \label{tab:time}
    \begin{tabular}{cccc}
    \hline
    Model &  MID & email-Eu & Facebook \\
    \hline
    Restricted SR & 8.4 seconds & 12 minutes& 50 minutes\\
    MULCH & 31 seconds & 28 minutes & 16 hours\\
    CHIP  & 0.48 seconds & 26 seconds & 3.0 minutes\\
    BHM   & 3.5 seconds & 50 seconds & 3.5 minutes\\
    DLS   & 90 minutes & & \\
    \hline
    \end{tabular}
\end{table}

We compare the computational efficiency of our restricted SR model against the MULCH, CHIP, BHM and DLS models by measuring the wall clock time to fit a dataset. We fit each model separately on the 3 largset datasets: MID, eu-Email, and Facebook datasets.
The wall clock times are shown in Table \ref{tab:time}. Our restricted SR model is much faster than MULCH and DLS, especially on a large dataset like Facebook. This shows the potential that our restricted SR model can be scaled to larger datasets. 
It is nearly impossible to apply refinement on MULCH when the dataset is large because it is too slow. The univariate Hawkes process models, CHIP and BHM, are faster than our model on all datasets. However, our models have better predictive ability, which is indicated by the higher test log-likelihood and better dynamic link prediction results in Section \ref{sec:realDataPredictive} for our restricted SR model compared to CHIP and BHM.

In the dynamic link prediction experiments, we also notice that the more complex dependencies in the MULCH model can significantly increase computation time. 
For example, on the email-Eu dataset, the dynamic link prediction experiment required 24 minutes for MULCH.  In comparison, our restricted SR model required only 1.9 minutes, which is comparable to the simpler CHIP model that required 2.7 minutes. 
The BHM is the fastest at the dynamic link prediction task, requiring only 1.9 seconds. 
This is because all nodes in the same block are equally excited, so the node pair which are in the same block pair will have the same intensity function, and we just need to compute it once for each block pair. 
The restricted SR, CHIP, and models MULCH allow each node pair to have a different intensity function, and thus, dynamic link prediction is slower for these models.

\subsection{Ablation Studies}
\label{sec:ablation}
Recall that the SR model we introduced in Section \ref{sec:sr_model} had 6 parameters for block pairs $(a,b)$ and $(b,a)$ with $a\neq b$: $M_{ab},M_{ba}, \alpha^n_{ab}, \alpha^n_{ba},\alpha^r_{ab},\alpha^r_{ba}$.
We then introduced the restricted SR model in Section \ref{sec:restricted_sr} by assuming that $\alpha^r_{ab} = \alpha^r_{ba}$ (equal reciprocal excitation) in order to reduce the number of parameters to 5, which led to the GMM results in Section \ref{sec:gmm_estimation}. 
One could instead assume equal self excitation so that $\alpha^n_{ab} = \alpha^n_{ba}$, providing an alternative restriction to 5 parameters. 

We perform experimental comparisons on four different variants of our proposed restricted SR model. RES-SR-r denotes the restricted SR model with $\alpha^r_{ab} = \alpha^r_{ba}$ so that the reciprocal excitation parameter is the same within a block pair. 
RES-SR-n denotes the restricted SR model with $\alpha^n_{ab} = \alpha^n_{ba}$ so that the self excitation parameter is the same within a block pair. 
In both cases, we consider versions both with and without our local refinement procedure from Section \ref{sec:refinement}.

The predictive accuracy of the different variants is shown in Tables \ref{tab:testLogLikAblation} and \ref{tab:dlpAUC_Ablation} for test log-likelihood and dynamic link prediction AUC, respectively. 
The model labeled RES-SR-r + refinement is the variant we labeled as the restricted SR model in Tables \ref{tab:testLogLik} to \ref{tab:time}. 
For the MID dataset, all variants choose $K=1$, so we have only a single diagonal block pair $a=b=1$. 
Thus, the RES-SR-r and RES-SR-n models are the same, and there is no refinement necessary, so the log-likelihoods and AUCs are the same for all variants.
We find that the refinement procedure can improve the predictive ability in most cases, especially on the large datasets like email-Eu, and Facebook datasets. 
However, improvement in predictive ability is not guaranteed, as the refinement only increases the train data log-likelihood and not necessarily the test data, which could lead to overfitting. 

\begin{table}[tp]
	\centering
    \caption{Mean test log-likelihood per event for 5 real network datasets across all variants of the restricted SR model. Larger values indicate higher predictive ability. Bold entry denotes highest log-likelihood for each dataset, and underline denotes the second highest one.}
    \label{tab:testLogLikAblation}
    \begin{tabular}{cccccc}
    \hline
    Model & Reality      & Enron      & MID      & email-Eu   & Facebook   \\
    \hline
    RES-SR-r & $\underline{-4.48}$ & $\b{-5.39}$ & $\b{-3.52}$  & $-3.77$ & $-7.37$ \\
    RES-SR-r + refinement & $-4.49$ & $\underline{-5.41}$ & $\b{-3.52}$  &$\b{-3.64}$& $\b{-7.30}$ \\
    RES-SR-n  & $-4.61$ & $-5.48$  & $\b{-3.52}$&$-3.80 $ & $-7.41$ \\
    RES-SR-n + refinement &$\b{-4.33}$   & $-5.48$  & $\b{-3.52}$ & $\underline{-3.74}$&$ \underline{-7.33}$ \\
    \hline
    \end{tabular}
\end{table}

\begin{table}[tp]
    \setlength{\tabcolsep}{3pt}
	\centering
    \caption{Dynamic link prediction AUC for 5 real network datasets across all variants of the restricted SR model. Mean (standard deviation) of AUC over 100 random short time intervals. Bold entry denotes highest mean link prediction AUC for a dataset, and underline denotes the second highest one.}
    \label{tab:dlpAUC_Ablation}
    \begin{tabular}{cccccc}
    \hline
    Model & Reality         & Enron               & MID& email-Eu&Facebook \\
    \hline
    RES-SR-r & $0.913(.051)$ & $\underline{0.801(.007)}$      & $\b{0.968(.026)}$ &$0.943(.007)$& $0.754 (.093)$\\
    RES-SR-r + ref. & $0.921(.041)$ & $\b{0.810(.004)}$      & $\b{0.968(.026)}$& $\b{0.958(.006)}$ & $\underline{0.763(.097)}$\\
    RES-SR-n & $\underline{0.943(.040)}$ & $0.794	(.007)$      & $\b{0.968(.026)}$ &$0.943(.006)$&$0.759 (.087)$\\
    RES-SR-n + ref. & $\b{0.947(.035)}$ & $0.794	(.007)$      & $\b{0.968(.026)}$&$\underline{0.955(.007)}$ &$\b{0.765(.105)}$\\
    \hline
    \end{tabular}
\end{table}

The wall clock time to fit each different variant to each of the 3 largest datasets is shown in Table \ref{tab:timeAblation}. 
We find that the refinement procedure typically takes 2 to 3x the time of the estimation procedure without refinement, similar to what we observed with the simulated networks. 
Even with the refinement procedure, the restricted SR model is still highly scalable, as it fits the large Facebook data with over 40,000 nodes in under 1 hour.

\begin{table}[tp]
	\centering
    \caption{Wall clock time to fit each different variant of the restricted SR model on the 3 largest real network datasets. For each model, the $K$ is chosen to be the one that maximize the test log-likelihood.}
    \label{tab:timeAblation}
    \begin{tabular}{cccc}
    \hline
    Model &  \multicolumn{1}{c}{MID}& \multicolumn{1}{c}{email-Eu}&\multicolumn{1}{c}{Facebook} \\
    \hline
    RES-SR-r & 4.4 seconds & 2.8 minutes  &15 minutes\\
    RES-SR-r + refinement & 8.4 seconds & 12 minutes& 50 minutes\\
    RES-SR-n & 4.8 seconds & 2.5 minutes & 14 minutes\\
    RES-SR-n + refinement & 8.5 seconds & 9.8 minutes  & 43 minutes\\
    \hline
    \end{tabular}
\end{table}

\section{Conclusion}
In this paper we have theoretically analyzed a spectral clustering algorithm applied to the directed weighted count matrix for community detection in continuous time temporal networks constructed from relational events data. We introduced the Dependent Community Hawkes (DCH) models, a general class of block models allowing for dependencies across node pairs within two block pairs through a mutually exciting Hawkes process. 
The DCH models generalized the recently proposed MULCH model by \cite{soliman2022multivariate} as well as several other models.

Our upper bound brings out the relationship between the accuracy of spectral clustering and several model quantities including the time interval $T$, the number of nodes $n$, the number of communities $K$, the Hawkes process parameters, and a quantity $\gamma_{\max}$ that quantifies the amount of dependence induced by the mutually exciting Hawkes processes. Extensive simulation results verified our theoretical insights. 

We then proposed a new model from the DCH class of models, which we call the Self and Reciprocal excitation (SR) model. 
It is more flexible than other simpler DCH models from the literature \citep{junuthula2019block,arastuie2020chip} but much simpler than MULCH, which enabled us to develop a computationally efficient and statistically consistent GMM estimator for the parameters. We demonstrate that the proposed SR model with the proposed estimators is computationally almost as attractive as the CHIP model of \cite{arastuie2020chip}, while providing empirical data fits competitive with MULCH. 

While there are several results available on the accuracy of spectral clustering for community detection in network data, not much is known about how dependencies across the edges affect spectral clustering or how the method performs for weighted graphs. Our results in this paper provide insights into both of those questions using a plausible model for network data generation. This is our contribution to the literature on spectral clustering for static networks. On the other hand, our results provide estimation methods with theoretical guarantees and computational efficiency for a broad class of models for temporal networks or relational events data.

\acks{This material is based upon work supported by the National Science Foundation grants DMS-1830547, DMS-1830412, IIS-1755824, and IIS-2318751.}



\newpage

\appendix

\section{Proof of Main Results}
\subsection{Proof of Theorem \ref{thm:spectral_norm_bound}}
\label{sec11}

\begin{proof}
Our proof technique for bounding the spectral norm of the deviation of the count matrix from its expectation involves multiple steps. First, we obtain upper bounds on the quantities necessary to bound the spectral norm of the deviation of a Gaussian random matrix with dependent entries and having the same mean and covariance as the count matrix elements. Then, we combine the result on the rate of convergence of the count matrix to the Gaussian vector with this result to obtain the statement of the theorem. 

Accordingly, we obtain an upper bound on the max row sum of $\boldsymbol R=(\b I - \b \Gamma)^{-1}$ as a function of the DCH parameters using several arguments from linear algebra and properties of special matrices. Let $\boldsymbol G = \b I -\boldsymbol\Gamma$. Since $\b \Gamma$ is block-diagonal, we know $\b G$ is also a block diagonal matrix, with the blocks $\b G_{(a,b),(b,a)}= \b I - \b \Gamma_{(a,b),(b,a)}$, where $\b I$ is the identity matrix of appropriate dimension.
By our assumption (1) in the statement of the theorem,  $\rho(\boldsymbol \Gamma_{(a,b),(b,a)})\leq\sigma^*<1$ for any $1\leq a\leq b\leq K$ by the properties of the block diagonal matrix. Thus each of $\boldsymbol G_{(a,b),(b,a)}$ is invertible. Now, by the properties of block diagonal matrix, we know $\b R = \b G^{-1}$ is also a block diagonal matrix with the blocks given by $\b R_{(a,b),(b,a)} = \b G^{-1}_{(a,b),(b,a)}$.
We notice that for each block pair $(a,b),a\neq b$, we can also write $\boldsymbol G_{(a,b)}$ as a block matrix, i.e.,
\begin{equation*}
\label{eq:bm_Gamma_G}
\boldsymbol{\Gamma}_{(a, b),(b,a)}=\left(\begin{array}{ll}
\boldsymbol{\Gamma}_{a b \rightarrow a b} & \boldsymbol{\Gamma}_{ba \rightarrow ab} \\
\boldsymbol{\Gamma}_{ab \rightarrow ba} & \boldsymbol{\Gamma}_{b a \rightarrow b a}
\end{array}\right) \quad \Rightarrow \quad 
\boldsymbol{G}_{(a, b),(b,a)}=\left(\begin{array}{ll}
\boldsymbol{G}_{a b \rightarrow a b} & \boldsymbol{G}_{ba \rightarrow ab} \\
\boldsymbol{G}_{ab \rightarrow ba} & \boldsymbol{G}_{b a \rightarrow b a}
\end{array}\right).
\end{equation*}
From our assumption (2) in the statement of the theorem, for each sub-block matrix in $\b \Gamma_{(a, b),(b,a)}$, the row sums are identical, and thus we can use $\gamma_{\cdot\cdot\rightarrow\cdot\cdot}$ to denote these row sums (e.g., $\mGamma_{ab\rightarrow ab}\boldsymbol 1 = \gamma_{ab\rightarrow ab}\boldsymbol 1$, where $\b 1$ is a column vector containing all 1s). 
Also, since $\mGamma_{(a, b),(b,a)}$ is a non-negative matrix, we know the minimum row sum of $\b \Gamma_{(a, b),(b,a)}$ is greater than $\min\{\mgamma_{ab\rightarrow ab}, \mgamma_{ba\rightarrow ba} \}$. 

Then by Proposition \ref{lem:bound_eig}, we know that $\rho(\bmGamma)\geq \min\{\mgamma_{ab\rightarrow ab}, \mgamma_{ba\rightarrow ba} \}$.
Without loss of generality, we assume $\mgamma_{ab\rightarrow ab}$ is the minimum of the two and therefore, 
\[\mgamma_{ab\rightarrow ab} \leq \rho(\bmGamma)\leq \sigma^*<1.
\]
But since $\mgamma_{ab\rightarrow ab}$ is also the maximum row sum of the sub-block matrix  $\mGamma_{ab\rightarrow ab}$ (in fact all the rows have identical sums), by Proposition \ref{lem:bound_eig},
\[\rho(\mGamma_{ab\rightarrow ab})\leq \mgamma_{ab\rightarrow ab}\leq \sigma^*<1.
\]
So $(\sigma^*+\epsilon) \boldsymbol I -\mGamma_{ab\rightarrow ab}$ is invertible for any $\epsilon>0$ . Consider the following block matrix: 
\begin{equation*}
(\sigma^*+\epsilon) \boldsymbol I -\bmGamma = 
\left(\begin{array}{cc}
(\sigma^*+\epsilon) \boldsymbol I -\boldsymbol{\Gamma}_{a b \rightarrow a b} &- \boldsymbol{\Gamma}_{ba \rightarrow ab} \\
-\boldsymbol{\Gamma}_{ab \rightarrow ba} & (\sigma^*+\epsilon) \boldsymbol I-\boldsymbol{\Gamma}_{b a \rightarrow b a}
\end{array}\right).
\end{equation*}
Clearly this matrix is invertible, and we can define the Schur complement as below:
\begin{equation}
\label{eq:schur_complement}
\begin{split}
\left[(\sigma^*+\epsilon) \boldsymbol I -\bmGamma\right]&/\left[(\sigma^*+\epsilon) \boldsymbol I -\mGamma_{ab\rightarrow ab}\right] \\
 :=&(\sigma^*+\epsilon) \boldsymbol I -\mGamma_{ba\rightarrow ba} - \mGamma_{ab\rightarrow ba} \left[(\sigma^*+\epsilon) \boldsymbol I -\mGamma_{ab\rightarrow ab}\right]^{-1}\mGamma_{ba\rightarrow ab}.
\end{split}
\end{equation}
 Notice that all sub-matrices involved in \eqref{eq:schur_complement} above satisfy the conditions in Proposition \ref{lem:identical_row_sum} since each of them has identical row sums. Therefore, using the result in Proposition \ref{lem:identical_row_sum}, we have 
\begin{align*}
    \big(\left[(\sigma^*+\epsilon) \boldsymbol I -\bmGamma\right] & /\left[(\sigma^*+\epsilon) \boldsymbol I -\mGamma_{ab\rightarrow ab}\right]\big)\boldsymbol 1 \\
    & = \left(\sigma^*+\epsilon-\mgamma_{ba\rightarrow ba} - \frac{\mgamma_{ab\rightarrow ba} \mgamma_{ba\rightarrow ab}}{\sigma^*+\epsilon-\mgamma_{ab \rightarrow ab}}\right)\boldsymbol 1. 
\end{align*}
Next note that for any $\epsilon>0$, $(\sigma^*+\epsilon) \boldsymbol I -\bmGamma$ is a M-matrix (\cite{pena1995m}). This implies its inverse and the Schur complement above are non-negative matrices. Then we must have for any $\epsilon>0$,
\begin{equation}
\label{eq:bound_gamma}
\sigma^*+\epsilon-\mgamma_{ba\rightarrow ba} - \frac{\mgamma_{ab\rightarrow ba} \mgamma_{ba\rightarrow ab}}{\sigma^*+\epsilon-\mgamma_{ab \rightarrow ab}}\geq 0.
\end{equation}
Further, since $\mGamma$ is a non-negative matrix, so $\gamma_{ab\rightarrow ba},\gamma_{ba\rightarrow ab}$ are non-negative. Also, since $\sigma^*\geq \gamma_{ab \rightarrow ab}$, we know
\[\frac{\mgamma_{ab\rightarrow ba} \mgamma_{ba\rightarrow ab}}{\sigma^*+\epsilon-\mgamma_{ab \rightarrow ab}}\geq 0.
\]
Thus we can derive $\mgamma_{ba\rightarrow ba}\leq \sigma^*$ from the inequality (\ref{eq:bound_gamma}). Then from Proposition \ref{lem:bound_eig}, we know $\rho(\mGamma_{ba \rightarrow ba})\leq \mgamma_{ba\rightarrow ba} = \sigma^*<1$. Thus $\boldsymbol{G}_{a b \rightarrow a b} = \boldsymbol I - \boldsymbol{\Gamma}_{a b \rightarrow a b}$ and $\boldsymbol{G}_{ba \rightarrow ba} = \boldsymbol I - \boldsymbol{\Gamma}_{ba \rightarrow ba}$ are invertible. The matrix $\bmG$ can be inverted blockwise as follows:
\begin{equation}
\label{eq:G_inverse}
\begin{split}
& \bmG^{-1}\\
&=\left(\begin{array}{cc}
\left({{\boldsymbol{G}_{a b \rightarrow a b}}}-{{\boldsymbol{G}_{ba \rightarrow a b}} {\boldsymbol{G}^{-1}_{ba \rightarrow ba}}} {{\boldsymbol{G}_{ab \rightarrow ba}}}\right)^{-1} & {0} \\
{0} & \left({{\boldsymbol{G}_{ba \rightarrow ba}}}-{{\boldsymbol{G}_{ab \rightarrow ba}} {\boldsymbol{G}^{-1}_{a b \rightarrow a b}}} {{\boldsymbol{G}_{ba \rightarrow a b}}}\right)^{-1}
\end{array}\right)\\
& \qquad \times \left(\begin{array}{cc}
\boldsymbol{I} & -{{\boldsymbol{G}_{ba \rightarrow a b}} {\boldsymbol{G}^{-1}_{ba \rightarrow ba}}} \\
-{{\boldsymbol{G}_{ab \rightarrow ba}} {\boldsymbol{G}^{-1}_{a b \rightarrow a b}}} & \boldsymbol{I}
\end{array}\right),
\end{split}
\end{equation}
which is a product of two block matrices \citep{lu2002inverses}. Since $\b G_{ab\rightarrow ab} = \b I- \b \Gamma_{ab\rightarrow ab}$, we know it also has identical row sum $1- \gamma_{ab\rightarrow ab}$, and similarly, $\b G_{ba\rightarrow ba}$ has identical row sum $1- \gamma_{ba\rightarrow ba}$. Also, since $\b G_{ba\rightarrow ab} = -\b \Gamma_{ba\rightarrow ab}$ and $\b G_{ab\rightarrow ba} = -\b \Gamma_{ab\rightarrow ba}$, we know they have identical row sum $-\gamma_{ba\rightarrow ab}$ and $-\gamma_{ab\rightarrow ba}$ respectively. Using Proposition \ref{lem:identical_row_sum} and inequality (\ref{eq:bound_gamma}), we know that 
\begin{equation*}
\begin{split}
     \left({{\boldsymbol{G}_{ba \rightarrow ba}}}-{{\boldsymbol{G}_{ab \rightarrow ba}} {\boldsymbol{G}^{-1}_{a b \rightarrow a b}}} {{\boldsymbol{G}_{ba \rightarrow a b}}}\right)^{-1} \boldsymbol 1 &= \left(1-\mgamma_{ba\rightarrow ba} - \frac{\mgamma_{ba\rightarrow ab} \mgamma_{ab\rightarrow ba}}{1-\mgamma_{ab \rightarrow ab}}\right)^{-1}\boldsymbol 1\\
    &\leq (1-\sigma^*)^{-1}\boldsymbol 1.
\end{split}
\end{equation*}
Similarly, we can get 
$$\left({{\boldsymbol{G}_{a b \rightarrow a b}}}-{{\boldsymbol{G}_{ba \rightarrow a b}} {\boldsymbol{G}^{-1}_{ba \rightarrow ba}}} {{\boldsymbol{G}_{ab \rightarrow ba}}}\right)^{-1} \boldsymbol 1\leq (1-\sigma^*)^{-1}\boldsymbol 1.$$
Using assumption (2) from the theorem and Proposition \ref{lem:identical_row_sum}, along with the fact that $\mgamma_{ba \rightarrow ba}\leq \sigma^*$,  we  have 
\begin{equation*}
-{{\boldsymbol{G}_{ba \rightarrow a b}} {\boldsymbol{G}^{-1}_{ba \rightarrow ba}}}\boldsymbol 1 = \mgamma_{ba \rightarrow a b}(1-\mgamma_{ba \rightarrow ba})^{-1}\boldsymbol 1\leq \gamma_{\max}(1-\sigma^*)^{-1}\boldsymbol 1.
\end{equation*}
Similarly we can have $-{{\boldsymbol{G}_{ba \rightarrow ab}} {\boldsymbol{G}^{-1}_{ab \rightarrow ab}}}\boldsymbol 1\leq \gamma_{\max}(1-\sigma^*)^{-1}\boldsymbol 1$. Plug in these upper bounds to (\ref{eq:G_inverse}), and we can compute the row sum bound for the $\boldsymbol R_{(a,b),(b,a)}$ now:
\begin{equation}
\label{eq:row_sum_R}
\begin{split}
    & \boldsymbol R_{(a,b),(b,a)} \boldsymbol 1 = \bmG^{-1} \boldsymbol 1\\ 
    &= \left(\begin{array}{cc}
\left({{\boldsymbol{G}_{a b \rightarrow a b}}}-{{\boldsymbol{G}_{ba \rightarrow a b}} {\boldsymbol{G}^{-1}_{ba \rightarrow ba}}} {{\boldsymbol{G}_{ab \rightarrow ba}}}\right)^{-1} & {0} \\
{0} & \left({{\boldsymbol{G}_{ba \rightarrow ba}}}-{{\boldsymbol{G}_{ab \rightarrow ba}} {\boldsymbol{G}^{-1}_{a b \rightarrow a b}}} {{\boldsymbol{G}_{ba \rightarrow a b}}}\right)^{-1}\end{array}\right)\\
& \qquad \times \left(\begin{array}{l}
 (1+\mgamma_{ba \rightarrow a b}(1-\mgamma_{ba \rightarrow ba})^{-1}) \boldsymbol 1 \\
 (1+\mgamma_{ab \rightarrow ba}(1-\mgamma_{ab \rightarrow ab})^{-1}) \boldsymbol 1
\end{array}\right)\\
&= \left(\begin{array}{l}
 \left(1-\mgamma_{ab\rightarrow ab} - \frac{\mgamma_{ba\rightarrow ab} \mgamma_{ab\rightarrow ba}}{1-\mgamma_{ba \rightarrow ba}}\right)^{-1}(1+\mgamma_{ba \rightarrow a b}(1-\mgamma_{ba \rightarrow ba})^{-1}) \boldsymbol 1 \\
 \left(1-\mgamma_{ba\rightarrow ba} - \frac{\mgamma_{ba\rightarrow ab} \mgamma_{ab\rightarrow ba}}{1-\mgamma_{ab \rightarrow ab}}\right)^{-1}(1+\mgamma_{ab \rightarrow ba}(1-\mgamma_{ab \rightarrow ab})^{-1}) \boldsymbol 1
\end{array}\right)\\
&\leq(1-\sigma^*)^{-1}\left(1 + \gamma_{\max}(1-\sigma^*)^{-1}\right) \boldsymbol 1.
\end{split}
\end{equation}
Since $\bmG$ is a M-matrix, we know $\bmR$ is a non-negative matrix, so 
$$\|\bmR\|_{\infty} \leq (1-\sigma^*)^{-1}\left(1 + \gamma_{\max}(1-\sigma^*)^{-1}\right)\leq (1-\sigma^*)^{-2}\left(1 + \gamma_{\max}\right).$$
For $a=b$, $\boldsymbol \Gamma_{(a,a)}$ has identical row sums.
Thus, $\mR_{(a,a)}$ has identical row sums smaller than $(1- \sigma^*)^{-1}$ by Proposition \ref{lem:identical_row_sum}. Therefore, if we let $C_1 = (1-\sigma^*)^{-2}\left(1 + \gamma_{\max}\right)$, then we can have $\|\boldsymbol R\|_{\infty}\leq C_1$. We can use the same argument to prove that the max column sum of $\boldsymbol R$, $\|\boldsymbol R^T\|_{\infty}\leq C_1$.

Next using the result in Proposition \ref{prop:normal}, we have
\[\sqrt{T}\left(\frac{\operatorname{vec}\left(\boldsymbol{N}_{T}\right)}{T}- \boldsymbol{R} \operatorname{vec}(\boldsymbol{\mu})\right) \stackrel{d}{\Rightarrow} \mathcal{N}\left( \b 0, \boldsymbol{R} \operatorname{diag}(\boldsymbol{R} \operatorname{vec}(\boldsymbol{\mu})) \boldsymbol{R}^{T}\right)\]
as $T\rightarrow \infty$, and the speed of convergence can be characterized by the upper bound on the $d_2$ distance given in Proposition \ref{prop:normal}. Suppose $\b M$ is a random matrix such that $\operatorname{vec}(\b M)$ follows a $\mathcal{N}(\b 0,\boldsymbol{R} \operatorname{diag}(\b R\operatorname{vec}(\b \mu)) \boldsymbol{R}^{T})$ distribution.
The relationship between the $d_2$ distance and the Kolmogorov distance in Proposition \ref{Kdist} can be used to conclude that, for any $x \in \mathbb{R}^{n^2}$,
\[
\Bigg |P\left(\sqrt{T}\left(\frac{\operatorname{vec}\left(\boldsymbol{N}_{T}\right)}{T}- \boldsymbol{R} \operatorname{vec}(\boldsymbol{\mu})\right)>x\right) -  P\left(\operatorname{vec}(\b M) >x \right) \Bigg | < \frac{\kappa}{T^{1/6}},\]
for a constant $\kappa(n)$ which does not depend on $T$ but may depend on $n$. In the above statement, the notation $\b x\geq \b y$ for two vectors $\b x,\b y$ means $x_i \geq y_i$ for all co-ordinates $i$. 

By the assumption of the theorem and the fact $\|\boldsymbol R\|_{\infty}\leq C_1$, we know 
\begin{equation}
\label{eq:bound_mean}
    \boldsymbol R \operatorname{vec} (\boldsymbol \mu) \leq C_1\mu_{\max}\boldsymbol 1.
\end{equation}
 Thus, using the sub-multiplicative property of $\|\cdot\|_{\infty}$ norm, we have 
\begin{equation*}
    \|\mR \operatorname{diag}(\boldsymbol R \operatorname{vec}(\boldsymbol \mu)) \boldsymbol R^T\|_{\infty}\leq \|\mR\|_{\infty} \|\operatorname{diag}(\boldsymbol R \operatorname{vec}(\boldsymbol\mu))\|_{\infty}  \|\boldsymbol R^T\|_{\infty}\leq C_1^3\mu_{\max}.
\end{equation*}

Provided the result above, we are ready to calculate the upper bound of  $\E\| \b M\|$ using Proposition \ref{lem:khintchine_inequality}. Note that $\b M$ has jointly Gaussian entries and $\E \b M = \b 0$. We first compute
\begin{equation*}
\begin{split}
    \left\|\mathbb{E}  \b M^{T} \b M\right\|_{\infty} &= \max_{i} \sum_{1\leq k\leq n} \left(\sum_{1\leq j\leq n} |\E {M}_{ki} {M}_{kj}|\right)\\
    &\leq \sum_{1\leq k\leq n}   \|\mR \operatorname{diag}(\boldsymbol R \operatorname{vec}(\boldsymbol\mu)) \boldsymbol R^T\|_{\infty}\\
    &\leq n C_1^3\mu_{\max},
\end{split}
\end{equation*}
where the first inequality is due to $\sum_{k,l}|\E {M}_{ij} {M}_{kl}| \leq \| \mR \operatorname{diag}(\boldsymbol R \operatorname{vec}(\b \mu)) \boldsymbol R^T\|_{\infty}$ for any $(i,j)$.
Since $\mathbb{E}  \b M^{T} \b M$ is a symmetric matrix, we know $\left\| \mathbb{E} \b M^{T} \b M \right\|\leq \left\| \mathbb{E} \b M^{T} \b M \right\|_{\infty}\leq  C_1^3 n\mu_{\max} $.
Similarly, we can also show $\left\|\mathbb{E} \b M \b M^{T} \right\|\leq  C_1^3 n\mu_{\max} $. Thus by Proposition \ref{lem:khintchine_inequality}, we have 
\begin{equation}
\begin{split}
\label{eq:expect_spectral_bound}
    \E\left\| \b M\right\| &\leq 2\sqrt{(1+2 \log n)} \max\left\{ \left\|\E \b M^{T} \b M\right\|^{1/2}, \left\|\E {\b M \b M^{T}}\right \|^{1/2}  \right\}\\
    &\leq 2\sqrt{n C_1^3\mu_{\max}(1+2 \log n)} 
\end{split}.
\end{equation}

We can also compute a tail bound from  (\ref{eq:expect_spectral_bound}). Note that $\|{\b M}\|=\sup_{\|\boldsymbol v\|=\|\boldsymbol w\|=1}|\boldsymbol v^T{\b M}\boldsymbol w|$, and we have
\begin{equation*}
\begin{split}
   \E |\boldsymbol v^T{\b M}\boldsymbol w|^2&= \E \left(\sum_{ij} v_{i} w_{j} {M}_{i j}\right)^{2}\\
   &=\E\left(\sum_{i j k l} v_{i} w_{j} v_{k} w_l {M}_{i j} {M}_{k l}\right)\\
   &=\sum_{ij} v_{i}w_j \sum_{k l} v_{k} w_{l} \E{M}_{i j} {M}_{kl}\\
   &\leq \sum_{ij} v_{i} w_{j}\left(\sum_{kl}\left(v_{k} w_{l}\right)^{2} \sum_{k l}\left(\E {M}_{i j} {M}_{k l}\right)^{2}\right)^{\frac{1}{2}}\\
   &\leq \sum_{ij} v_{i} w_{j}\left(1 \left(\sum_{k l}\E  {{M}_{i j} {M}_{k l}}\right)^{2}\right)^{\frac{1}{2}}\\
   &\leq \sum_{ij} v_{i} w_{j} \left\|\boldsymbol{R} \operatorname{diag}(\boldsymbol{R} \operatorname{vec}(\boldsymbol{\mu})) \boldsymbol{R}^{T}\right\|_{\infty} \\
   &\leq n C_1^3\mu_{\max},
\end{split}
\end{equation*}
Thus, from Gaussian concentration (Theorem 5.8 in \cite{boucheron2013concentration}), for any $n>0$ and $a>0$ we have,
\begin{equation*}
    \boldsymbol{P}\left(\left\| {\b M} \right\|\geq \E \left\| {\b M} \right\| +a\right) \leq e^{-a^{2} /\left(2 n C_1^3\mu_{\max}\right)}.
\end{equation*}
If we choose $a = \sqrt{2nC_1^3\mu_{\max}n \log n \log T}$ and use the result from (\ref{eq:expect_spectral_bound}), then we can have
$$
\boldsymbol{P}\left(\left\|{\b M}\right\|\geq  3\sqrt{n C_1^3\mu_{\max}(1+2 \log n) \log T}\right) \leq e^{-\log n \log T}.
$$
Then for any $T>1$, we have with probability at least $1- \exp( -\log n \log T) -\kappa T^{-1/6}$:
\[
\sqrt{\frac{T}{\log T}} \left\|\frac{\boldsymbol{N}_{T} -E[\b N_T]}{T} \right\| \leq \b 3(1-\sigma^*)^{-3}\sqrt{n (1 + \gamma_{\max})^3\mu_{\max}(1+2 \log n)}.
\]

\end{proof}

\subsection{Proof of Theorem \ref{thm:misclustering_rate}}
\label{t5proof}
\begin{proof}
Let $\mathcal{I}=\{1\leq i\leq n:\|\b X_{i\cdot}\| =0\}$ represent indices of all 0 rows in $\b X$. Then we can bound the number of nodes in $\mathcal{I}$ as:
\begin{equation*}
\begin{split}
    |\mathcal{I}| &\leq \sum_{i=1}^n \|\b X_{i\cdot} - \t{\b X}_{i\cdot}\b Q\|^2 /\|\t{\b X}_{i\cdot}\|^2\\
    &\leq  \sum_{i=1}^n \|\b X_{i\cdot} - \t{\b X}_{i\cdot}\b Q\|^2/ \left(2\min_{1\leq j\leq K} n_j^{-1}\right)\\
    &\leq 0.5 n_{\max}\|\b X - \t{\b X}\b Q\|_F^2,
\end{split}
\end{equation*}
 where the first inequality is because, for any node $i$ in $\mathcal{I}$, we have $\|\b X_{i\cdot} - \t{\b X}_{i\cdot}\b Q\|^2 =\|\t{\b X}_{i\cdot}\|^2$, and  the second inequality is from Lemma \ref{lem:row_length_X}. For any $1\leq i\leq n$ and $i\notin \mathcal{I}$, let $\b X^*_{i\cdot} = \b X_{i\cdot}/\|\b X_{i\cdot}\|$ denote the row normalization of $\b X_i$. Then we have 
\begin{equation*}
    \begin{split}
        \sum_{i=1,i\notin \mathcal{I}}^n \|\b X^*_{i\cdot} - \t{\b X}^*_{i\cdot}\b Q\|^2  &= \sum_{i=1,i\notin \mathcal{I}}^n \left\| \frac{\b X_{i\cdot}}{\|{\b X}_{i\cdot}\|} - \frac{\t{\b X}_{i\cdot}\b Q}{\|\t{\b X}_{i\cdot}\|}\right\|^2\\
        &\leq 4 \sum_{i=1}^n \|\b X_{i\cdot} - \t{\b X}_{i\cdot}\b Q\|^2/\|\t{\b X}_{i\cdot}\|^2\\
        &\leq 2 n_{\max}\|\b X - \t{\b X}\b Q\|_F^2 ,\\
    \end{split}
\end{equation*}
where the first inequality comes from Lemma D.2 in \cite{rohe2016co}.
Now, we can slightly modify the proof of Theorem 3.1 in \cite{rohe2016co} and get the misclustering error rate:
\begin{equation*}
\begin{split}
r &\leq \frac{1}{n} \left(|\mathcal{I}| + 2(2+\varepsilon)^2\sum_{i=1, i\notin \mathcal{I}}^{n} \|\b X^*_{i\cdot} - \t{\b X}^*_{i\cdot}\b Q\|^2\right)\\
&\leq \frac{5(2+\varepsilon)^2 n_{\max}\|\b X - \t{\b X}\b Q\|_F^2}{n}.
\end{split}
\end{equation*}
We use the result from Proposition \ref{lem:eigenvectors_perturbation} to get $\|\b X - \t{\b X}\b Q\|_F^2$, and then we have 
\begin{equation}
\label{eq:misclustering_error}
\begin{split}
r&\leq \frac{80(2+\varepsilon)^2 n_{\max} K\|\frac{1}{T}({\b N_T}-\E \boldsymbol N_T)\|^{2}}{n  \lambda_{K}^{2}}
\end{split}.
\end{equation}
Under the same assumptions in Theorem \ref{thm:spectral_norm_bound}, we have, with probability $1-\exp (-\log n \log T)-\frac{\kappa(n)}{T^{1/6}}$,
$$\sqrt{\frac{T}{\log T}}\left\|\frac{1}{T}\left(\boldsymbol N_T - \E \boldsymbol N_T\right)\right\| \leq 3(1-\sigma^*)^{-3}\sqrt{n (1 + \gamma_{\max})^3\mu_{\max}(1+2 \log n)}. $$

Now, we apply this result to (\ref{eq:misclustering_error}). 
Then
\begin{equation*}
    \begin{split}
        r&\leq \frac{80(2+\varepsilon)^2 n_{\max} K\left\|{\b N_T}-\E \boldsymbol {N_T}\right\|^{2}}{n  \lambda_{K}^{2}}\\
        &\leq \frac{80(2+\varepsilon)^2 n_{\max} K}{n \lambda_{K}^{2}}2\left( 
        9(1-\sigma^*)^{-6} {\frac{n \log T}{T} (1 + \gamma_{\max})^3\mu_{\max}(1+2 \log n)}\right).
    \end{split}
\end{equation*}
with probability at least $1-\exp(\log n \log T)- \frac{\kappa(n)}{T^{1/6}}$ for any $n>1$ and  $T >1$.
\end{proof}

\subsection{Proof of Corollary \ref{cor:simplified_model_misclustering_error}}
\begin{proof}
    Since $\b \Gamma$ is a block diagonal matrix, we know $\sigma^* = \rho(\b\Gamma) = \max_{1\leq a\leq b\leq K}\rho(\b \Gamma_{(a,b)})$. Then using Proposition \ref{lem:bound_eig} (in Appendix \ref{sec:props_lemmas}), we can further show $ \sigma^*= \max\{ \gamma_1, \gamma_2\}$. By the definition of the $\gamma_{\max}$ in Theorem \ref{thm:spectral_norm_bound}, we note $\max\{ \gamma_1, \gamma_2/2\} \leq \gamma_{\max} \leq \max\{ \gamma_1, \gamma_2\}$, so $\sigma^*/2 \leq \gamma_{\max} \leq \sigma^*<1$. In particular this implies $\sigma* \leq 2 \gamma_{\max}$.
    By Proposition \ref{lem:identical_row_sum} (in Appendix \ref{sec:props_lemmas}) and the definition of $\b R$, we know that $\b R_{(a,a)} = (\b I - \b \Gamma_{(a,a)})^{-1}$ has row sums equal to $(1- \gamma_{1})^{-1}$, and $\b R_{(a,b)} = (\b I - \b \Gamma_{(a,b)})^{-1}$ has row sums equal to $(1- \gamma_{2})^{-1}$. Then from Proposition \ref{prop:normal}, we can derive the following form for the expected count matrix. For $i \neq j$, we have $\E ( {\b N_T})_{i j}= v_1T,  \text { if } z_i = z_j$ and  $\E ({\b N_T})_{i j}=v_2T, \text { if } z_i \neq z_j $, while $\E ({\b N_T})_{i j}= 0$ if $i=j$.
Here $v_1 = (1 -   \gamma_1)^{-1}\mu_1$ and $ v_2 =  (1 -   \gamma_2)^{-1}\mu_2$.

By the definition of $\E \boldsymbol {N_T}$, we can write
$$\frac{\E \boldsymbol {N_T}}{T}=\b Z\left(\left(v_{1}-v_{2}\right) \b I_{K}+v_{2} \b 1_{K} \b 1_{K}^{T}\right) \b Z^{T}$$
for some $\b Z\in \mathbb{R}^{n\times K}$.
The $K$th largest singular values of $\frac{\E \boldsymbol {N_T}}{T}$ is $\frac{n}{K}(v_1-v_2)$, so $ \lambda_K^2 = \frac{n^2}{K^2}(v_1 - v_2)^2$. Further $n_{\max} = \frac{n}{k}$. Then, from Theorem \ref{thm:misclustering_rate} we have the following:
\begin{align*}
   r\leq \frac{ c K^2\mu_{\max}^2\left(1+\gamma_{\max }\right)^3}{(v_1-v_2)^2 \left(1-\sigma^*\right)^{6}}\left( 
        \frac{\log T(1+2 \log n)}{nT\mu_{\max}} 
        \right),
\end{align*}
with probability at least $1-\exp(\log n \log T)- \frac{\kappa(n)}{T^{1/6}}$.
\end{proof}

\subsection{Proof of Lemma \ref{lem:id} and Theorem \ref{thm:gmm}}
\label{sec:app_gmm_proofs}
Let $\b \theta_0 = \{\b M_0, \b G_0\} \in \b \Theta$ be the true parameters. Further let $\h{\b g}_n (\b \Theta)$ be the sample version and $\b g_0 (\b \Theta)$ the population version of the GMM function. We use Theorem 2.6 in \cite{newey1994large} with $\hat{W}=W=I$.
\begin{proposition}
\label{prop:gmm}
\cite{newey1994large} Consider functions $\h{\b g}_n(\b \theta), \b g_0(\b \theta)$ defined on $\b\Theta\subset \mathbb{R}^{k}$ that satisfies
\begin{enumerate}
    \item The feasible parameters space $\b \Theta$ is compact.
    \item $\b g_0(\b \theta)$ is continuous on $\b\Theta$.
    \item $\b g_0(\b \theta)=\b 0$ if and only if $\b \theta = \b \theta_{0}$.
    \item $ \h{\b g}_n$ coverages to $\b g_0$ uniformly in probability.
\end{enumerate}
Let $\hat{\theta}_n$ be the minimizer of $\h{\b g}_n(\b\theta)^T \h{\b g}_n(\b\theta)$. Then
\[
\hat{\b \theta}_n \overset{p}{\to } \b \theta_0.
\]
\end{proposition}

We apply this proposition in the restricted SR model to show the consistency of our GMM estimator.
We first show a detailed proof of Lemma \ref{lem:id}, which corresponds to condition 3 in Proposition \ref{prop:gmm}. 

\begin{proof}[Lemma \ref{lem:id}]
We need to show the existence and uniqueness of the solution of $\b g_0(\b \theta_0)=\b 0$. The existence comes from the assumption that our model has true parameter $\b \theta_0$ in the feasible space, so we only need to prove the uniqueness. Suppose there is another $\t {\b \theta}=({\t{\b {M}}}_{(a,b), (b,a)}, \t{\b G}_{(a,b), (b,a)}) \in\b \Theta_{(a,b), (b,a)}$ such that $ \t{\b \theta} \neq \b \theta_0$ and $\b g_0(\t{\b \theta}) = \b 0$.
Then we can similarly define  $
 \t{ \b R} = (\b I - \t{\b G})^{-1},\quad
\t{\b \Lambda} =\t {\b R} \operatorname{vec}(\t{\b \mu}), \quad
\t{\b C} = \b   \t{ \b R}\operatorname{diag}(\t {\b \Lambda})  \t{ \b R}^T$.
Since the components of $\b g_0$ are linear functions of $\Tilde {\b \Lambda}, \Tilde {\b C}$; therefore, $\b g_0(\t {\b\theta}) =\b 0$  implies we must have $\b \Lambda_0 = \t{\b \Lambda}$ and $\b C_0 = \t{\b C}$. Consequently,
$${\b R^{(a,b),(b,a)}_0} \text{ diag}( {\b \Lambda}^{(a,b),(b,a)}_0) \b R^{(a,b),(b,a)T}_{0}= \t {\b R}^{(a,b),(b,a)} \operatorname{diag}\left( {\b \Lambda}^{(a,b),(b,a)}_0 \right) \t{\b R}^{(a,b),(b,a)T}.$$
Since the elements of $\Lambda_0$ are all non-negative, the solution of the above equation implies
\begin{equation}
\label{eq:GMM_C_equal}
\b R^{(a,b),(b,a)}_0 \operatorname{diag}\left( {\b \Lambda}^{(a,b),(b,a)}_0 \right)^{\frac{1}{2}} = \t {\b R}^{(a,b),(b,a)} \operatorname{diag}\left( {\b \Lambda}^{(a,b),(b,a)}_0 \right)^{\frac{1}{2}}\b O,
\end{equation}
 where $\b O$ is an orthogonal matrix. We write $\b G^{(a,b),(b,a)}_0, \t {\b G}^{(a,b),(b,a)}$ as 
 $$
 \b G^{(a,b),(b,a)}_0 = \begin{pmatrix}
\alpha^n_{0,ab} & \alpha^r_{0,ab} \\
\alpha^r_{0,ab} & \alpha^n_{0,ba} \\
\end{pmatrix},\quad
\t {\b G}^{(a,b),(b,a)} = \begin{pmatrix}
\t \alpha^n_{ab} &  \t \alpha^r_{ab} \\
 \t \alpha^r_{ab} & \t \alpha^n_{ba} \\
\end{pmatrix},
 $$
 and by the definition of ${\b R}^{(a,b),(b,a)}$ and the $2\times 2$ matrix inverse formula, we can show that
\begin{equation*}
{\b R}^{(a,b),(b,a)}_0 = 
\det\left({\b R}^{(a,b),(b,a)}_0\right)
\left(\begin{array}{cc}
1-\alpha^n_{0,ba} & \alpha^r_{0,ab} \\
\alpha^r_{0,ab} & 1-\alpha^n_{0,ab}
\end{array}\right).\\ 
\end{equation*}
Since the true parameter is in our parameter space $\b \Theta_{ab}$, we note that $\b M^{(a,b),(b,a)}_0= \begin{pmatrix}
M^{ab}_0 \\
M^{ba}_0\\
\end{pmatrix}$ has all positive elements. Further, because $\rho({\b G}^{(a,b),(b,a)}_0)<1$, and the two eigenvalues $\lambda_1, \lambda_2$ of it are real numbers, we have $\det({\b R_0^{(a,b),(b,a)}}) = \det(\b I-  {\b G}^{(a,b),(b,a)}_0) = (1-\lambda_1)(1-\lambda_2)>0$. Also, we know $1-\alpha^n_{0,ba}$ and $1-\alpha^n_{0,ab}$ are also positive in our parameters space and $\alpha^r_{0,ab}$ is non-negative. 

Using these results, we know the elements in $\b \Lambda^{(a,b),(b,a)}_0 = {\b R}^{(a,b),(b,a)}_0 \b M^{(a,b),(b,a)}_0$ are all positive, and we let $l_1 = \sqrt{\Lambda^{ab}_0}>0, l_2=\sqrt{\Lambda^{ba}_0}>0$. We can obtain similar results for $\t{\b R}$. Therefore we can write the elements in (\ref{eq:GMM_C_equal}) as below:
\begin{equation}
\label{eq:GMM_R_L}
\begin{aligned}
{\b R}^{(a,b),(b,a)}_0 \operatorname{diag}\left({\b \Lambda}^{(a,b),(b,a)}_0\right)^{\frac{1}{2}} &= 
\det\left({\b R}^{(a,b),(b,a)}_0\right)
\left(\begin{array}{cc}
l_1(1-\alpha^n_{0,ba}) & l_2\alpha^r_{0,ab} \\
l_1\alpha^r_{0,ab} & l_2(1-\alpha^n_{0,ab})
\end{array}\right),\\
\t{\b R}^{(a,b),(b,a)} \operatorname{diag}\left({\b \Lambda}^{(a,b),(b,a)}_0\right)^{\frac{1}{2}} &= 
\det\left(\t{\b R}^{(a,b),(b,a)}\right)
\left(\begin{array}{cc}
l_1(1-\t \alpha^n_{ba}) & l_2 \t \alpha^r_{ab} \\
l_1 \t \alpha^r_{ba} & l_2(1-\t \alpha^n_{ab})
\end{array}\right).
\end{aligned}
\end{equation}
 For the $2\times 2$ matrix, the orthogonal matrix is either the rotation matrix or the reflection matrix, i.e.,
$$\b O = \left(\begin{array}{cc}
\cos \theta & -\sin \theta \\
\sin \theta & \cos \theta
\end{array}\right) \text{(rotation) or } 
\b O = \left(\begin{array}{cc}
\cos \theta & \sin \theta \\
\sin \theta & -\cos \theta
\end{array}\right)\text{(reflection)}.
$$
We showed the determinant of ${\b R}^{(a,b),(b,a)}$ is positive for a $\b R$ which is in the parameter space. Using the formula $\operatorname{det}(\b A\b B) = \operatorname{det}(\b A)\operatorname{det}(\b B)$ for any square matrix $\b A, \b B$, we can conclude that $\det(\t {\b R}^{(a,b),(b,a)})\det(\b O) = \det({\b R}^{(a,b),(b,a)}_0)$ from (\ref{eq:GMM_C_equal}). Thus we must have $\det(\b O) > 0$ to ensure both sides have the same sign, and that means $\b O$ can only be the rotation matrix (the reflection matrix has determinant -1), so $\det(\b O) = 1$ (the rotation matrix has determinant 1), and it implies $\det(\t {\b R}^{(a,b),(b,a)}) = \det({\b R}^{(a,b),(b,a)}_0)$. We use this result and (\ref{eq:GMM_R_L}) and then plug them into (\ref{eq:GMM_C_equal}) to obtain
\begin{equation}
\label{eq:GMM_consistency_1}
    \left(\begin{array}{ll}
l(1-\t\alpha^n_{ba})\cos \theta + \t\alpha^r_{ab}\sin \theta & - l(1-\t\alpha^n_{ba})\sin \theta+  \t\alpha^r_{ab} \cos \theta\\
l \t\alpha^r_{ab}\cos \theta+(1-\t\alpha^n_{ab})\sin \theta & - l \t\alpha^r_{ab}\sin \theta+(1-\t\alpha^n_{ab})\cos \theta
\end{array}\right) =
\left(\begin{array}{cc}
l(1-\alpha^n_{0,ba}) &  \alpha^r_{0,ab} \\
l\alpha^r_{0,ab} & 1-\alpha^n_{0,ab}
\end{array}\right)
\end{equation}
for some $\theta$ and $l$ is defined as $l = \frac{l_1}{l_2}>0$. For the matrix equation (\ref{eq:GMM_consistency_1}), at the right hand side, we notice that the $(1,2)$-th element multiplied by $l$ equals the $(2,1)$-th element. Then for the left hand side, we can get the following equation,
\begin{equation*}
    (1-\t\alpha^n_{ab})\sin \theta =- l^2(1-\t\alpha^n_{ba})\sin \theta.
\end{equation*}
Since we know $l>0$, and $1-\t\alpha^n_{ab} >0, 1-\t\alpha^n_{ba}>0$ from Lemma \textcolor{red}{2}, the equation above holds if and only if $\sin \theta = 0$. Therefore $\cos \theta$ is $\pm 1$.  We plug this in (\ref{eq:GMM_consistency_1}), the $(1,1)$-th element of the matrix equation is $l(1-\t\alpha^n_{ba})\cos \theta = l(1-\alpha^n_{ba})$. Since $l, 1-\t\alpha^n_{ba}, 1-\alpha^n_{ba}$ are all positive in the parameter space (from Lemma \textcolor{red}{2}), we know $\cos \theta>0$, and as a result, $\cos \theta=1$. Therefore, $\b O$ must be the identity matrix.  Since $\operatorname{diag}( {\b \Lambda}^{(a,b),(b,a)}_0)^{\frac{1}{2}}$ is invertible, this implies ${\b R}^{(a,b),(b,a)}_0 = \t {\b R}^{(a,b),(b,a)}$ from (\ref{eq:GMM_C_equal}) . Because ${\b R}_{(a,b)} $ is also invertible, we conclude that $\b \Gamma_{(a,b)}, vec (\b \mu_{(a,b)}$ can be uniquely determined. Thus, condition 3 in Proposition \ref{prop:gmm} and Lemma \ref{lem:id} is proved. 

\end{proof}

Now, we use the above result to prove Theorem \ref{thm:gmm}.
\begin{proof}[Theorem \ref{thm:gmm}]
For any block pair $(a,b)$, define the function $\b g_0(\b M_{(a,b), (b,a)}, \b G_{(a,b), (b,a)} )$ on $\b \Theta_{(a,b), (b,a)}$ as follows.  
\begin{equation*}
\begin{gathered}
    g_{01}( .,.)  = (\Lambda_{ab})_0 -  \Lambda_{ab},  \quad  \quad 
   g_{02}(.,. )  =  (\Lambda_{ba})_0 -  \Lambda_{ab},\\
    g_{03}(.,.)  =  (C_{ab,ab})_0 -  C_{ab,ab},  \quad \quad  
    g_{04}(.,. )  =  (C_{ba,ba})_0 -  C_{ba,ba},\\
    g_{05}(.,. )  =  (C_{ab,ba})_0 -  C_{ab,ba}.
\end{gathered}
\end{equation*}
The sample version of the function $\hat{\b g}_n$ is defined by replacing $\b {\Lambda}_0$ and $\b { C}_0$ with $\b {\h\Lambda}$ and $\b {\h C}$ respectively. Now we need to verify the conditions in the Proposition \ref{prop:gmm} are satisfied. Condition 1 is satisfied by the definition of $\b \Theta_{(a,b), (b,a)}$ and stability condition in Lemma \ref{lem:stability}. Condition 2 can also be easily checked since $\b g_0$ is a vectorized composite function of some basic matrix operations, which are continuous in our parameter space. 
Condition 3 is the identification condition of GMM stated in Lemma \ref{lem:id}. 

To verify condition 4, we need to show $\b{\h\Lambda}$ and $\b {\h C}$ converge to the population statistics, that is, $\b{\h\Lambda} \stackrel{p}{\rightarrow} \b{\Lambda}$ and $\b {\h C} \stackrel{p}{\rightarrow} \b{C}$ uniformly for all $\b \theta \in \b \Theta$.  
Since our estimators depend on $T$, we let 
\[\b{\Lambda}_{(a,b),(b,a),T} = \mathbb{E}(\b{\h\Lambda}_{(a,b),(b,a)}), \quad \quad \b { C}_{(a,b),(b,a),T}=\E(\b {\h C}_{(a,b),(b,a)}),\]
 denote the expectations of the sample moments computed at time $T$ for any block pair $a,b$. We note  $\hat{\b\Lambda}_{(a,b),(b,a)}$ and $\hat{\b C}_{(a,b),(b,a)}$ are sample means of $n_{ab}$ random functions, and  $\hat{g}_n$ is polynomial function of $\theta$. Therefore, the uniform law of large  numbers (ULLN) for functions is applicable, and we have 
 \[\b{\h\Lambda}_{(a,b),(b,a)} \stackrel{p}{\rightarrow}\b{\Lambda}_{(a,b),(b,a),T}, \quad \b {\h C}_{(a,b),(b,a)} \stackrel{p}{\rightarrow} \b { C}_{(a,b),(b,a),T}, \] as $n_{ab}\rightarrow \infty$.
 
 Therefore we only need to show $\b{\Lambda}_{(a,b),(b,a),T} \rightarrow \b{\Lambda}_{(a,b),(b,a)}$ and ${ \b C}_{(a,b),(b,a),T}\rightarrow { \b C}_{(a,b),(b,a)}$ as $T\rightarrow \infty$. Assuming the Hawkes process is stationary, from \cite{bacry2015hawkes} and \cite{hawkes1971spectra}, we know $\mathbb{E}\left[d\b N_{(i,j), t}\right]$ is fixed, and from the definition we know that
\begin{equation*}
    {\Lambda}^{(a,b)}_T = \frac{1}{T}\int_{0}^T \mathbb{E}\left[d\b N_{(i,j), t}\right]={\Lambda}^{(a,b)} , \quad     {\Lambda}^{(b,a)}_T = \frac{1}{T}\int_{0}^T \mathbb{E}\left[d\b N_{(j,i), t}\right]={\Lambda}^{(b,a)}.
\end{equation*}
So $\b{\Lambda}_{(a,b),(b,a),T} = \b{\Lambda}_{(a,b),(b,a)}$. Therefore,  $\b{\Lambda}_{(a,b),(b,a),T}$ is an unbiased estimator of  $\b{\Lambda}_{(a,b),(b,a)}$ \citep{achab2017uncovering}. However, that is not the case for the estimator of the covariance matrix.

Let us denote the covariance density for the bivariate Hawkes process of the $(i,j)$ pair as
$$\Phi_{ij,ji}(\tau) = \frac{\mathbb{E}\left(d \b N_{(i,j), t} d \b N_{(j,i), t+\tau}\right)-\mathbb{E}\left(d \b N_{(i,j), t}\right)\mathbb{E}\left(d \b N_{(j,i), t+\tau}\right)}{(dt)^2}, $$
which does not depend on $t$, and $\Phi(\tau) = \Phi(-\tau)$ has non-negative elements in our parameter space \citep{gao2018functional}. In \cite{bacry2015hawkes} and \cite{achab2017uncovering}, it has been shown that $\int_{\tau\in \mathbb{R}}\Phi(\tau)d\tau = \b  R\operatorname{diag}(\b \Lambda) \b R^T={ \b C}$. From \cite{gao2018functional}, we know that the covariance of the count at time $T$ can be computed by
\begin{equation}
\label{eq:C_T}
\begin{split}
    &{ \b C}^{(a,b),(b,a)}_T \\
    &= \frac{1}{T}\operatorname{Cov}(\b N_{(i,j), T}, \b N_{(j,i), T})\\
    &= \frac{1}{T}\int_0^T\int_0^T \Phi(t_2 - t_1)dt_1dt_2\\
    &= \frac{1}{T}\left[\int_0^{T}\int_{-H}^H \Phi(\tau)d\tau dt - \underbrace{\int_0^{H}\int_{t_1-H}^0 \Phi(t_2-t_1)dt_2dt_1}_{\epsilon_{T,H,1}} -\underbrace{\int_{T-H}^{T}\int_{T}^{t_1+H} \Phi(t_2-t_1)dt_2dt_1}_{\epsilon_{T,H,2}}\right.\\
    &~~~~~~~~~~\left.+ \underbrace{\int_H^{T}\int_{0}^{t_1-H} \Phi(t_2-t_1)dt_2dt_1}_{\epsilon_{T,H,3}}+ \underbrace{\int_0^{T-H}\int_{t_1+H}^{T} \Phi(t_2-t_1)dt_2dt_1}_{\epsilon_{T,H,4}}\right]\\
    &=\int_{-H}^H\Phi(\tau)d\tau + \frac{1}{T}\left( \epsilon_{T,H,1} +\epsilon_{T,H,2}+\epsilon_{T,H,3}+\epsilon_{T,H,4}\right),
\end{split}
\nonumber
\end{equation}
where we choose $H= \sqrt T$. For the 1st term, we have $\int_{-H}^H\Phi(\tau)d\tau \rightarrow { \b C}_{(a,b)}$ as $H\rightarrow \infty$ since it is integrable, so we only need to show $\frac{1}{T}\epsilon_{T,H,i} \rightarrow 0$ for $i=1,2,3,4$. Actually, we have
\begin{equation*}
\begin{split}
\frac{1}{T}\epsilon_{T,H,1} &= \frac{1}{T}\int_{0}^{H}\int_{t_1-H}^{0} \Phi(t_2-t_1)dt_2dt_1\\
&\leq \frac{1}{T} \int_{0}^{H}\int_{t_1-H}^{t_1+H} \Phi(t_2-t_1)dt_2dt_1\\
&= \frac{H}{T}\int_{-H}^{H} \Phi(\tau)d\tau \\
\frac{1}{T}\epsilon_{T,H,3} &=\frac{1}{T}\int_H^{T}\int_{0}^{t_1-H} \Phi(t_2-t_1)dt_1dt_2\\
&\leq \frac{1}{T}\int_H^{T}\int_{t_1-T}^{t_1-H} \Phi(t_2-t_1)dt_1dt_2\\
&\leq \int_{-T}^{-H} \Phi(\tau)d\tau\\
\end{split}
\end{equation*}
Similarly, we can get $\frac{1}{T}\epsilon_{T,H,2}\leq \frac{H}{T}\int_{-H}^{H} \Phi(\tau)d\tau$ and $\frac{1}{T}\epsilon_{T,H,4}\leq  \int_{H}^{T} \Phi(\tau)d\tau$. We can see $\frac{H}{T} =  \frac{1}{\sqrt T}\rightarrow 0$ and $H = \sqrt T\rightarrow \infty$, then $\frac{1}{T}\epsilon_{T,H,i}$ will all converge to 0 for $i=1,2,3,4$. Therefore, we can get $\b{C}_{(a,b),(b,a),T} \rightarrow \b{C}_{(a,b),(b,a)}$, and condition 4 is proved.
\end{proof}

\section{Proofs of Other Results}
\label{sec:props_lemmas}

\subsection{Additional Propositions}

The following proposition from \cite{soliman2022multivariate} is based on a few observations regarding matrices with identical row sums. Let $\boldsymbol{1}$ denote the column vector of all $1$'s.

\begin{proposition}
\label{lem:identical_row_sum}
[Proposition A.1 in \cite{soliman2022multivariate}] For any matrix $\boldsymbol A$, if $\boldsymbol A\boldsymbol{1} = a\boldsymbol{1}$, i.e., the row sums of $\boldsymbol{A}$ are identical, then the following results hold,
\begin{enumerate}
    \item If $\boldsymbol A^{-1}$ exists, then $\boldsymbol A^{-1}\boldsymbol{1} = a^{-1} \boldsymbol{1}$.
    \item If $\boldsymbol B \boldsymbol 1 = b \boldsymbol 1$ for some matrix $ \boldsymbol B$, then $\boldsymbol A \boldsymbol B \boldsymbol 1 = ab \boldsymbol 1$
    \item If $\boldsymbol B \boldsymbol 1 = b\boldsymbol 1$ for some matrix $ \boldsymbol B$, then $(\boldsymbol A + \boldsymbol B)\boldsymbol 1  = (a+b)\boldsymbol 1$
\end{enumerate}
\end{proposition}

\begin{proof}
First, we have $\boldsymbol A^{-1}\boldsymbol A \boldsymbol 1 = a\boldsymbol A^{-1} \boldsymbol  1$, so $\boldsymbol A^{-1}\boldsymbol{1} = a^{-1} \boldsymbol{1}$. Second, $\boldsymbol A \boldsymbol B \boldsymbol 1 = b\boldsymbol A  \boldsymbol 1 = ab \boldsymbol 1$. Last, $(\boldsymbol A + \boldsymbol B)\boldsymbol 1  = \boldsymbol A \boldsymbol 1+ \boldsymbol B \boldsymbol 1 =(a+b)\boldsymbol 1$.
\end{proof}

Next, we re-state a result from \cite{minc1974nonnegative} regarding the relationship between spectral radius, and minimum and maximum row sum of a  non-negative matrix.
\begin{proposition}
\label{lem:bound_eig}
[\cite{minc1974nonnegative} Theorem 4.2, p14; Theorem 1.1, p24] If $\boldsymbol A \in \mathbb{R}^{n\times n}$ is a non-negative matrix, then 
$$\min _{1 \leq i \leq n} \sum_{j=1}^n \b A_{ij} \leq \rho(\b A) \leq \max _{1 \leq i \leq n} \sum_{j=1}^n\b A_{ij}$$
\end{proposition}

Finally, we state a slight variation of the matrix noncommutative Khintchine inequality. 

\begin{proposition}
\label{lem:khintchine_inequality}[\cite{oliveira2010sums} Matrix noncommutative Khintchine inequality] Let $\b A\in \mathbb{R}^{n\times n}$ be a random matrix with jointly Gaussian entries and $\E \b A = \b 0$, then 
\begin{equation*}
    \E \|\b A\| \leq 2\sqrt{1+2 \log n} \max\left\{ \|\E \b A^T\b A\|^{1/2}, \|\E \b A\b A^T\|^{1/2}  \right\}.
\end{equation*}
\end{proposition}

\begin{proof}
Let $\b B= \left(\begin{array}{cc}
\b 0 & \b A \\
\b A^T & \b 0\\
\end{array}\right)$. Then $\b B$ is a symmetric matrix with jointly Gaussian entries and $\E \|\b A\| = \E \|\b B\|$. In fact, $\b B$ can be written as a sum of finite random independent symmetric matrices, i.e.,
$\b B = \sum_{i=1}^m \epsilon_i \b H_i $, where $\epsilon_i$ are independent standard Gaussian random variables, and $\b H_i\in\mathbb{R}^{2n\times 2n}$ are some fixed symmetric matrices. Using Corollary 2.4 in \cite{tropp2018second}, we have 
\begin{equation*}
    \begin{split}
\E \|\b B\| \leq 2\sqrt{1+2 \log n} \|\E \b B^T \b B\|^{1/2} &=  2\sqrt{1+2 \log n} \left\| 
\left(\begin{array}{cc}
\E \b A^T \b A & \b 0 \\
 \b 0 & \E \b A \b A^T\\
\end{array}\right)
\right\|^{1/2}\\
&\leq 2\sqrt{1+2 \log n}\max\left\{ \|\E \b A^T\b A\|^{1/2}, \|\E \b A\b A^T\|^{1/2}  \right\}
. 
\end{split}
\end{equation*}
\end{proof}

\subsection{Proof of Lemma \ref{lem:row_length_X}}
\label{l4proof}
\begin{proof}
Let $\b D = (Z^TZ)^{1/2}=\operatorname{diag}(\sqrt{n_1},\dots, \sqrt{n_k})$. Then, using the SVD for $\b D\b B\b D$, we can have $\b D \b B \b D = \b U \b \Lambda \b V^T$, where $\b U, \b V \in \mathbb{R}^{K\times K}$ are orthonormal matrices, and $\b \Lambda \in \mathbb{R}^{K\times K}$ is a diagonal matrix. 
Let ${\t{\b X}}_L =\b Z\b D^{-1} \b U, {\t{\b X}}_R =\b Z\b D^{-1} \b V$,
then we can have $\tN = \t{\b X}_L \b\Lambda \t{\b X}_R^T$ is the SVD of $\tN$ because $$\t{\b X}_L^T\t{\b X}_L = \b U^T \b D^{-1} \b Z^T   \b Z \b D^{-1} \b U = \b U^T \b D^{-1} \b D^2 \b D^{-1} \b U = \b I,$$ 
and similarly we can show $\t{\b X}_R^T\t{\b X}_R = \b I$. Let $\b Y = (\b D^{-1} \b U|\b D^{-1} \b V)$ which is a column concatenation of $\b D^{-1} \b U$ and $\b D^{-1} \b V$. Then clearly from the definition of $\t{\b X} $ we have $\t{\b X} =\b Z\b Y$. Moreover, 
\begin{equation*}
\b Y \b Y^T = \left(\b D^{-1} \b U|\b D^{-1} \b V\right) \left(\begin{array}{l}
\b U^T \b D^{-1}  \\
 \b V^T \b D^{-1}
\end{array}\right)
=\b D^{-1} \b U \b U^T \b D^{-1} + \b D^{-1} \b V \b V^T \b D^{-1}=2\b D^{-2}.
\end{equation*}
This result implies ${\b Y}$ is row orthogonal and the $k$th row length is $\|{\b Y}_{k\cdot}\|= \sqrt{2n_{k}^{-1}}$. Then we know for any $1\leq i < j \leq K$,
$$\|{\b Y}_{i\cdot}- {\b Y}_{j\cdot}\|^2 = \|{\b Y}_{i\cdot}\|^2 +\| {\b Y}_{j\cdot}\|^2
=2(n_{i}^{-1}+n_{j}^{-1}).$$
The second claim comes from $\t{\b X}_{i\cdot} =\b  Z_{i\cdot} \b Y=\b Y_{z_i\cdot}$ and hence $\t{\b X}^*_{i\cdot} =  \b Y_{z_i\cdot}^* = \b  Z_{i\cdot} \b Y^*$.
\end{proof}

\subsection{A Variation of the Davis Kahan Theorem}
\label{sec:davis_kahan}
The following is a variation of the Davis Kahan theorem \citep{dk70}. 

\begin{proposition} 
\label{lem:eigenvectors_perturbation}
Let $\b X_L (\b X_R)$ be the top K left (right) singular vectors of ${\b N_T}$. Let $\b X$ be the column concatenate matrix $\b X = (\b X_L |\b X_R )\in\mathbb{R}^{n\times 2K}$. We use $ \lambda_1 \geq \dots \geq \lambda_K >0$ to denote the top $K$ positive singular values of $\frac{\E \boldsymbol {N_T}}{T}$. Then there exists a $2K\times 2K$ orthonormal matrix $\b Q$ such that 
$$\|\b X - \t{\b X} \b Q\|_F^2 \leq \frac{16K\| \frac{1}{T}({\b N_T} - \E \boldsymbol {N_T}) \|^2}{ \lambda_K^2}.$$
\end{proposition}

\begin{proof}
The proof is similar to \cite{lei2015consistency} and \cite{rohe2016co}, but we modify them to analyze the concatenate singular subspace for the (expected) count matrix. 
By the Proposition 2.2 of \cite{vu13}, there exist orthonormal matrices $\b Q_L, \b Q_R \in\mathbb{R}^{K\times K}$ such that
\begin{equation*}
\begin{split}
\|\b X_L - \t{\b{ X}}_L \b Q_L\|_F^2 &\leq 2\left\|(\b I - \b X_L \b X_L^T)\t{\b{ X}}_L\t{\b{ X}}_L^T\right\|_{F}^{2},\\
\|\b X_R - \t{\b{ X}}_R \b Q_R\|_F^2 &\leq 2\left\|(\b I - \b X_R \b X_R^T)\t{\b{ X}}_R\t{\b{ X}}_R^T\right\|_{F}^{2}.\\
\end{split}
\end{equation*}
Let $\b Q = \left(\begin{array}{ll}
\b Q_L & \b 0 \\
\b 0 & \b Q_R\\
\end{array}\right)$ which is a orthonormal matrix and we have
$$\|\b X - \t{\b{ X}} \b Q\|_F^2  = \|\b X_L - \t{\b{ X}}_L \b Q_L\|_F^2 + \|\b X_R - \t{\b{ X}}_R \b Q_R\|_F^2. $$
Using the Wedin theorem (\cite{stewart1998perturbation}) and let $\b \Delta= {\b N_T} - \tN$, we have 
\begin{equation*}
\begin{split}
    \left\|(\b I - \b X_L \b X_L^T)\t{\b{ X}}_L \t{\b{ X}}_L^T\right\|_{F}^{2} + \left\|(\b I - \b X_R \b X_R^T)\t{\b{ X}}_R \t{\b{ X}}_R^T\right\|_{F}^{2} &\leq \frac{\left\|\t{\boldsymbol{X}}_L^{T} \boldsymbol{\Delta}\right\|_{F}^{2}+\left\|\boldsymbol{\Delta} \t{\boldsymbol{X}}_R\right\|_{F}^{2}}{\delta^{2}}\\
    &\leq \frac{2K\|\b\Delta\|^2}{(\t \lambda_K - \|\b \Delta\|)^{2}},
\end{split}
\end{equation*}
where $\delta = \min \left(\min _{1 \leq i \leq K, K \leq j \leq n}\left|\t \lambda_{i}-{\lambda}_{j}(\b N_T)\right|, \min _{1 \leq i \leq K}  \lambda_{i}\right)$, $\lambda_i(\b N_T)$ is the $i$-th largest singular value of $\b N_T$ and $\lambda_i$ is as defined in the statement of the proposition. The last inequality comes from the Weyl theorem (\cite{stewart1998perturbation}) which states $\left|\tilde{\lambda}_{i}-\lambda_{i}(\b N_T)\right| \leq\|\b \Delta\|$ for $i=1,\dots,n$, and the triangle inequality. Thus if $\|\b \Delta\|\leq  \lambda_K/2$, $\frac{2K\|\b\Delta\|^2}{( \lambda_K - \|\b \Delta\|)^{2}}\leq \frac{8K\|\b\Delta\|^2}{ \lambda_K^{2}}$. If $\|\b \Delta\| >  \lambda_K/2$, we can have
$$\left\|(\b I - \b X_L \b X_L^T) \t{\b{ X}}_L \t{\b{ X}}_L^T\right\|_{F}^{2} + \left\|(\b I - \b X_R \b X_R^T) \t{\b{ X}}_R \t{\b{ X}}_R^T\right\|_{F}^{2}\leq 2K\leq \frac{8K\|\b\Delta\|^2}{ \lambda_K^{2}}.$$
\end{proof}

\subsection{Proof of Lemma \ref{lem:stability}}

\begin{proof}
\cite{hawkes1971spectra} has shown a sufficient condition for the process to be stationary is that $\rho\left(\b G_{(a,b),(b,a)}\right) \leq \sigma^*<1$, so we will only need to prove the equivalent sufficient condition in terms of the parameters. Let $\lambda$ be any eigenvalue of $\b G_{(a,b),(b,a)}$. We know it satisfies 
\begin{equation}
\label{eq:restricted_SR_eigenvalues_eq}
(\alpha^n_{ab}-\lambda)(\alpha^n_{ba}-\lambda) - (\alpha^r_{ab})^2 =0.
\end{equation}
Since $(\alpha^n_{ab} - \alpha^n_{ba})^2 +4(\alpha^r_{ab})^2\geq 0$,
 we should have two real value roots in (\ref{eq:restricted_SR_eigenvalues_eq}). The sum of these eigenvalues is $\frac{\alpha^n_{ab}+\alpha^n_{ba}}{2}\geq 0$, thus the conditions for their absolute values are smaller or equal to $\sigma^*$ are
$$\frac{\alpha^n_{ab}+\alpha^n_{ba}}{2}  \leq \sigma^* \text{ and } (\alpha^n_{ab}- \sigma^*)(\alpha^n_{ba}-\sigma^*) - (\alpha_{ab}^r)^2\geq 0.$$
This condition is equivalent to $\alpha^n_{ab} \leq \sigma^*, \alpha^n_{ba} \leq \sigma^*$ and $\alpha^r_{ab} < \sqrt{ {(\sigma^*-\alpha^n_{ab})(\sigma^*-\alpha^n_{ba})}}$ in our parameter space.
\end{proof}

\section{Additional Details on Simulation Experiments}

\subsection{Derivation of Expected Count Matrix in the Simulation Varying $\gamma_{\max}$}
\label{sec:appendix_gamma_max}
The entries of the expected count matrix $\E\b N_T$ are as follows:
\begin{equation*}
\E( \b N_T)_{ij} = 
\begin{cases}
0.002T, & z_i=z_j \\
0.001T, & z_i=1, z_j=2 \\
0.0001T, & z_i=2, z_j=1
\end{cases}.
\end{equation*}
It is easy to check this result  when $z_i=z_j$ because only the base intensity $\b \mu$ influences the event counts. When $z_i=1, z_j=2$, we get
\begin{equation*}
\begin{aligned}
\left(\begin{array}{c}
\E (\b N_T)_{ij}\\
\E (\b N_T)_{ji}\\
\end{array}\right) &= T\left(\b I - \left(\begin{array}{cc}
\alpha^n_{12}&\alpha^r_{12}\\
\alpha^r_{21} &\alpha^n_{21}\\
\end{array}\right)\right)^{-1}\left(\begin{array}{c}
\mu_{12}\\
\mu_{21}\\
\end{array}\right)\\
&=T
\left(\begin{array}{cc}
1&- s\\
0 &1\\
\end{array}\right)^{-1}\left(\begin{array}{c}
0.001 -s\\
0.0001\\
\end{array}\right)\\
&=T
\left(\begin{array}{c}
0.001 \\
0.0001\\
\end{array}\right).
\end{aligned}
\end{equation*}

\subsection{Sensitivity of Hawkes Process Parameters on Community Detection}
\label{sec:app_sim_community_detection}

In these experiments, we study the dependence of spectral clustering error on Hawkes process parameters which are summarized by the quantities $\mu_{\max}$ and $h(\gamma_1,\gamma_2,\mu_1,\mu_2)$. As we have shown in the theoretical results, the misclustering error rate is expected to be smaller when $\mu_{\max}$ and $h(\gamma_1,\gamma_2,\mu_1,\mu_2)$ are larger.  We can think of $h(\gamma_1,\gamma_2,\mu_1,\mu_2)$ as a representation of the signal to noise ratio of the model. It is easy to see that when $h(\gamma_1,\gamma_2,\mu_1,\mu_2)$ is close to 0, the difference in the expected counts between communities and within communities are nearly indistinguishable, which makes it hard for the algorithm to find the true community memberships. 

\begin{table}[t]
    \centering
    \caption{Descriptions of the 6 types of excitation in the MULCH model following an event from
    node $i$ in block $a$ to node $j$ in block $b$.}
    \label{tab:excitationDescriptions}
    \begin{tabular}{cp{4.9in}}
        \hline
        Parameter                         & Excitation Type \\
        \hline
        $\alpha_{ab}^{n}$ & \emph{Self excitation}: continuation of event $(x,y)$ \\
        $\alpha_{ab}^{r}$ & \emph{Reciprocal excitation}: event $(y,x)$ taken in response to event $(x,y)$ \\
        $\alpha_{ab}^{tc}$ & \emph{Turn continuation}: $(x,b)$ following $(x,y)$ to other nodes except for $y$ in the same block $b$\\
        $\alpha_{ab}^{ac}$ & \emph{Allied continuation}: event $(a,y)$ following $(x,y)$ from other nodes except $x$ in block $a$ \\
        $\alpha_{ab}^{gr}$ & \emph{Generalized reciprocity}: $(y,a)$ following $(x,y)$ to other nodes except $x$ in block $a$\\
        $\alpha_{ab}^{ar}$ & \emph{Allied reciprocity}: event $(b,x)$ following $(x,y)$ from other nodes except $y$ in block $b$ \\
        \hline
    \end{tabular}
\end{table}

\begin{figure}[tp]
    \newcommand{\figwidth}{0.244\textwidth}
    \centering
    \hfill
    \centering
    \begin{subfigure}[c]{\figwidth}
        \centering
        \includegraphics[width=\textwidth]{ 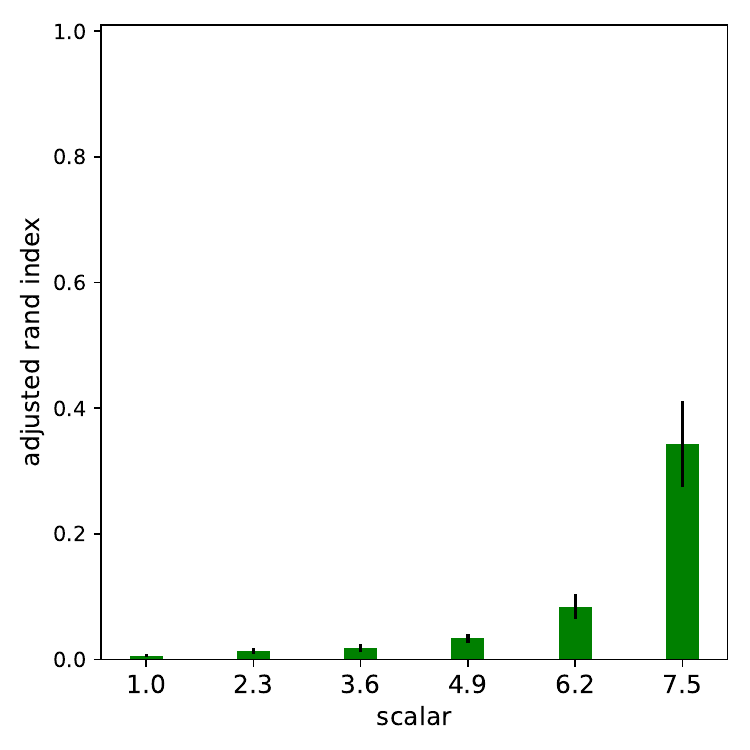}
        \caption{$s\times\alpha_{1}^{n}$}
        \label{fig:scalar_self}
    \end{subfigure}
    \hfill
    \centering
    \begin{subfigure}[c]{\figwidth}
        \centering
        \includegraphics[width=\textwidth]{ 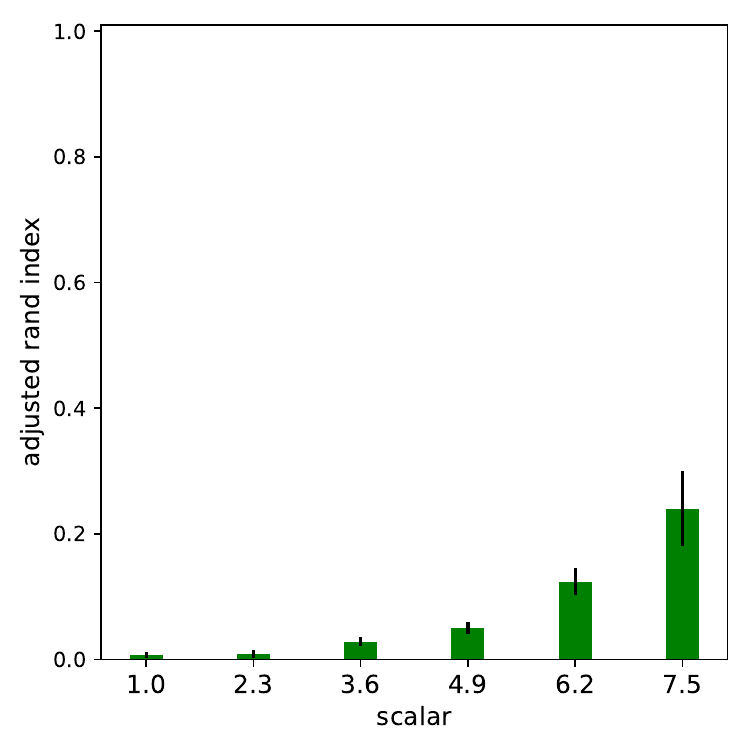}
        \caption{$s\times\alpha_{1}^{r}$}
        \label{fig:scalar_rec}
    \end{subfigure}
    \hfill
    \centering
    \begin{subfigure}[c]{\figwidth}
        \centering
        \includegraphics[width=\textwidth]{ 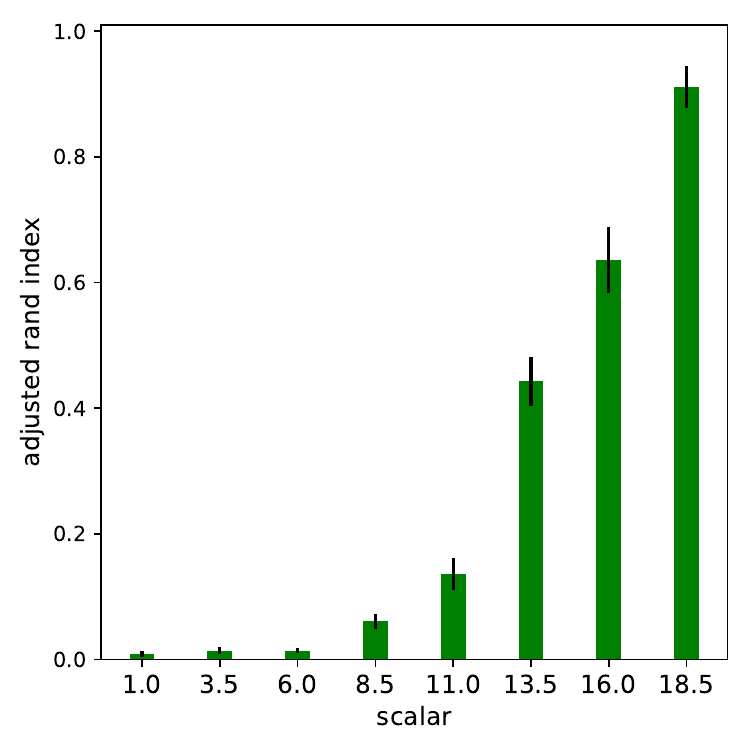}
        \caption{$s\times\alpha_{1}^{tc}$}
    \end{subfigure}
    \hfill
    \begin{subfigure}[c]{\figwidth}
        \centering
        \includegraphics[width=\textwidth]{ 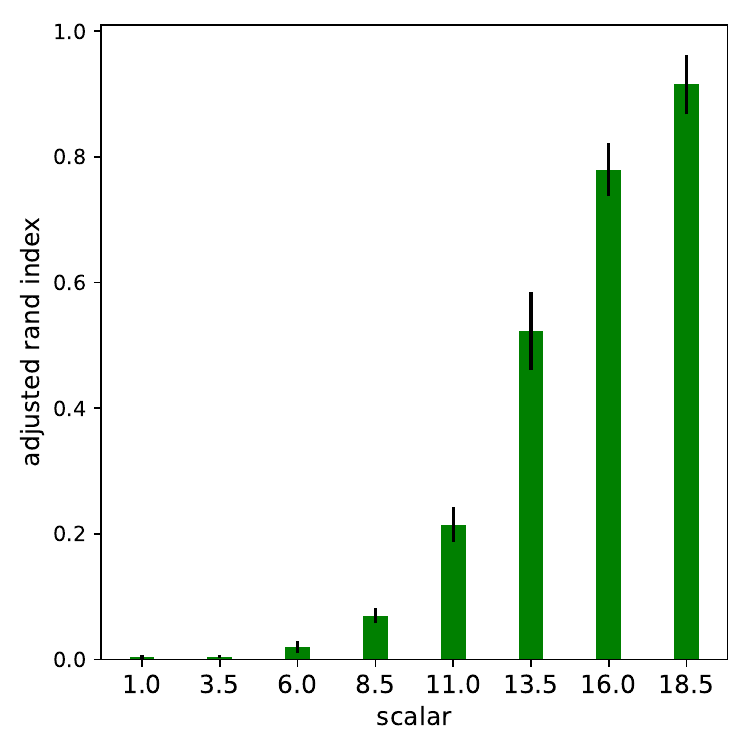}
        \caption{$s\times\alpha_{1}^{ac}$}
    \end{subfigure}
        \hfill
    \\
    \centering
    \begin{subfigure}[c]{\figwidth}
        \centering
        \includegraphics[width=\textwidth]{ 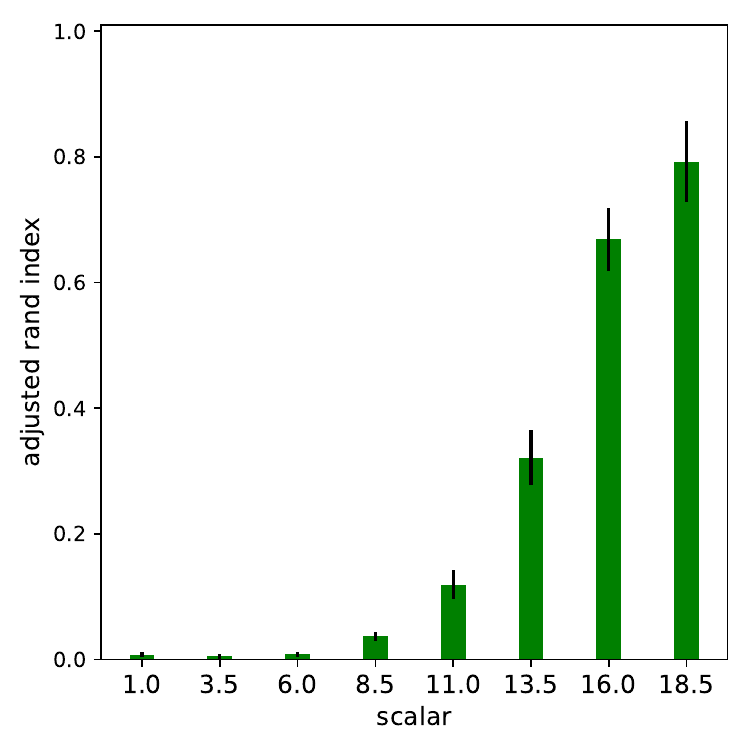}
        \caption{$s\times\alpha_{1}^{gr}$}
    \end{subfigure}
    \hfill
    \centering
    \begin{subfigure}[c]{\figwidth}
        \centering
        \includegraphics[width=\textwidth]{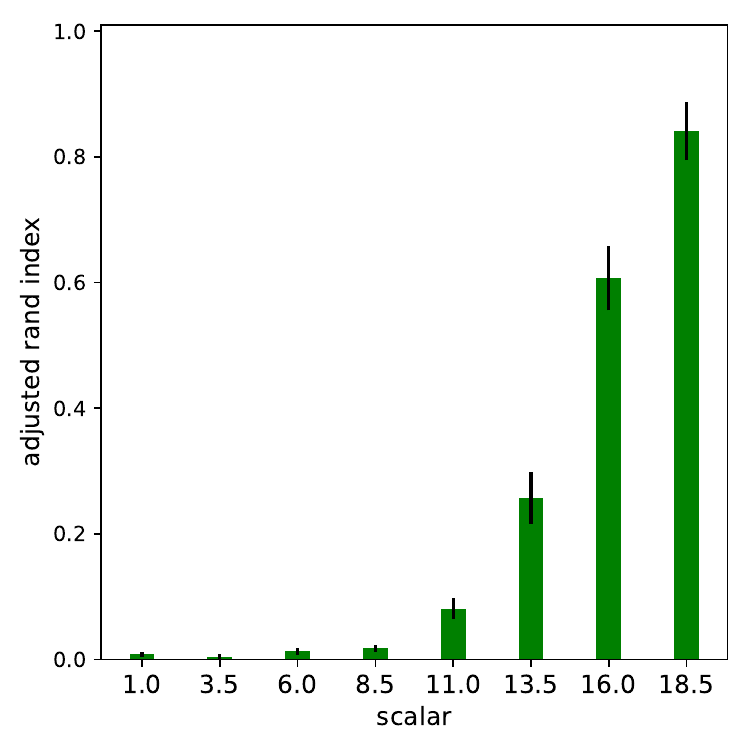}
        \caption{$s\times\alpha_{1}^{ar}$}
        \label{fig:scalar_ar}
    \end{subfigure}
    \hfill
    \begin{subfigure}[c]{\figwidth}
        \centering
        \includegraphics[width=\textwidth]{ 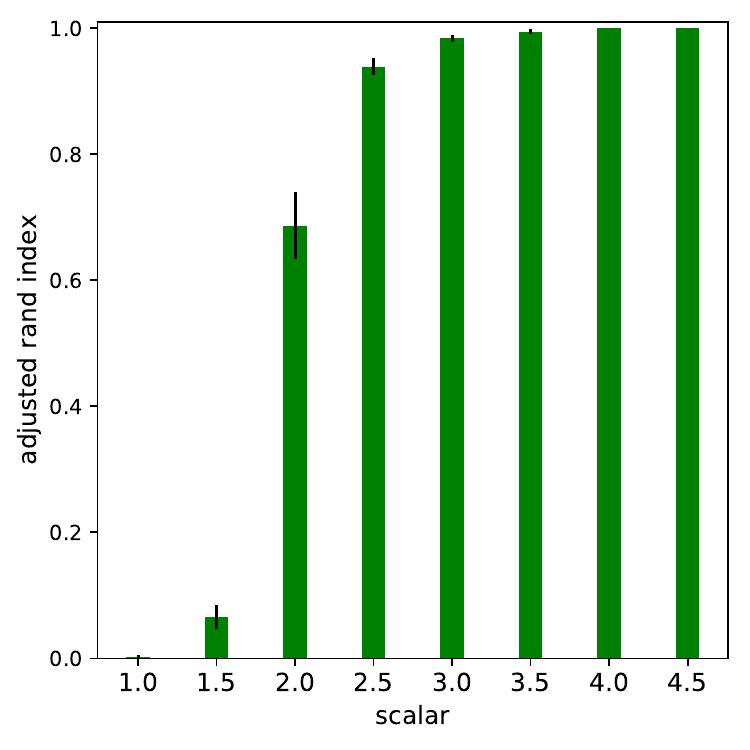}
        \caption{$\mu_1 = 10^{-3}, \mu_2 = \frac{\mu_1}{s}$}
        \label{fig:mu_divide}
    \end{subfigure}
    \hfill
    \begin{subfigure}[c]{\figwidth}
        \centering
        \includegraphics[width=\textwidth]{ 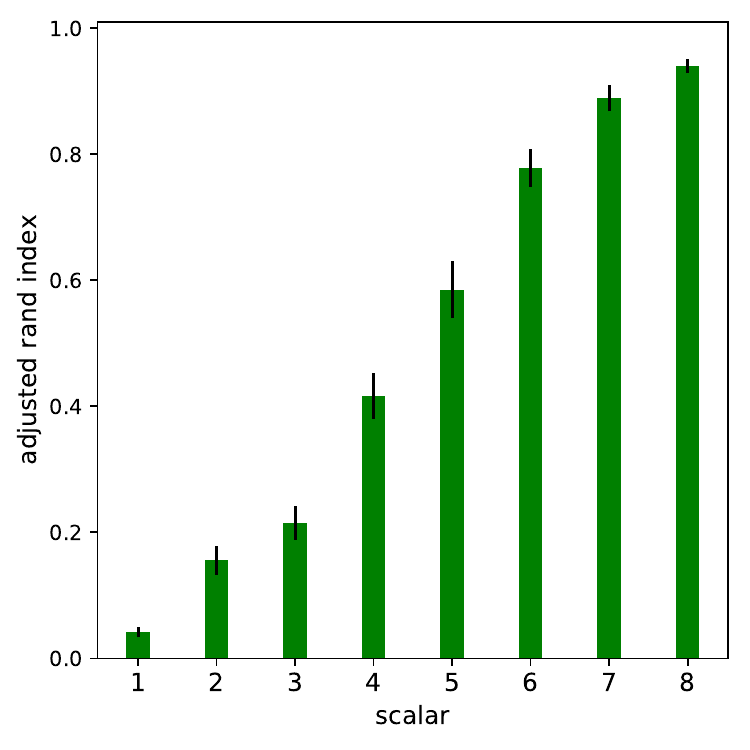}
        \caption{$\mu_{2}\!=\!10^{-4} s, \mu_{1}\!=\!2{\mu_2}$}
        \label{fig:mu_max_mul}
    \end{subfigure}
    \hfill
    \caption{Average adjusted Rand index while modifying one or more parameters at each time and keeping all other parameters as in Section \ref{eq:experiment_param_setting} ($\pm$ standard error over 15 runs).}
    \label{fig:scalar_params}
\end{figure}

To evaluate the influence of $h(\gamma_1,\gamma_2,\mu_1,\mu_2)$, we fix $K=4$, $n=100$, $T=700$ and all decay parameters in the kernel be $\b\beta=1$. 
We use all six types of excitations in the MULCH model, described in Table \ref{tab:excitationDescriptions}. 
We first let the diagonal block pairs and off-diagonal block pairs have the same parameters, i.e., $
\label{eq:experiment_param_setting}
\left(\mu_{1}, \alpha_{1}^{n}, \alpha_{1}^{r}, \alpha_{1}^{tc}, \alpha_{1}^{ac}, \alpha_{1}^{gr}, \alpha_{1}^{ar} \right) =
(\mu_{2}, \alpha_{2}^{n}, \alpha_{2}^{r}, \alpha_{2}^{tc}, \alpha_{2}^{ac}, \alpha_{2}^{gr}, \alpha_{2}^{ar})
=(0.0001, 0.1, 0.1, 0.0015, 0.0015, 0.0015, 0.0015)$.
Then we pick one parameter among the intra-block parameters $(\alpha_{1}^{n}, \alpha_{1}^{r}, \alpha_{1}^{tc}, \alpha_{1}^{ac}, \alpha_{1}^{gr}, \alpha_{1}^{ar})$ at a time and multiply its value by an increasing scalar $s$ while fixing all other parameters. For Example we make $\alpha_{1}^{n}=2\times0.1$, $\alpha_{1}^{n}=3\times0.1$, etc., while keeping values of all other parameters unchanged. 
The community detection accuracy averaged over 15 simulations are shown in Figure \ref{fig:scalar_self}-\ref{fig:scalar_ar}. We can see in these figures that increasing the scalar $s$ can improve the clustering accuracy. This is what we expect since all of them will increase $\gamma_1$ and thus will also increase $h(\gamma_1,\gamma_2,\mu_1,\mu_2)$. Our inequality predicts the clustering accuracy should also increase in this case.

We conduct a seventh experiment where we let $\mu_{1}=0.001$ and $\mu_{2}=0.001/s$ while fixing all other parameters (Figure \ref{fig:mu_divide}). In this setting $\mu_{\max} = \mu_1 = 0.001$, while the function $h(\cdot)$ changes. Together, these seven experiments  show that by increasing $h(\cdot)$ while fixing $\mu_{\max}$, the accuracy of spectral clustering increases. We can also notice that the absolute changes of $\alpha_1^n$ and $\alpha_1^r$ (i.e., $|\alpha_1^n-\alpha_2^n|$ and $|\alpha_1^r-\alpha_2^r|$) have smaller effect on the accuracy comparing with the absolute changes of other four parameters ($\alpha_1^{tc}, \alpha_1^{ac}, \alpha_1^{gr},\alpha_1^{ar}$). That is reasonable because in our theoretical results we can see the multipliers of $\alpha_1^n$ and $\alpha_1^r$ are 1, while the multipliers of $\alpha_1^{ac},\alpha_1^{tc},\alpha_1^{gr},\alpha_1^{ar}$ are $(n/K-2)$, so $\gamma_1$ has weaker dependence on $\alpha_1^n$ and $\alpha_1^r$, and same is true for $h(\gamma_1,\gamma_2,\mu_1,\mu_2)$. Similarly, in Figure \ref{fig:mu_divide}, we can see the clustering accuracy increases when the scalar $s$ increases. This aligns with our theory since $h(\gamma_1,\gamma_2,\mu_1,\mu_2)$ has negative association with $\mu_2/\mu_1$. Therefore increasing the scalar will decrease $\mu_2/\mu_1$ and thus $h(\gamma_1,\gamma_2,\mu_1,\mu_2)$ increases. All the results imply increasing $h(\gamma_1,\gamma_2,\mu_1,\mu_2)$ will improve clustering accuracy, which gives support to our theory.

The $\mu_{\max}$ in the upper bound controls the density level of the network, and we expect the spectral clustering will be easier when the network becomes denser since there will be more information available, and the difference between the blocks will also be magnified.
To check the influence of $\mu_{\max}$ while fixing other parameters, we let $\mu_1=0.0001s$ and $\mu_2=\mu_1/2$, where $s$ is an increasing scalar and we keep all other parameters unchanged. In this setting, $h(\gamma_1,\gamma_2,\mu_1,\mu_2)$ will stay unchanged, but $\mu_{\max}=\mu_1$ will increase. As we show in Figure \ref{fig:mu_max_mul}, the accuracy increases steadily as the scalar increases.

\vskip 0.2in

\bibliography{mvhawkes,sbm}

\end{document}